\documentclass{article} 
\usepackage[preprint]{neurips_2025}

\usepackage{amsmath,amsfonts,bm}

\def\eqref#1{equation~\ref{#1}}

\def\1{\bm{1}}

\DeclareMathAlphabet{\mathsfit}{\encodingdefault}{\sfdefault}{m}{sl}
\SetMathAlphabet{\mathsfit}{bold}{\encodingdefault}{\sfdefault}{bx}{n}

\newcommand{\softmax}{\mathrm{softmax}}

\usepackage[utf8]{inputenc} 
\usepackage[T1]{fontenc}    
\usepackage{hyperref}       
\usepackage{url}            
\usepackage{booktabs}       
\usepackage{amsfonts}       
\usepackage{nicefrac}       
\usepackage{microtype}      
\usepackage{xcolor}         
\usepackage{siunitx}

\usepackage{amsthm}
\usepackage{enumitem}
\usepackage{thm-restate}
\usepackage[most]{tcolorbox}
\usepackage{graphicx}
\usepackage{wrapfig}
\usepackage{subcaption}
\usepackage{soul}

\newtheorem{assumption}{Assumption}

\newtheorem{corollary}{Corollary}

\newtheorem{lemma}{Lemma}

\newcommand{\CL}{\mathrm{CL}}
\newcommand{\NS}{\mathrm{NSCL}}
\newcommand{\CKA}{\mathrm{CKA}}
\newcommand{\RSA}{\mathrm{RSA}}
\newcommand{\Corr}{\mathrm{Corr}}
\newcommand{\positive}{\mathrm{pos}}
\newcommand{\negative}{\mathrm{neg}}
\newcommand{\one}{\mathbf{1}}
\newcommand{\RDM}{\mathrm{RDM}}
\newcommand{\expnum}[2]{#1\mathrm{e}^{-#2}}

\title{On the Alignment Between Supervised and Self-Supervised Contrastive Learning}

\author{%
  Achleshwar Luthra
  \And
  Priyadarsi Mishra
  \And
  Tomer Galanti
  \AND
  \vspace{-1.5em} \\
  \texttt{\{luthra,priyadarsimishra,galanti\}@tamu.edu} \\
  Department of Computer Science and Engineering \\
  Texas A\&M University
}

\begin{document}

\maketitle

\begin{abstract}
Self-supervised contrastive learning (CL) has achieved remarkable empirical success, often producing representations that rival supervised pre-training on downstream tasks. Recent theory explains this by showing that the CL loss closely approximates a supervised surrogate, Negatives-Only Supervised Contrastive Learning (NSCL) loss, as the number of classes grows. Yet this loss-level similarity leaves an open question: {\em Do CL and NSCL also remain aligned at the representation level throughout training, not just in their objectives?}

We address this by analyzing the representation alignment of CL and NSCL models trained under shared randomness (same initialization, batches, and augmentations). First, we show that their induced representations remain similar: specifically, we prove that the similarity matrices of CL and NSCL stay close under realistic conditions. Our bounds provide high-probability guarantees on alignment metrics such as centered kernel alignment (CKA) and representational similarity analysis (RSA), and they clarify how alignment improves with more classes, higher temperatures, and its dependence on batch size. In contrast, we demonstrate that parameter-space coupling is inherently unstable: divergence between CL and NSCL weights can grow exponentially with training time.

Finally, we validate these predictions empirically, showing that CL–NSCL alignment strengthens with scale and temperature, and that NSCL tracks CL more closely than other supervised objectives. This positions NSCL as a principled bridge between self-supervised and supervised learning. Our code and project page are available at [\href{https://github.com/DLFundamentals/understanding_ssl_v2}{code}, \href{https://dlfundamentals.github.io/cl-nscl-representation-alignment/}{project page}].
\end{abstract}


\section{Introduction}

Self-supervised learning (SSL) has become the dominant approach for extracting transferable representations from large-scale unlabeled data. By leveraging training signals derived directly from the data, SSL methods avoid costly annotation while producing features that generalize across modalities, from vision~\citep{pmlr-v119-chen20j,He_2020_CVPR,zbontar2021barlow,he2022masked,oquab2024dinov2learningrobustvisual} 
to language~\citep{gao-etal-2021-simcse,reimers-gurevych-2019-sentence}, 
speech~\citep{schneider2019wav2vecunsupervisedpretrainingspeech,10.5555/3495724.3496768,10.1109/TASLP.2021.3122291,pmlr-v162-baevski22a}, 
and vision–language~\citep{pmlr-v139-radford21a,pmlr-v139-jia21b,Zhai_2023_ICCV,tschannen2025siglip2multilingualvisionlanguage}. Among SSL approaches, \emph{contrastive learning (CL)} has been particularly successful: methods such as SimCLR~\citep{pmlr-v119-chen20j}, MoCo~\citep{He_2020_CVPR,Chen_2021_ICCV}, and CPC~\citep{oord2019representationlearningcontrastivepredictive} train encoders by pulling together augmented views of the same input while pushing apart other samples. This simple principle has yielded state-of-the-art performance, often rivaling or surpassing supervised pre-training.

Despite this empirical success, a central puzzle remains: why does CL recover features so well aligned with semantic class boundaries? CL models often support nearly supervised-level downstream performance~\citep{amir2021deep,shaul2023reverse,weng2025clusteringpropertiesselfsupervisedlearning}, suggesting that supervision is somehow implicit in the objective. Recent theoretical progress sheds light on this: \citet{luthra2025selfsupervisedcontrastivelearningapproximately} showed that the CL objective closely approximates a supervised variant, \emph{Negatives-Only Supervised Contrastive Learning (NSCL)}, where same-class samples are excluded from the denominator. Their analysis established that the CL–NSCL \emph{losses} converge as the number of classes grows, and further characterized the geometry of NSCL minimizers and their linear probe performance. These results indicate that CL implicitly carries a supervised signal at the \emph{loss level}.

Yet this view leaves a crucial question unresolved:
\begin{tcolorbox}[colback=blue!5!white, colframe=blue!10!black, arc=1pt]
\centering
\textbf{\em Do contrastive and supervised contrastive models remain \\ aligned throughout training, not just at the level of their objectives?}
\end{tcolorbox}
\begin{wrapfigure}[16]{r}{0.7\textwidth} 
\vspace{-\baselineskip} 
\centering
\setlength{\tabcolsep}{5pt}
\begin{tabular}{ccc}
\includegraphics[width=0.298\linewidth]{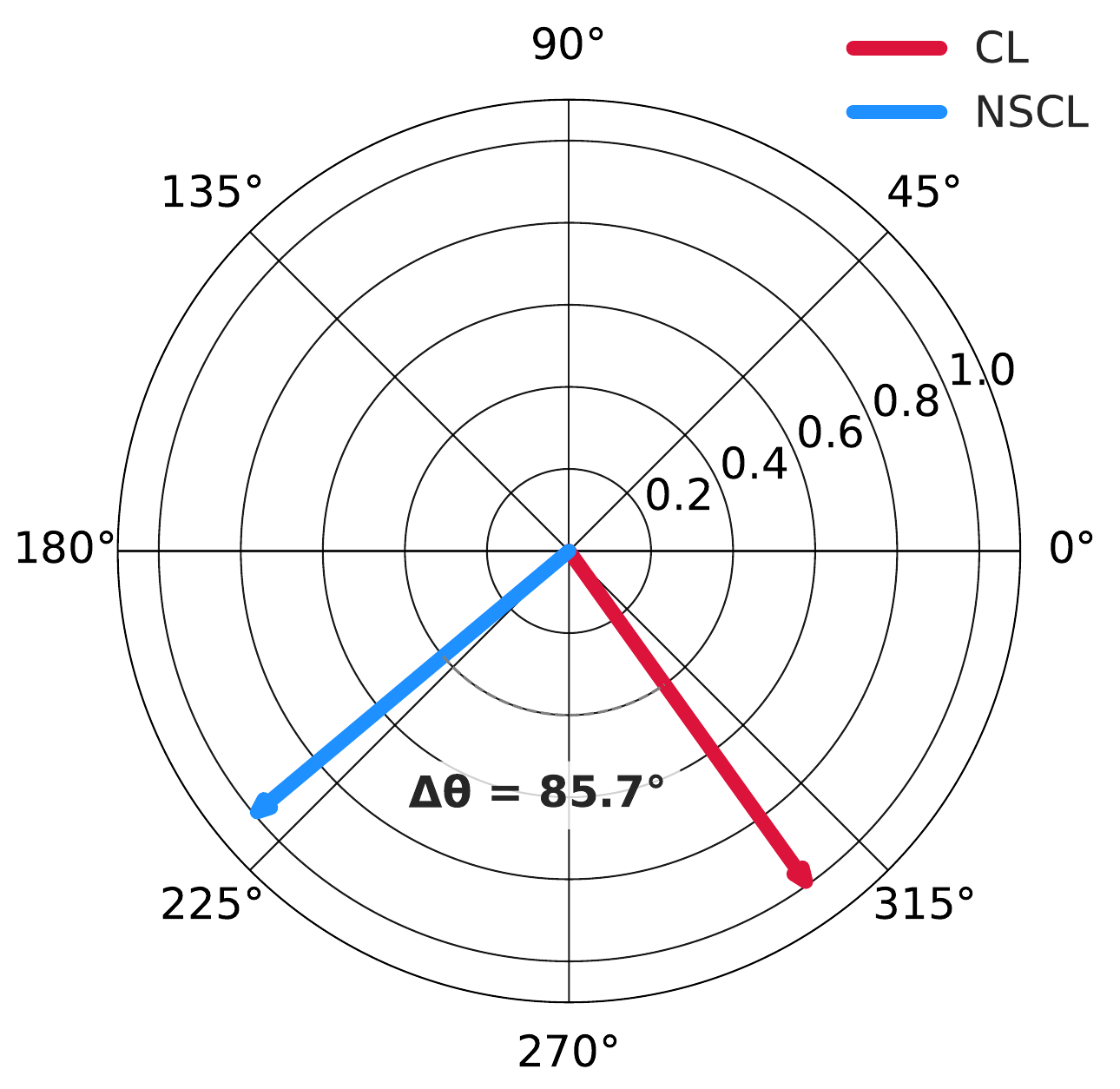} &
\includegraphics[width=0.298\linewidth]{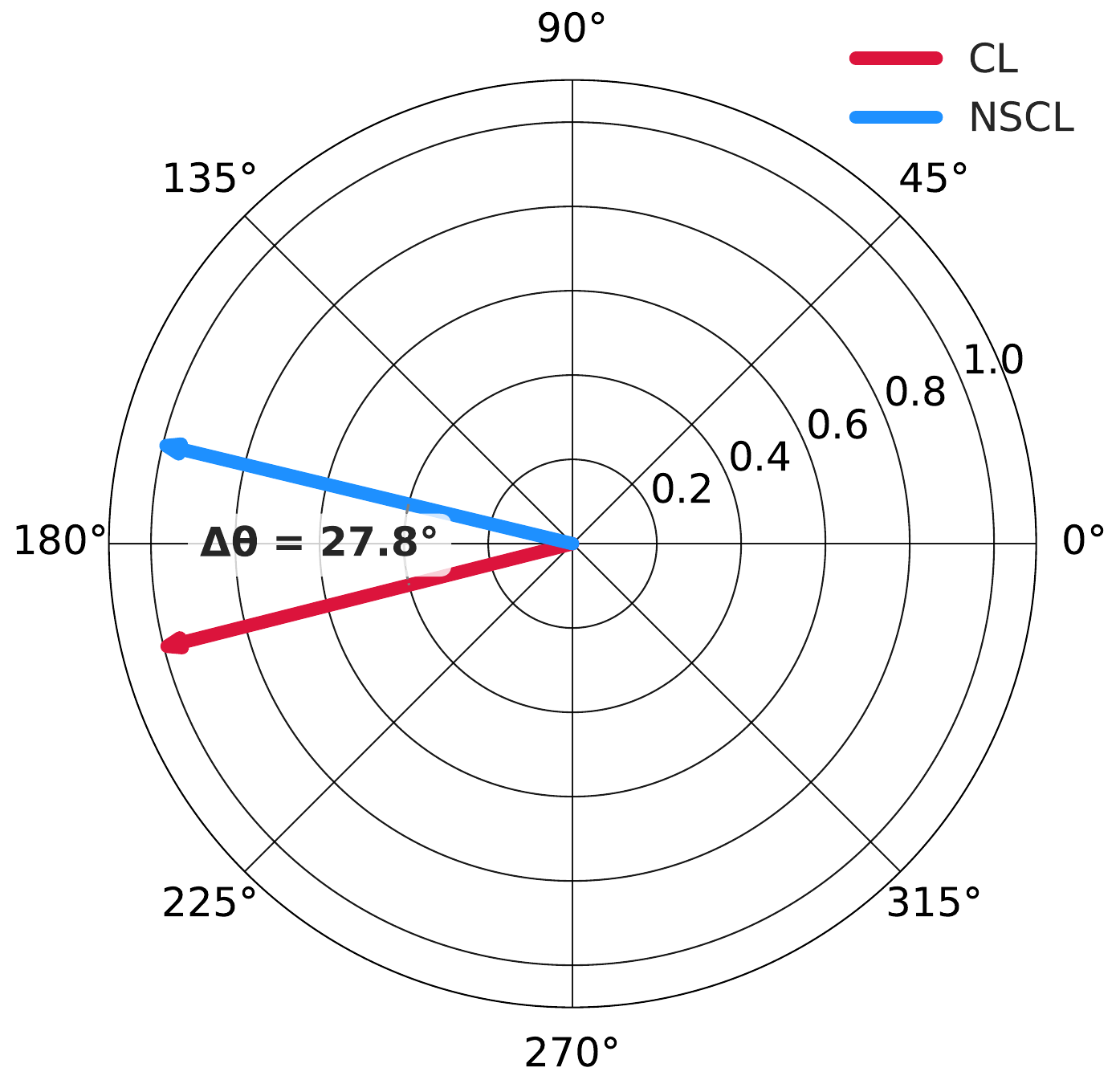} &
\includegraphics[width=0.298\linewidth]{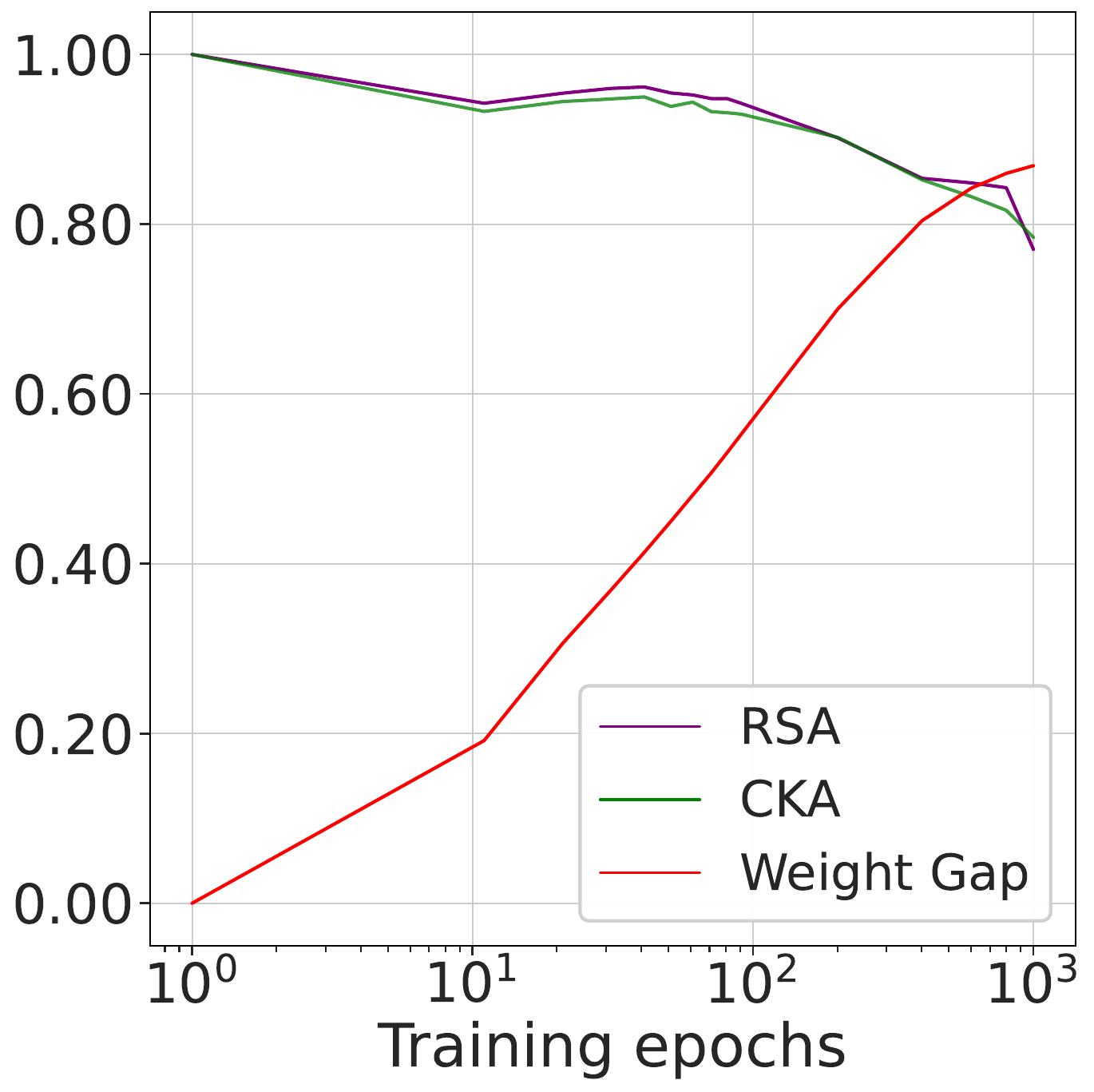} \\
\small (a) Weight Space &
\small (b) Representation Space &
\small (c) Similarity Metrics
\end{tabular}
\caption{\textbf{Comparison of learning dynamics for CL and NSCL models.}
(a) Weight space vectors show divergent paths ($85.7^\circ$ apart).
(b) In contrast, representation space vectors for a target class show high alignment ($27.8^\circ$ apart).
(c) This is confirmed over training epochs, where representational similarity (CKA, RSA) remains high while the weight gap increases (see figure details in App.~\ref{app:experiments}).}
\label{fig:teaser}
\vspace{-0.5\baselineskip}
\end{wrapfigure}

Loss-level similarity does not guarantee that optimization paths coincide. In principle, differences in curvature, gradient noise, or learning rate schedules could amplify small loss discrepancies, causing stochastic gradient descent (SGD) trajectories to diverge. Thus, it remains unclear whether CL merely converges to a solution \emph{similar} to NSCL, or whether their parameter and representations remain coupled across training.

\textbf{Our contributions.\enspace} We study the alignment between CL and NSCL under shared randomness (same initialization, mini-batches, and augmentations):

\begin{itemize}[leftmargin=10pt]

\item \textbf{From drift to metrics.} The similarity control yields explicit, high-probability \emph{lower bounds} on linear CKA and RSA at every epoch, showing that CL and NSCL representations remain non-trivially aligned and that the certified alignment tightens as $C$ and $\tau$ grow, and it may increase with $B$ depending on learning-rate scaling (Cors.~\ref{cor:cka-main}--\ref{cor:rsa-main}). For completeness, we also bound parameter drift under $\beta$-smoothness (Thm.~\ref{thm:nonconvex}), which can grow exponentially even when representations remain aligned.
\item \textbf{Empirical validation.} We validate our theory with experiments on CIFAR-10/100, Tiny-ImageNet, mini-ImageNet, and ImageNet-1K. We find that (i) CL–NSCL alignment strengthens with more classes and higher temperatures as well as correlates with the bound's dependence on the batch size; and (ii) NSCL aligns with CL more strongly than other supervised learning methods (such as cross-entropy minimization and supervised contrastive learning (SCL)~\citep{NEURIPS2020_d89a66c7}).
\end{itemize}


\section{Related Work}\label{sec:related}

A large body of work has sought to explain the success of contrastive learning (CL) from 
different perspectives. Early accounts linked CL to mutual information maximization between 
views of the same input~\citep{NEURIPS2019_ddf35421}, though subsequent analyses showed 
that enforcing mutual information constraints too strongly can degrade downstream 
performance~\citep{pmlr-v108-mcallester20a,Tschannen2020On}. A different line of work 
formalizes CL in terms of \emph{alignment} and \emph{uniformity} properties of the 
representation space~\citep{wang2020understanding,wang2021understanding,chen2021intriguing}, 
capturing how positives concentrate while negatives spread across the sphere. These 
geometric criteria, while intuitive, do not fully explain how samples from different 
semantic classes are organized under CL training. 

To address this, several papers have studied the ability of CL to recover latent clusters 
and semantic structures~\citep{arora2019theoreticalanalysiscontrastiveunsupervised,
tosh2021contrastive,zimmermann2021contrastive,pmlr-v151-ash22a,
NEURIPS2021_2dace78f,NEURIPS2021_27debb43,10.5555/3600270.3602219,
pmlr-v162-shen22d,wang2022chaos,pmlr-v162-awasthi22b,pmlr-v162-bao22e}. 
Most of these results rely on restrictive assumptions, such as conditional independence 
of augmentations given cluster identity~\citep{arora2019theoreticalanalysiscontrastiveunsupervised,
tosh2021contrastive,pmlr-v162-saunshi22a,pmlr-v162-awasthi22b}. To weaken such assumptions, 
\citet{haochen2023a} proposed analyzing spectral contrastive objectives that encourage 
cluster preservation without requiring augmentation connectivity, while 
\citet{pmlr-v195-parulekar23a} showed that InfoNCE itself learns cluster-preserving 
embeddings when the hypothesis class is capacity-limited. 

Another perspective comes from linking CL to supervised learning. For instance, 
\citet{balestriero2024the} showed that in linear models, self-supervised objectives such 
as VicReg coincide with supervised quadratic losses. Building on this supervised view, \citet{luthra2025selfsupervisedcontrastivelearningapproximately} established an explicit coupling between the InfoNCE contrastive loss and a supervised variant that removes positives from the denominator. In contrast to prior results, these bounds are label-agnostic, architecture-independent, and hold uniformly throughout
optimization. 

Beyond clustering and supervision, other theoretical studies have examined different aspects 
of CL: feature learning dynamics in linear and shallow nonlinear networks~\citep{NEURIPS2022_7b5c9cc0,
10.5555/3648699.3649029,pmlr-v139-wen21c,tian2023understanding}, the role and optimality of 
augmentations~\citep{NEURIPS2020_4c2e5eaa,feigin2025theoreticalcharacterizationoptimaldata}, 
the projection head~\citep{gupta2022understandingimprovingroleprojection,
gui2023unravelingprojectionheadscontrastive,xue2024investigating,ouyang2025projection}, 
sample complexity~\citep{alon2024optimal}, and strategies to reduce batch-size 
requirements~\citep{DBLP:conf/icml/YuanWQDZZY22}. Finally, several works explore connections 
between contrastive and non-contrastive SSL paradigms~\citep{wei2021theoretical,
balestriero2022contrastive,NEURIPS2021_02e656ad,garrido2023on,shwartzziv2023an}. 


\section{Problem Setup}\label{sec:setup}

We work with a class-balanced dataset $S=\{(x_i,y_i)\}_{i=1}^N\subset\mathcal{X}\times[C]$, where $[C]=\{1,\dots,C\}$ and each class $c$ contributes $n$ examples ($N=Cn$). An encoder $f_w:\mathcal{X}\to\mathbb{R}^d$ with parameters $w\in\mathbb{R}^p$ maps inputs to embeddings. Similarity is measured by a bounded function $\mathrm{sim}:\mathbb{R}^d\times\mathbb{R}^d\to[-1,1]$; throughout our experiments we use cosine similarity on $\ell_2$-normalized embeddings, $\mathrm{sim}(u,v)=\langle u,v\rangle/(\|u\|\|v\|)$, which satisfies this bound.

Data augmentations are modeled by a Markov kernel $\alpha(\cdot\,|\,x)$ on $\mathcal{X}$: given $x$, we draw an independent view $x'\sim\alpha(x)$. Unless stated otherwise, augmentation draws are independent across samples, across repeated views of the same sample, and across training steps. We write $x_i'\sim\alpha(x_i)$ for a single view and $(x_i^{(1)},x_i^{(2)})\stackrel{\text{i.i.d.}}{\sim}\alpha(x_i)$ for two views of the same input.

Fix a batch size $B\in\mathbb N$. A batch is a multiset $\mathcal B=\{(x_{j_t},x'_{j_t},y_{j_t})\}_{t=1}^B$ sampled with replacement from $S$, together with independent augmentations $x'_{j_t}\sim\alpha(x_{j_t})$. For any anchor triplet $(x_i,x'_i,y_i)\in\mathcal B$, define the \emph{per-anchor CL loss} and the \emph{CL batch loss} as:
{\small
\begin{equation*}
\begin{aligned}
\ell_{i,\mathcal B}(w)
&:=
-\log \left(
\frac{\exp \bigl(\mathrm{sim}(f_w(x_i),f_w(x'_i))/\tau\bigr)}
{\sum_{(x_j,x'_j,y_j)\in \mathcal B} \Bigl[\mathbf 1\{j\neq i\}\,\exp \bigl(\mathrm{sim}(f_w(x_i),f_w(x_{j_t}))/\tau\bigr)
+\exp \bigl(\mathrm{sim}(f_w(x_i),f_w(x'_{j_t}))/\tau\bigr)\Bigr]}
\right),\\[2mm]
\bar{\ell}^{\CL}_{\mathcal B}(w)
&:=
\frac{1}{B}\sum_{(x_i,x'_i,y_i)\in\mathcal B}\ell_{i,\mathcal B}(w).
\end{aligned}
\end{equation*}
}
For the same realized batch \(\mathcal B\), define the per-anchor \emph{negative subset} $\mathcal B_i^-:=\{(x_{j_t},x'_{j_t},y_{j_t})\in\mathcal B:\; y_{j_t}\neq y_i\}$. The \emph{NSCL per-anchor} and \emph{batch} losses are
\begin{small}
\begin{equation*}
\begin{aligned}
\ell_{i,\mathcal B_i^-}(w)
&:=
-\log \left(
\frac{\exp\bigl(\mathrm{sim}(f_w(x_i),f_w(x'_i))/\tau\bigr)}
{\sum_{(x_j,x'_j,y_j)\in \mathcal B_i^-}
 \Bigl[\exp\bigl(\mathrm{sim}(f_w(x_i),f_w(x_j))/\tau\bigr)
       +\exp\bigl(\mathrm{sim}(f_w(x_i),f_w(x'_j))/\tau\bigr)\Bigr]}
\right),\\[2mm]
\bar{\ell}^{\NS}_{\mathcal B}(w)
&:=
\frac{1}{B}\sum_{(x_i,x'_i,y_i)\in\mathcal B}\ell_{i,\mathcal B_i^-}(w).
\end{aligned}
\end{equation*}
\end{small}
Prior work~\citep{luthra2025selfsupervisedcontrastivelearningapproximately} shows that the CL–NSCL \emph{loss} gap is uniformly $\mathcal{O}(1/C)$, but what we ultimately care about is whether the \emph{embeddings} align. To quantify representation similarity we use linear Centered Kernel Alignment (CKA)~\citep{pmlr-v97-kornblith19a} and Representation Similarity Analysis (RSA)~\citep{kriegeskorte2008rsa} defined on cosine-similarity matrices: for $N$ common inputs with embeddings $Z=\{z_i\}_{i=1}^N$ and $Z'=\{z'_i\}_{i=1}^N$, let $\Sigma(Z)_{ij}=\cos(z_i,z_j)$ and $H=I-\tfrac{1}{N}\mathbf 1\mathbf 1^\top$; linear CKA is
\[
\mathrm{CKA}(Z,Z')~=~\frac{\langle H\Sigma(Z)H,\; H\Sigma(Z')H\rangle_F}{\|H\Sigma(Z)H\|_F\,\|H\Sigma(Z')H\|_F},
\]
and RSA is the \emph{Pearson} correlation between the (upper–triangular) off–diagonal
entries of the dissimilarity matrices $\RDM(Z)=\one\one^\top-\Sigma(Z)$ and
$\RDM(Z')=\one\one^\top-\Sigma(Z')$:
\[
\RSA(Z,Z')~=~\Corr \left(
  \mathrm{vec}_{\triangle}(\RDM(Z)),
  \mathrm{vec}_{\triangle}(\RDM(Z'))
\right),
\]
where $\mathrm{vec}_{\triangle}$ stacks the upper–triangular entries $(i<j)$ column-wise.

This raises the following question: \textbf{\em Beyond a small objective gap, does training CL and NSCL actually lead to similar representations (e.g., high CKA/RSA)?}

In the spirit of Thm.~1 of \citet{luthra2025selfsupervisedcontrastivelearningapproximately}, we prove that when two runs use shared randomness (same initialization, mini-batches, and augmentations), the per-step gradient mismatch is uniformly bounded (Lem.~\ref{lem:gap-param}). Similarly, we show that the CL and NSCL similarity matrices remain close throughout training (Thm.~\ref{thm:sim-coupling}), which yields explicit CKA/RSA lower bounds (Cors.~\ref{cor:cka-main}-\ref{cor:rsa-main}).

{\bf Additional notation for high–probability factors.\enspace} Fix a training horizon $T\in\mathbb N$, a confidence level $\delta\in(0,1)$, and a temperature $\tau>0$.
For later use, define $\epsilon_{B,\delta} := \sqrt{\tfrac{1}{2B}\log\Bigl(\tfrac{TB}{\delta}\Bigr)}$ and $\Delta_{C,\delta}(B;\tau) := \frac{2\,\mathrm e^{2/\tau} \left(\tfrac{1}{C}+\epsilon_{B,\delta}\right)}
{1-\tfrac{1}{C}-\epsilon_{B,\delta}}$, and assume $\epsilon_{B,\delta}<1-\tfrac{1}{C}$ so the denominator is positive.

\section{Theory}\label{sec:theory}

We examine how contrastive learning (CL) and negatives-only supervised contrastive learning (NSCL) co-evolve when initialized identically and trained with the same mini-batches and augmentations. While one might first attempt to study their trajectories in parameter space, such an approach quickly breaks down: without strong assumptions on the loss landscape (e.g., convexity or strong convexity), small reparameterizations can distort distances, and nonconvex dynamics cause parameter drift to grow uncontrollably over time (see App.~\ref{app:params}). For this reason, we set weight-space coupling aside and turn instead to the aspect that directly shapes downstream behavior—the \emph{representations}—analyzing their alignment in similarity space.

\subsection{Coupling in Representation (Similarity) Space}

Let $\Sigma_t\in[-1,1]^{N\times N}$ denote the pairwise similarity matrix of a fixed reference set at step $t$ (cosine similarity of normalized embeddings; diagonals are $1$). We analyze the coupled evolution of the CL and NSCL similarities,
\[
\Sigma_t^{\CL}, \;\Sigma_t^{\NS}\in[-1,1]^{N\times N},
\]
under identical mini-batches and augmentations. This representation-space view is invariant to reparameterization and directly tracks representational geometry.

{\bf Surrogate similarity dynamics.\enspace} To make the analysis explicit, we work with a ``similarity-descent'' surrogate that updates only those entries touched by the current batch. For a realized mini-batch $\mathcal B_t=\{(x_j,x_j',y_j)\}_{j=1}^B$ (with $x_j'\sim\alpha(x_j)$), let $\bar\ell^{\CL}_{\mathcal B_t}(\Sigma)$ and $\bar\ell^{\NS}_{\mathcal B_t}(\Sigma)$ be the usual InfoNCE-type losses written as functions of the relevant similarity entries (with temperature $\tau>0$). Define the batch-gradient maps
\[
G_t^{\CL} \;:=\; \nabla_{\Sigma}\,\bar\ell^{\CL}_{\mathcal B_t}\big(\Sigma_t^{\CL}\big),
\qquad
G_t^{\NS} \;:=\; \nabla_{\Sigma}\,\bar\ell^{\NS}_{\mathcal B_t}\big(\Sigma_t^{\NS}\big),
\]
setting all untouched entries to zero. The surrogate updates are
\begin{equation}\label{eq:Sigma-descent}
\Sigma^{\CL}_{t+1} \;=\; \Sigma^{\CL}_t - \eta_t\,G^{\CL}_t,
\qquad
\Sigma^{\NS}_{t+1} \;=\; \Sigma^{\NS}_t - \eta_t\,G^{\NS}_t,
\end{equation}
with shared initialization and shared randomness (same $\mathcal B_t$ and augmentations). In App.~\ref{app:sigma-just} we show that these surrogate dynamics faithfully track the similarity evolution induced by parameter-space SGD. We now formalize the coupling bound.

\begin{restatable}[Similarity-space coupling]{theorem}{SimCoupling}
\label{thm:sim-coupling}
Fix $B,T\in\mathbb N$, $\delta\in(0,1)$, and temperature $\tau>0$.
Consider the coupled similarity-descent recursions \eqref{eq:Sigma-descent} for CL and NSCL with shared initialization and shared mini-batches/augmentations. Then, with probability at least $1-\delta$ over the draws of the mini-batches and augmentations, for any stepsizes $(\eta_t)_{t=0}^{T-1}$,
\begin{equation}\label{eq:Sigma-gap}
\bigl\|\Sigma^{\CL}_T-\Sigma^{\NS}_T\bigr\|_F
~\le~
\exp\Bigl(\frac{1}{2\tau^2 B}\sum_{t=0}^{T-1}\eta_t\Bigr)\;
\frac{1}{\tau\sqrt{B}}\Bigl(\sum_{t=0}^{T-1}\eta_t\Bigr)\,\Delta_{C,\delta}(B;\tau).
\end{equation}
\end{restatable}

{\bf From similarity drift to CKA/RSA guarantees.\enspace} We translate the high-probability control on the similarity drift from Thm.~\ref{thm:sim-coupling},
into bounds on two standard representational metrics.

{\bf CKA.\enspace} Recall from Sec.~\ref{sec:setup} that linear CKA~\citep{pmlr-v97-kornblith19a} is
the normalized Frobenius inner product between centered similarity matrices.
$H:=I-\tfrac{1}{N}\mathbf 1\mathbf 1^\top$ be the centering projector and define centered Gram matrices $K^{\CL}_T:=H\Sigma^{\CL}_T H$ and $K^{\NS}_T:=H\Sigma^{\NS}_T H$.
The (linear) CKA at step $T$ is $\CKA_T=\frac{\langle K^{\CL}_T,K^{\NS}_T\rangle_F}{\|K^{\CL}_T\|_F\,\|K^{\NS}_T\|_F}\in[0,1]$. Because $\|H X H\|_F\le \|X\|_F$, any bound on $\|\Sigma^{\CL}_T-\Sigma^{\NS}_T\|_F$ controls $\|K^{\CL}_T-K^{\NS}_T\|_F$.  
For convenience, introduce the relative deviation $\rho_T := \frac{\|K^{\CL}_T - K^{\NS}_T\|_F}{\|K^{\CL}_T\|_F}$.

\begin{restatable}[CKA lower bound]{corollary}{CKAlowerbound}\label{cor:cka-main}
In the setting of Thm.~\ref{thm:sim-coupling}. Assume $\|K^{\CL}_T\|_F>0$. With probability at least $1-\delta$,
\[
\CKA_T ~\ge~ \frac{1-\rho_T}{1+\rho_T},
\qquad
\rho_T ~\le~
\frac{\exp\bigl(\tfrac{1}{2\tau^2 B}\sum_{t=0}^{T-1}\eta_t\bigr)\;
\frac{1}{\tau\sqrt{B}}\bigl(\sum_{t=0}^{T-1}\eta_t\bigr)\,\Delta_{C,\delta}(B;\tau)}{\|K^{\CL}_T\|_F}.
\]
\end{restatable}

{\bf RSA.\enspace} Recall from Sec.~\ref{sec:setup} that RSA~\citep{kriegeskorte2008rsa} is the Pearson correlation between the off-diagonal entries of representational dissimilarity matrices (RDMs). Let $M=\binom{N}{2}$ and define off-diagonal RDM vectors
$a_T,b_T\in\mathbb R^M$ by $a_T(u,v)=1-\Sigma^{\CL}_T(u,v)$ and $b_T(u,v)=1-\Sigma^{\NS}_T(u,v)$ for $u<v$. Write $\sigma_{D,T}>0$ for the empirical standard deviation of the entries of $a_T$. The RSA score is the Pearson correlation $\RSA_T=\Corr(a_T,b_T)$.  
Zeroing the diagonal does not increase Frobenius norms, so
$\|b_T-a_T\|_2\le \|\Sigma^{\NS}_T-\Sigma^{\CL}_T\|_F$.  
It will be useful to measure the relative discrepancy $r_T := \frac{\|b_T-a_T\|_2}{\sqrt{M}\,\sigma_{D,T}}$.

\begin{restatable}[RSA lower bound]{corollary}{RSAlowerbound}\label{cor:rsa-main}
In the setting of Thm.~\ref{thm:sim-coupling}. Assume $\sigma_{D,T} > 0$. With probability at least $1-\delta$,
\[
\RSA_T ~\ge~ \frac{1-r_T}{1+r_T},
\qquad
r_T ~\le~
\frac{\exp\bigl(\tfrac{1}{2\tau^2 B}\sum_{t=0}^{T-1}\eta_t\bigr)\;
\frac{1}{\tau\sqrt{B}}\bigl(\sum_{t=0}^{T-1}\eta_t\bigr)\,\Delta_{C,\delta}(B;\tau)}{\sqrt{M}\,\sigma_{D,T}}.
\]
\end{restatable}

These results complement the parameter–space analysis. While parameter trajectories may diverge exponentially (in the non-convex setting; see App.~\ref{app:param-space}), the induced similarities—and hence representational metrics such as CKA and RSA—remain tightly controlled by class count, batch size, learning rate, and temperature $\tau$. The key quantity is the similarity–matrix drift $\|\Sigma^{\CL}_T-\Sigma^{\NS}_T\|_F$, which Thm.~\ref{thm:sim-coupling} bounds in two stabilizing ways.

First, the exponential factor is moderated by the $\tfrac{1}{\tau^2 B}$ term in the exponent. Unlike parameter space, where the growth rate scales with $\beta$, the “instability rate’’ in similarity space is only $\tfrac{1}{2\tau^2 B}$ and is therefore negligible for typical batch sizes (e.g., $B\approx 10^2$–$10^3$).

Second, the prefactor $\tfrac{1}{\tau\sqrt{B}}\bigl(\sum_t \eta_t\bigr)\,\Delta_{C,\delta}(B;\tau)$ decreases rapidly with batch size and class count (note $\Delta_{C,\delta}(B;\tau)$ shrinks with larger $C$ and grows with smaller $\tau$ through $\mathrm e^{2/\tau}$). In practical regimes ($C\sim 10^3$, $B\sim 10^2$–$10^3$), this prefactor is small, making the total Frobenius gap negligible relative to the scale of the similarity matrices.

Together, these effects yield high–probability guarantees $\CKA_T \ge (1-\rho_T)/(1+\rho_T)$ and $\RSA_T \ge (1-r_T)/(1+r_T)$ with $\rho_T, r_T \ll 1$ in realistic conditions. Thus, even if parameters drift, the induced representations evolve in a coupled and stable manner—consistent with empirical findings that CL and NSCL remain closely aligned in practice.

{\bf Proof idea.\enspace}
We begin with a high–probability batch–composition guarantee (Cor.~\ref{cor:comp-hp}): with probability at least $1-\delta$, every anchor’s denominator contains the expected proportion of negatives up to an $\epsilon_{B,\delta}$ fluctuation. This rules out positive–heavy batches that would otherwise cause the NSCL renormalization to deviate substantially from CL. Conditioning on this event, the CL–NSCL batch–gradient gap decomposes into (i) a \emph{reweighting error}, bounded in total variation by $\Delta_{C,\delta}(B;\tau)$ (Lem.~\ref{lem:softmax-reweight}), and (ii) a \emph{stability term} from the dependence on the current similarities, controlled by the $\tfrac{1}{2\tau^2 B}$–Lipschitzness of the batch–gradient map in Frobenius norm (Lem.~\ref{lem:sigma-smooth} at temperature $\tau$). Using block–orthogonality across anchors (Lem.~\ref{lem:block-orth}), the reweighting contributions combine in quadrature, giving the per–step estimate (Lem.~\ref{lem:grad-gap-sim}),
\[
\bigl\|G^{\CL}_t(\Sigma^{\CL}_t)-G^{\NS}_t(\Sigma^{\NS}_t)\bigr\|_F
~\le~ \frac{1}{\tau}\cdot\frac{\Delta_{C,\delta}(B;\tau)}{\sqrt{B}}
 ~+~ \frac{1}{2\tau^2 B}\,\bigl\|\Sigma^{\CL}_t-\Sigma^{\NS}_t\bigr\|_F .
\]
Consequently, the similarity drift satisfies the recurrence
\[
\bigl\|\Sigma^{\CL}_{t+1}-\Sigma^{\NS}_{t+1}\bigr\|_F
~\le~\Bigl(1+\frac{\eta_t}{2\tau^2 B}\Bigr)\,\bigl\|\Sigma^{\CL}_{t}-\Sigma^{\NS}_{t}\bigr\|_F
 ~+~\eta_t\,\frac{1}{\tau}\cdot\frac{\Delta_{C,\delta}(B;\tau)}{\sqrt{B}},
\]
where the first term propagates existing error and the second injects the new discrepancy introduced at step $t$. Unrolling this recurrence (discrete Grönwall) yields
\[
\|\Sigma^{\CL}_T-\Sigma^{\NS}_T\|_F
~\le~
\exp\Bigl(\frac{1}{2\tau^2 B}\sum_{t=0}^{T-1}\eta_t\Bigr)\,
\frac{1}{\tau\sqrt{B}}\Bigl(\sum_{t=0}^{T-1}\eta_t\Bigr)\,\Delta_{C,\delta}(B;\tau).
\]
Finally, centering contracts Frobenius norms, so this control transfers directly to the centered Gram matrices, and applying standard $(1-\rho)/(1+\rho)$ and $(1-r)/(1+r)$ comparisons yields the claimed CKA/RSA lower bounds.

\section{Experiments}


{\bf Datasets and augmentations.\enspace} We experiment with the following standard vision classification datasets - CIFAR10 and CIFAR100~\citep{krizhevsky2009learning}, Mini-ImageNet~\citep{NIPS2016_90e13578}, Tiny-ImageNet~\citep{deep-learning-thu-2020}, and ImageNet-1K~\citep{5206848}. (See App.~\ref{app:experiments} for details.)

{\bf Methods, architectures, and optimizers. \enspace} For all our experiments, we have followed the SimCLR~\citep{pmlr-v119-chen20j} algorithm. We use a ResNet-50~\citep{He_2016_CVPR} encoder with a width-multiplier factor of 1. The projection head follows a standard two-layer MLP architecture composed of: \texttt{Linear($2048 \rightarrow 2048$) $\rightarrow$ ReLU $\rightarrow$ Linear($2048 \rightarrow 128$)}. For cross-entropy training, we attach an additional classification head \texttt{Linear($128 \rightarrow$ C)} where $C$ is the number of classes. 

For contrastive learning, we use the DCL loss that avoids positive-negative coupling during training~\citep{10.1007/978-3-031-19809-0_38}. For supervised learning, we use the following variants: Supervised Contrastive Loss~\citep{NEURIPS2020_d89a66c7}, Negatives-Only Supervised Contrastive Loss~\citep{luthra2025selfsupervisedcontrastivelearningapproximately}, and Cross-Entropy Loss~\citep{6773024}. To minimize the loss, we adopt the LARS optimizer~\citep{you2017large} which has been shown in~\citep{pmlr-v119-chen20j} to be effective for training with large batch sizes. For LARS, we set the momentum to $0.9$ and the weight decay to $\expnum{1}{6}$. All experiments are carried out with a batch size of $B = 1024$. The base learning rate is scaled with batch size as $0.3 \cdot \lfloor B / 256 \rfloor$, following standard practice~\citep{pmlr-v119-chen20j}. We employ a warm-up phase~\citep{goyal2017accurate} for the first 10 epochs, followed by a cosine learning rate schedule without restarts ~\citep{loshchilov2016sgdr} for the remaining epochs. All models were trained on a single node with one 94 GB NVIDIA H100 GPUs. 

{\bf Evaluation metrics. \enspace} To quantitatively measure the alignment between the learned representation spaces of different models, we monitor linear CKA and RSA (check Sec.~\ref{sec:setup} for details) during training. Both CKA and RSA range from 0 to 1, where 1 indicates identical similarity structures. For completeness, along with CKA and RSA, we also report downstream performance via Nearest Class Center Classifier~\citep{galanti2022on} and Linear Probing~\citep{kohn2015s, gupta2015distributional,belinkov2022probing} accuracies in Tab.~\ref{tab:acc_full_test_only}. 

\subsection{Experimental Results}
\label{sec:exp_results}

\begin{figure}[t]
    \centering
    \begin{tabular}{c@{\hspace{6pt}}c@{\hspace{6pt}}c@{\hspace{6pt}}c}
         \includegraphics[width=0.235\linewidth]{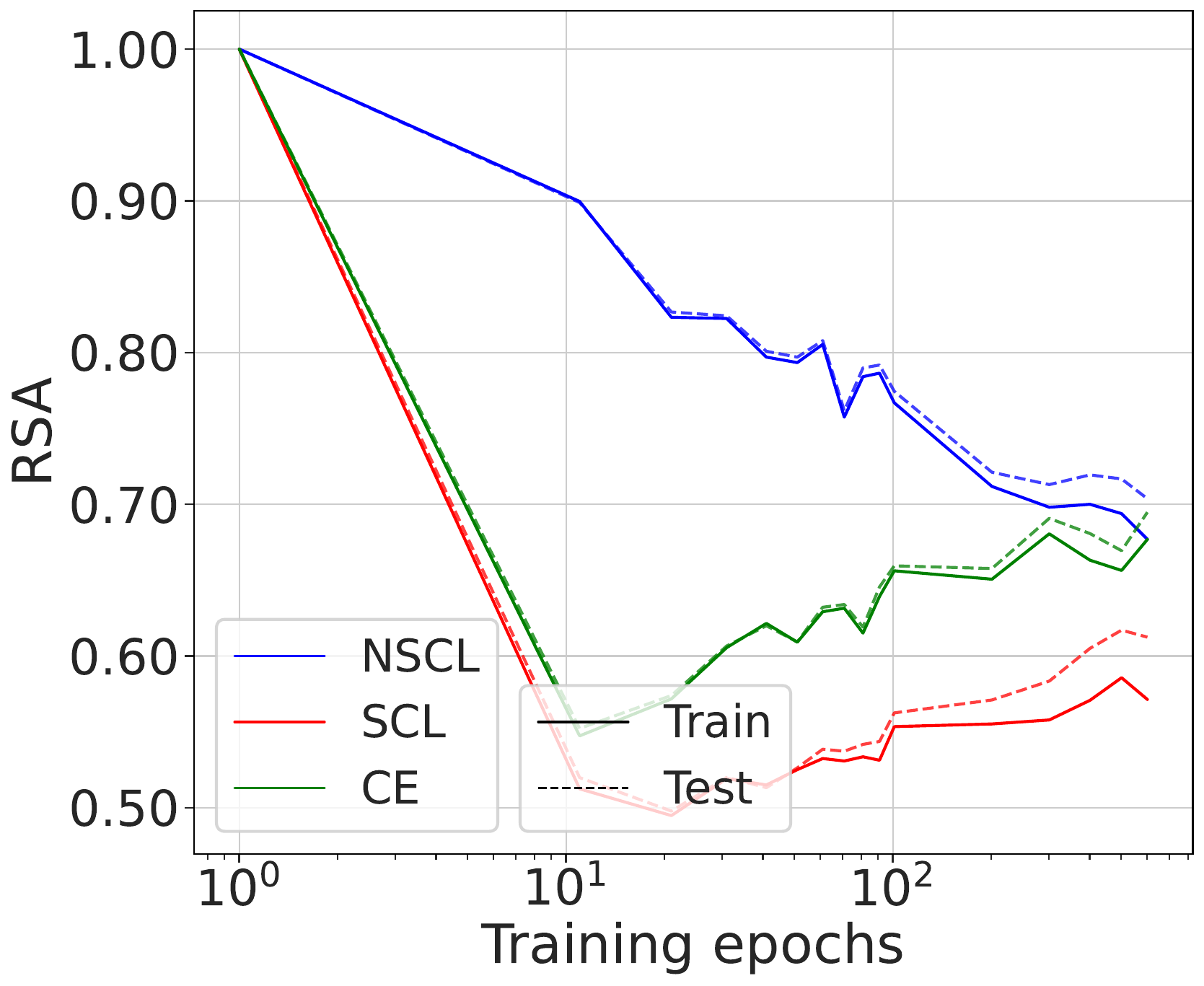} & 
         \includegraphics[width=0.235\linewidth]{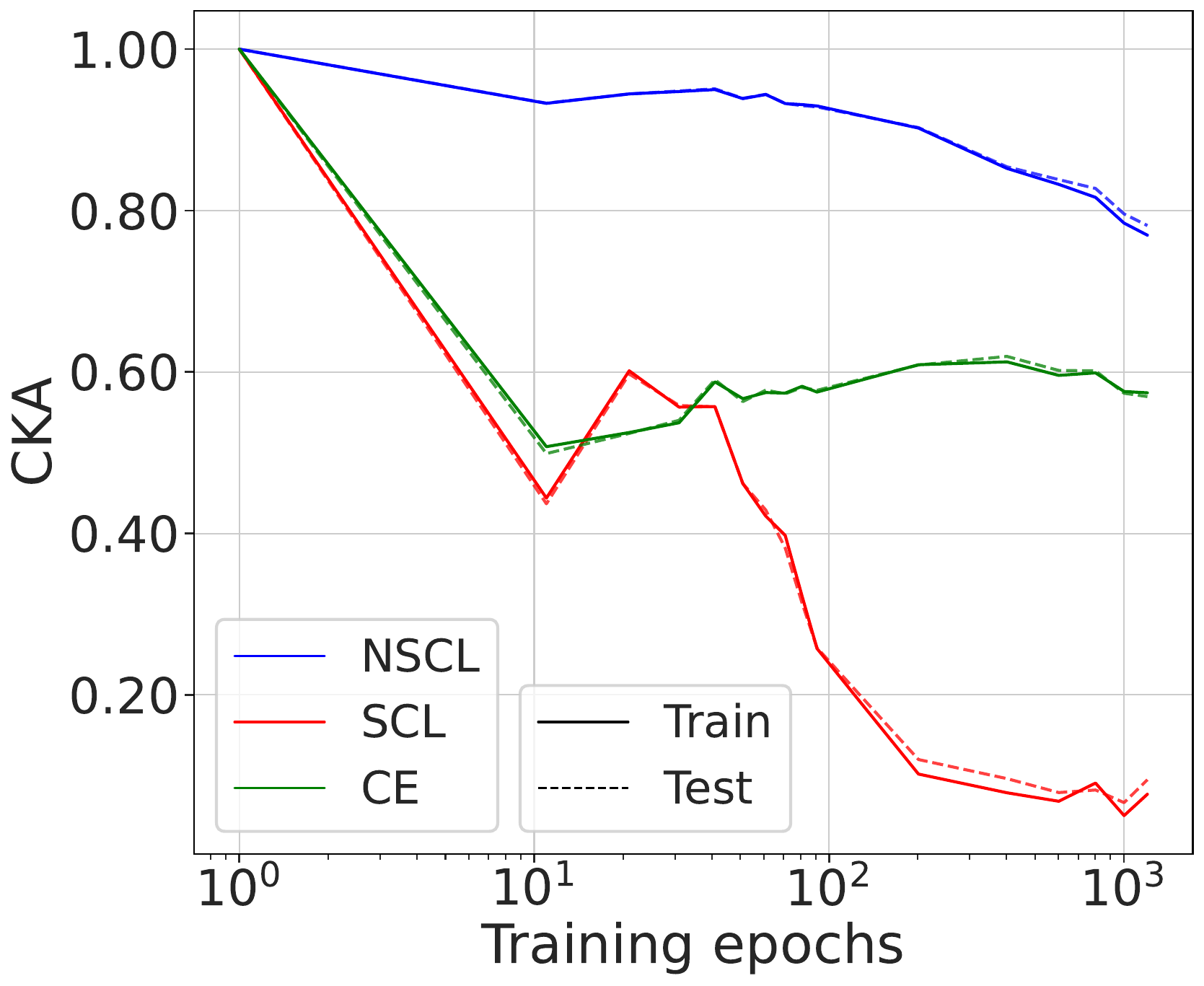} &
         \includegraphics[width=0.235\linewidth]{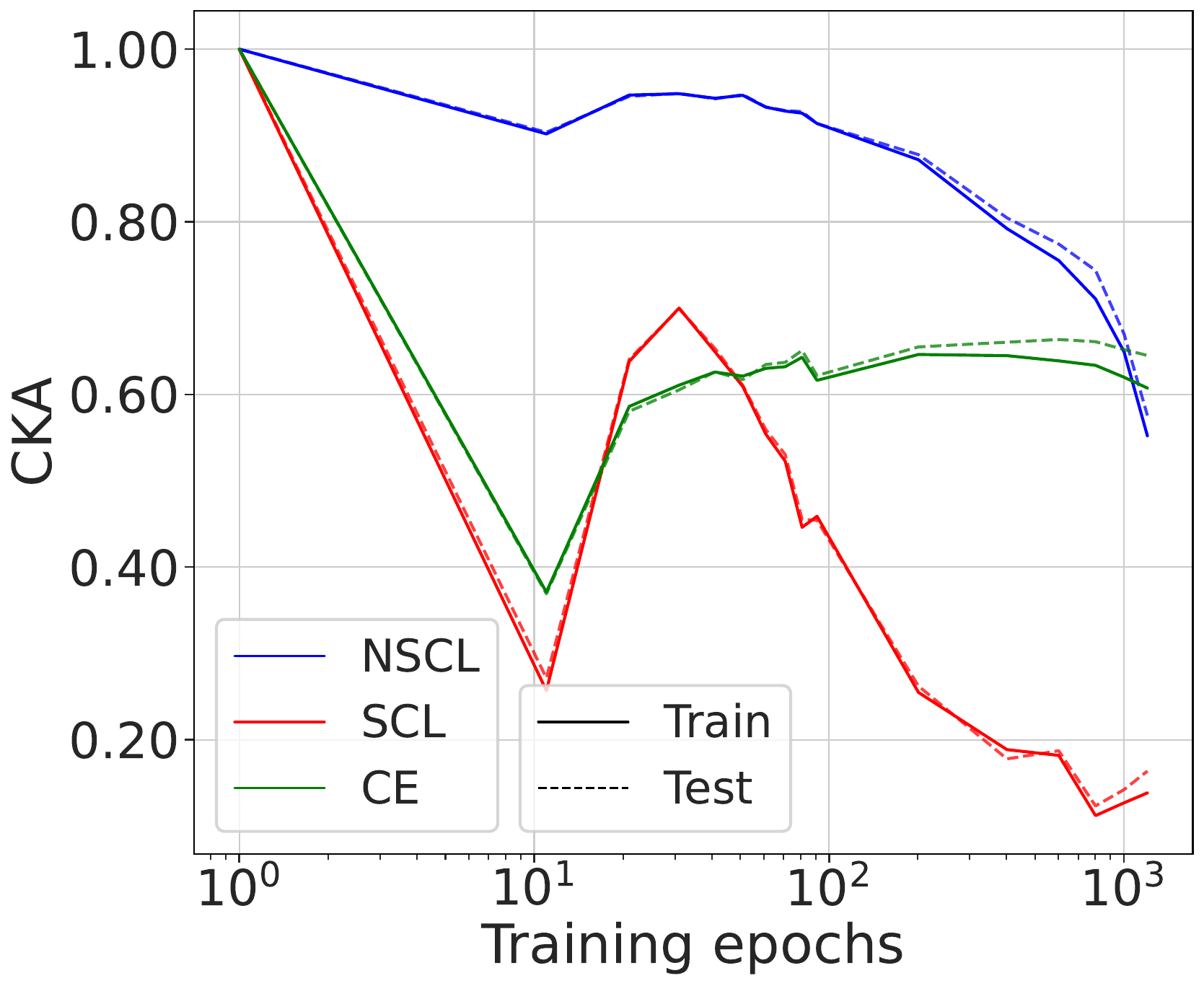} &
         \includegraphics[width=0.235\linewidth]{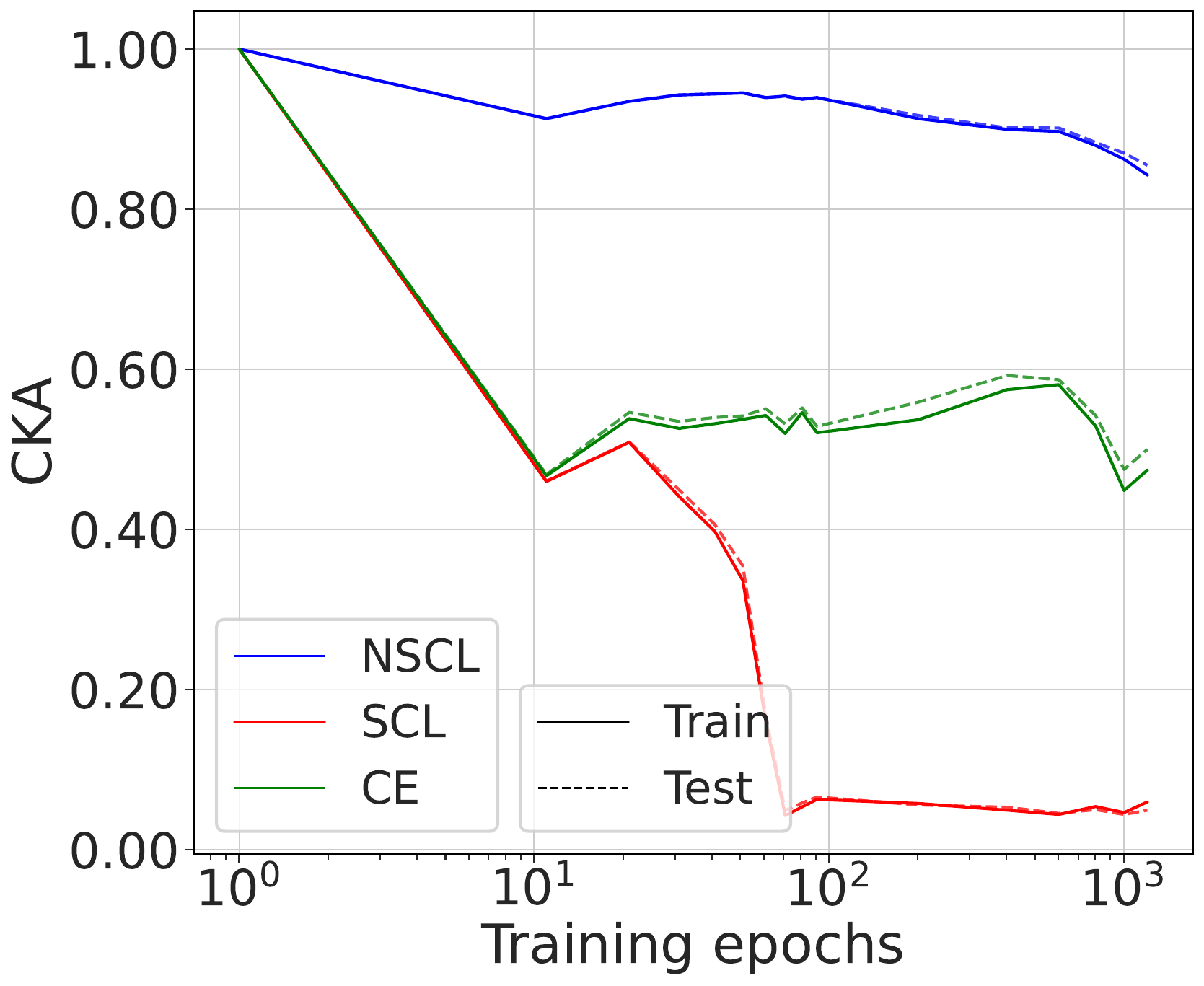} \\
         \includegraphics[width=0.235\linewidth]{figures/exp1/cifar10_rsa_plot.pdf} &
         \includegraphics[width=0.235\linewidth]{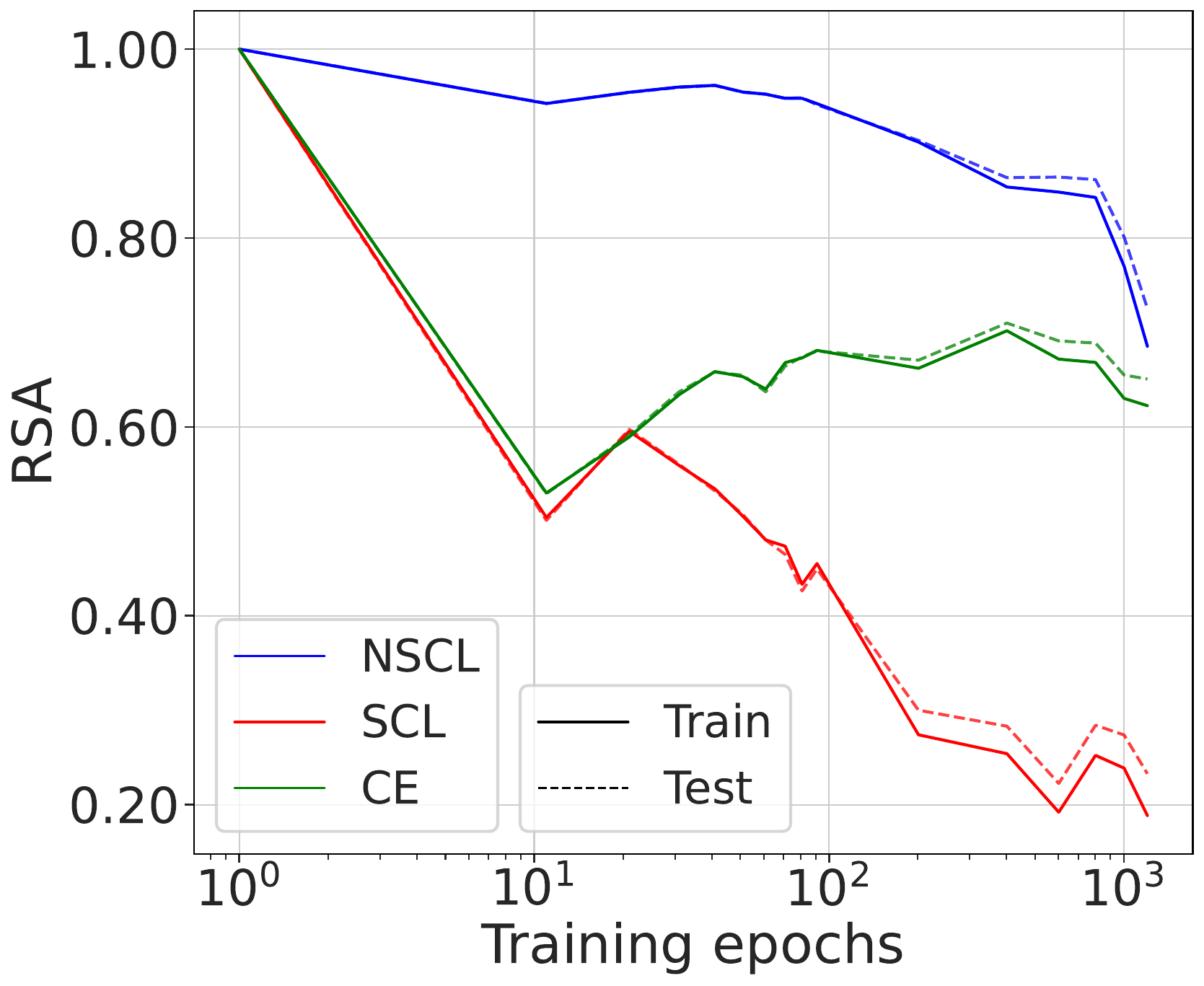} & 
         \includegraphics[width=0.235\linewidth]{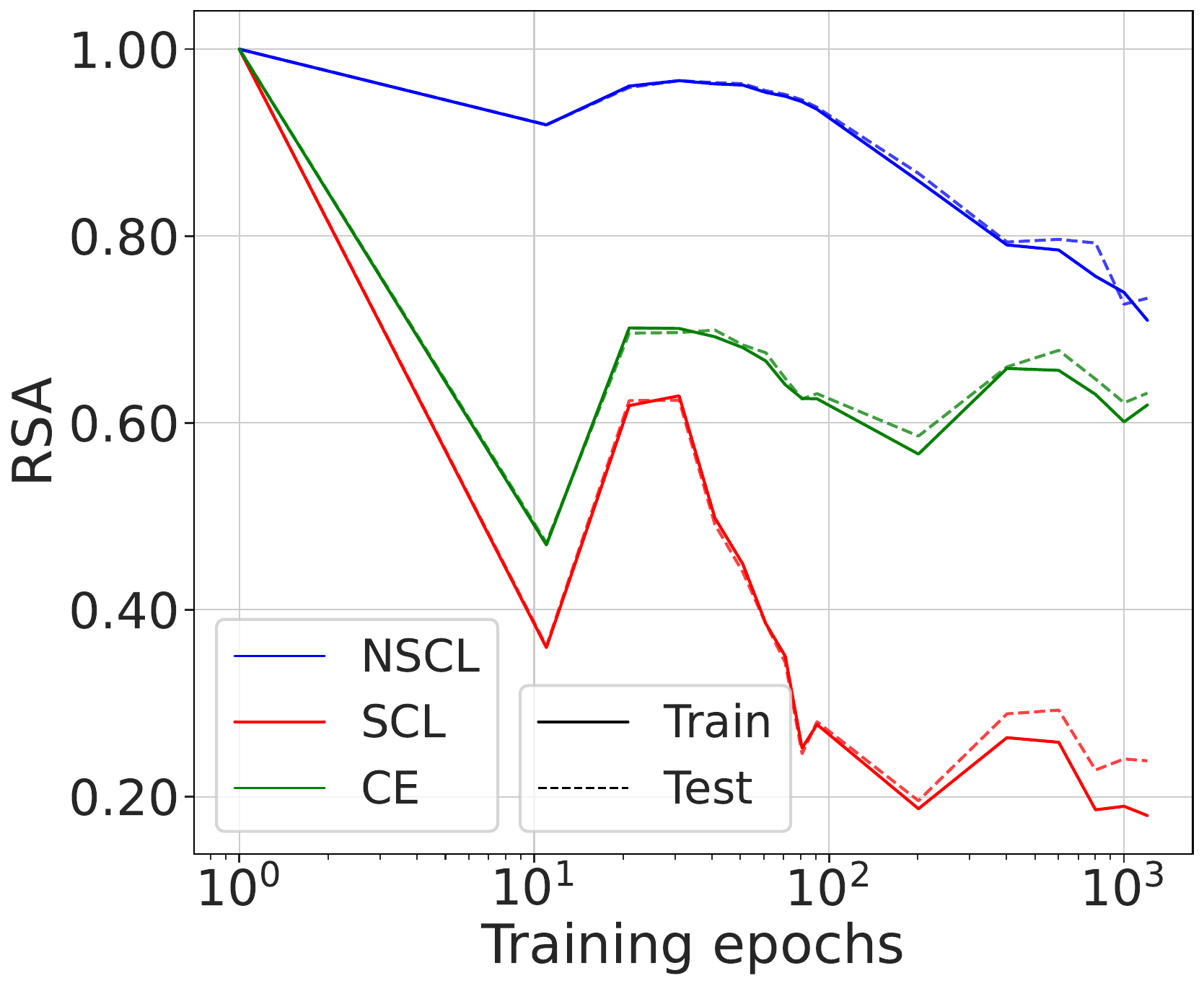} &
         \includegraphics[width=0.235\linewidth]{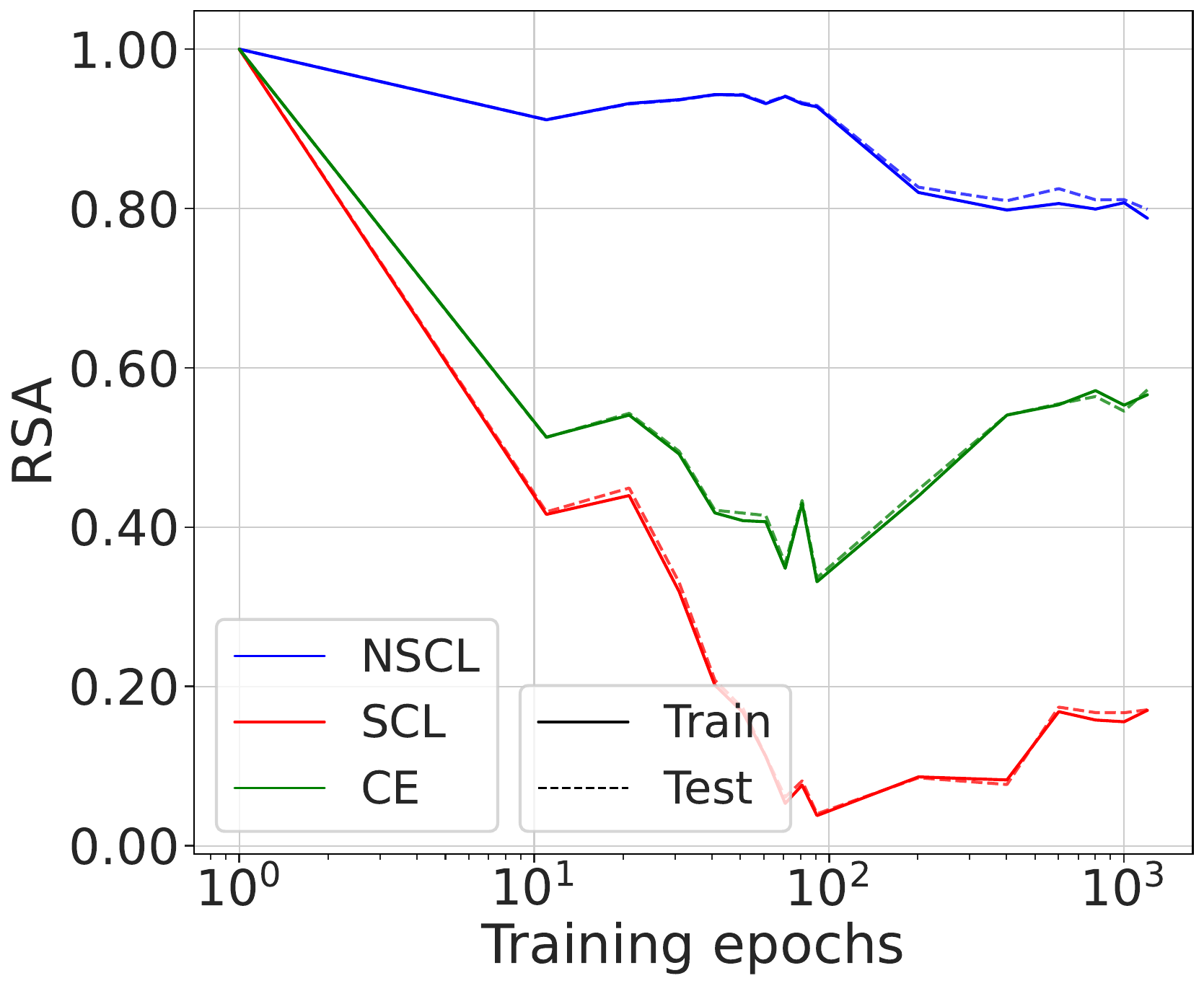} \\
         {\small {\bf (a)} CIFAR10} & {\small {\bf (b)} CIFAR100} & {\small {\bf (c)} Mini-ImageNet} & {\small {\bf (d)} Tiny-ImageNet}
    \end{tabular}   
\caption{\textbf{Alignment during training.} We train ResNet-50 models with decoupled CL, SCL, NSCL, and CE. For the first 1,000 epochs, the CL-trained model is substantially more aligned with the NSCL-trained model than with the others. However, alignment declines when training continues much longer.}
    \label{fig:alignment_epochs}
\end{figure}


\begin{figure}[t]
\centering
\setlength{\tabcolsep}{1.5pt} 
\begin{tabular}{@{}ccccc@{}}
\includegraphics[width=0.195\linewidth]{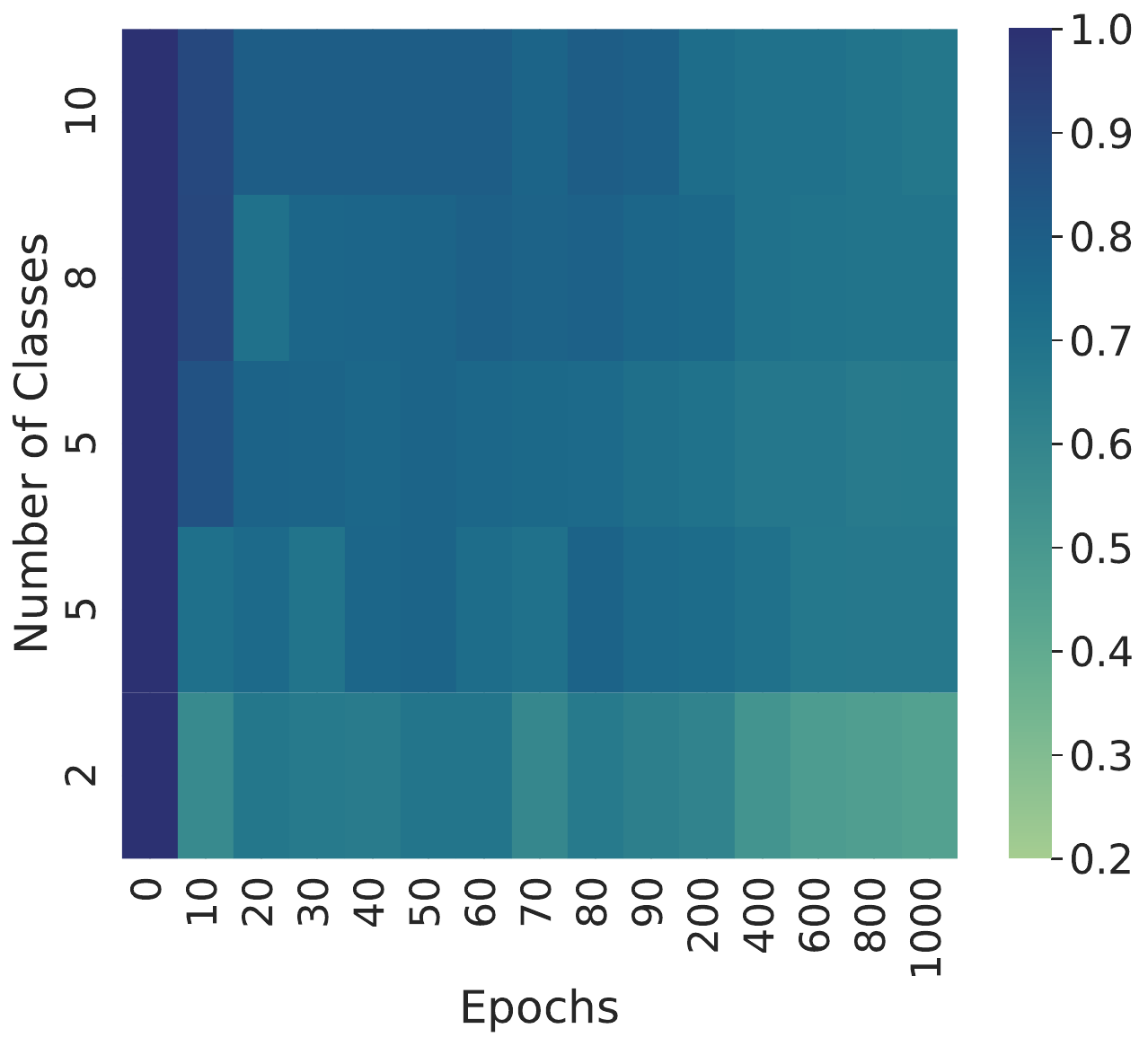} &
\includegraphics[width=0.195\linewidth]{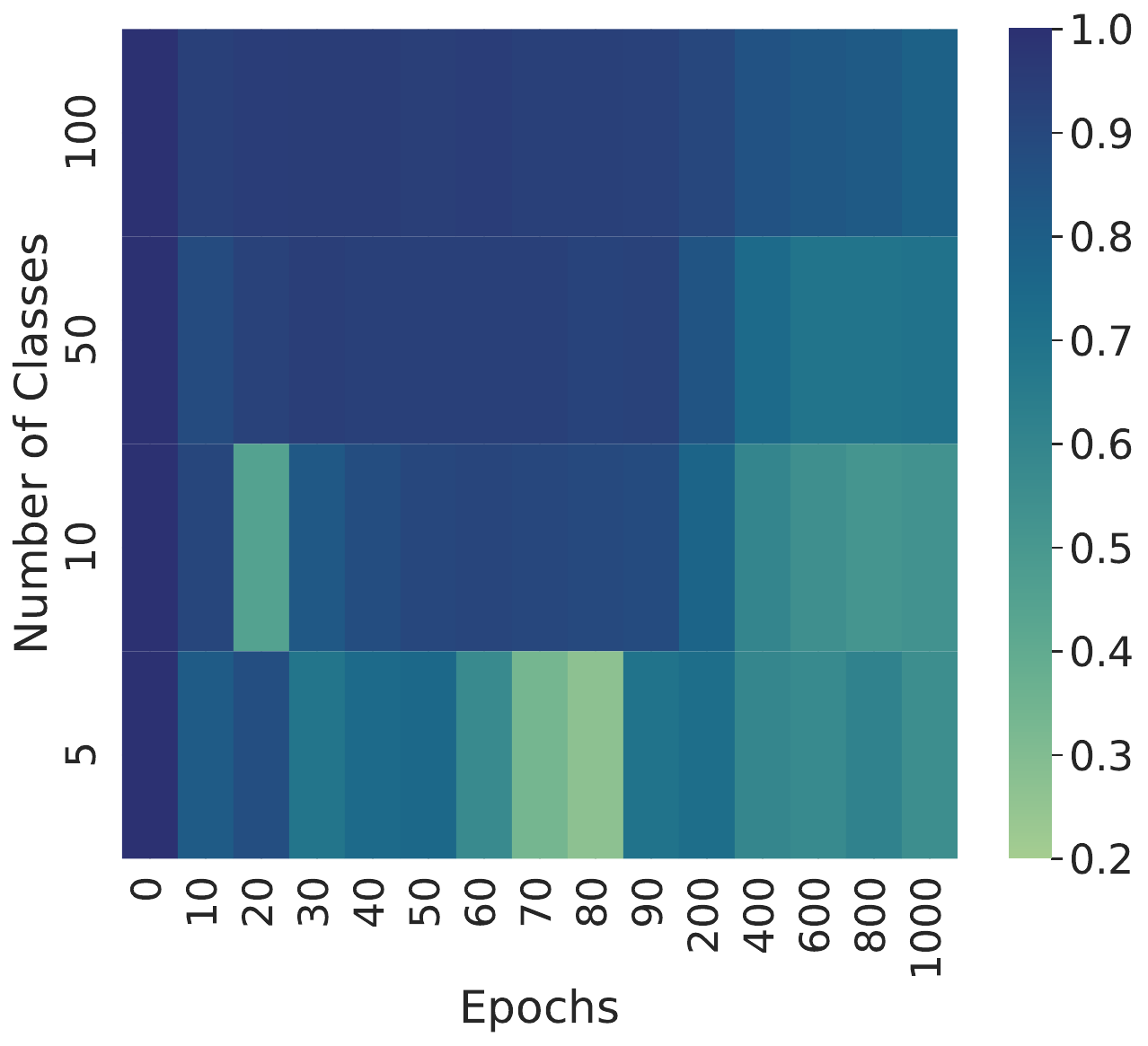} &
\includegraphics[width=0.195\linewidth]{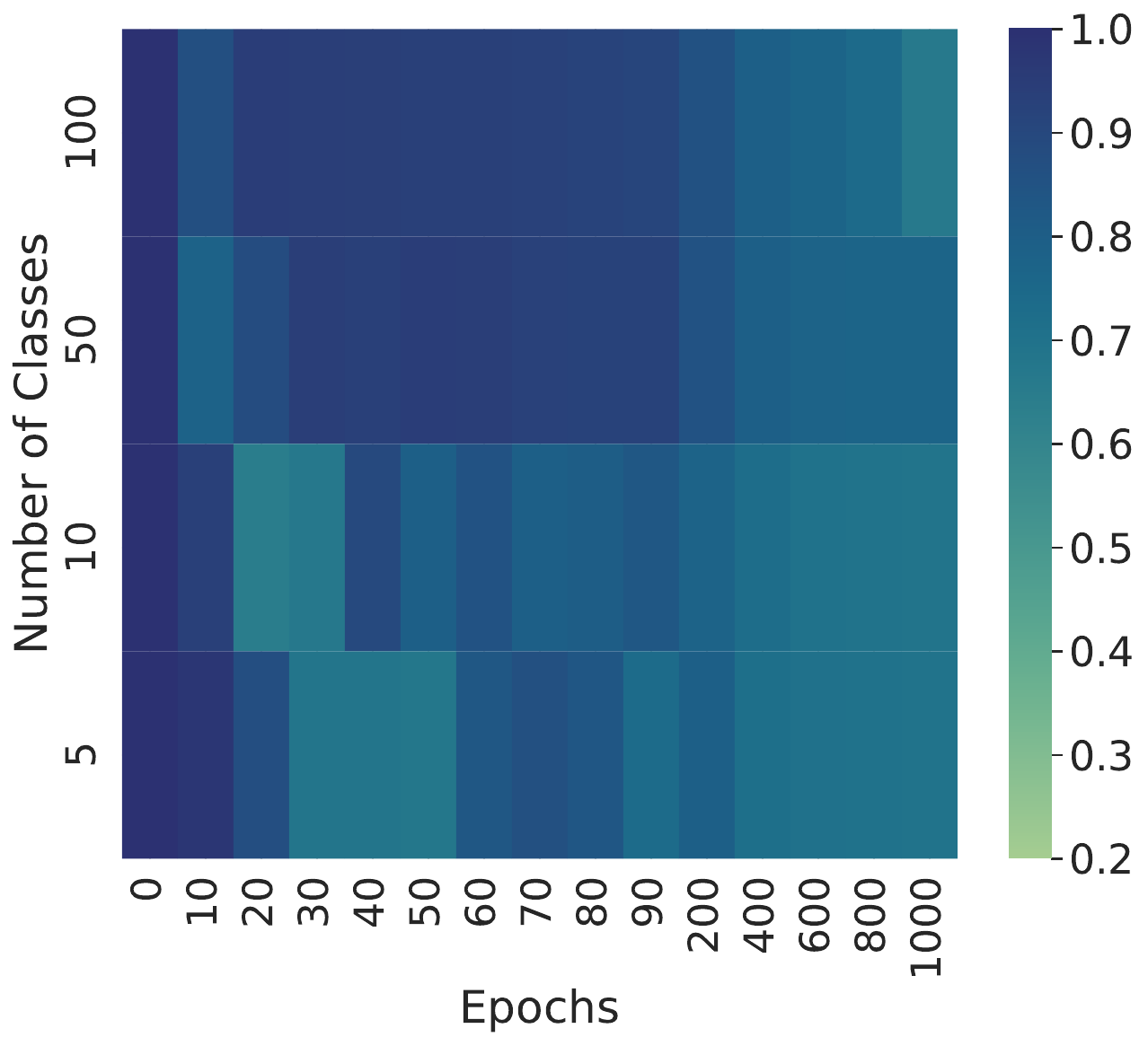} &
\includegraphics[width=0.195\linewidth]{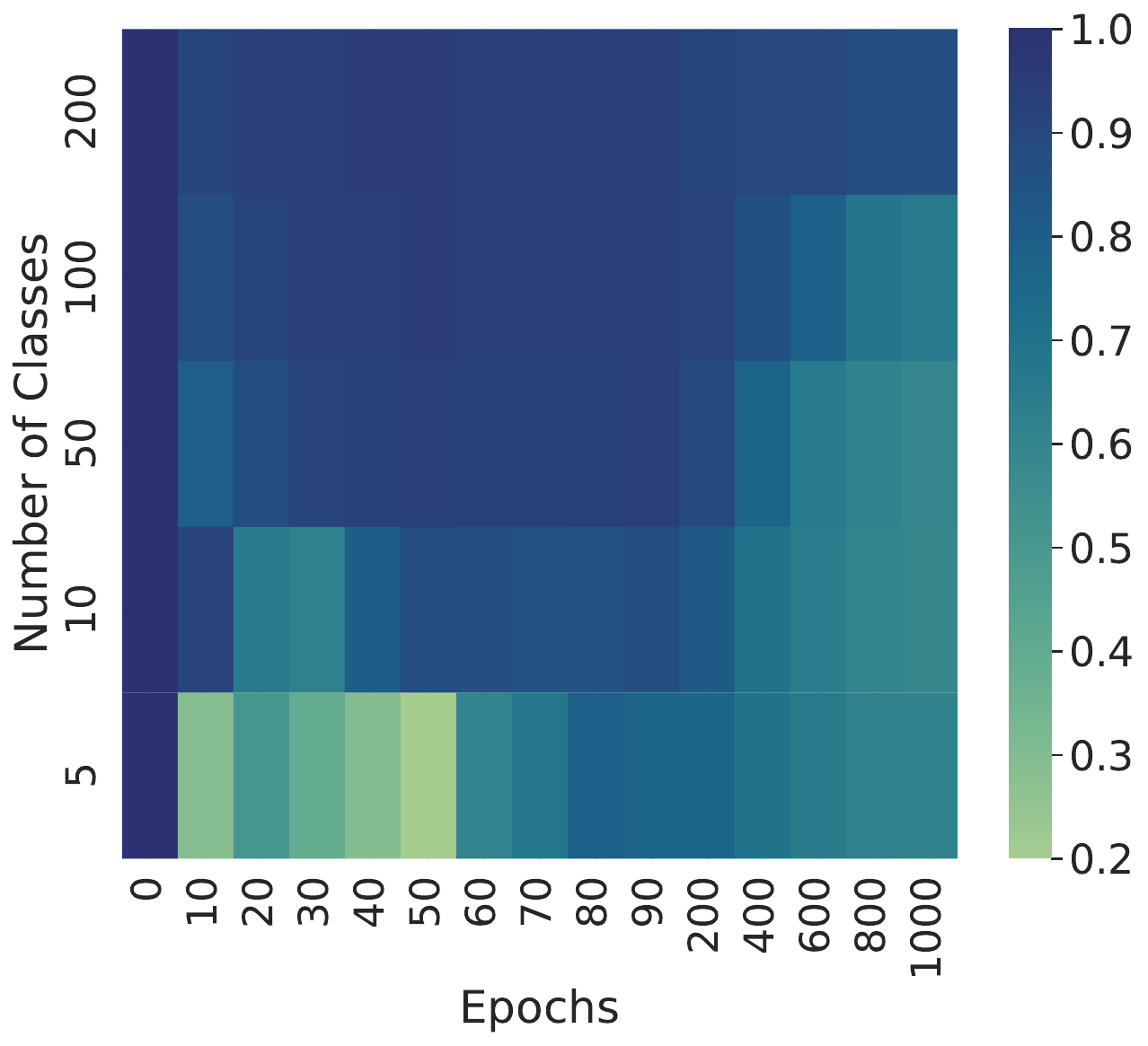} &
\includegraphics[width=0.195\linewidth]{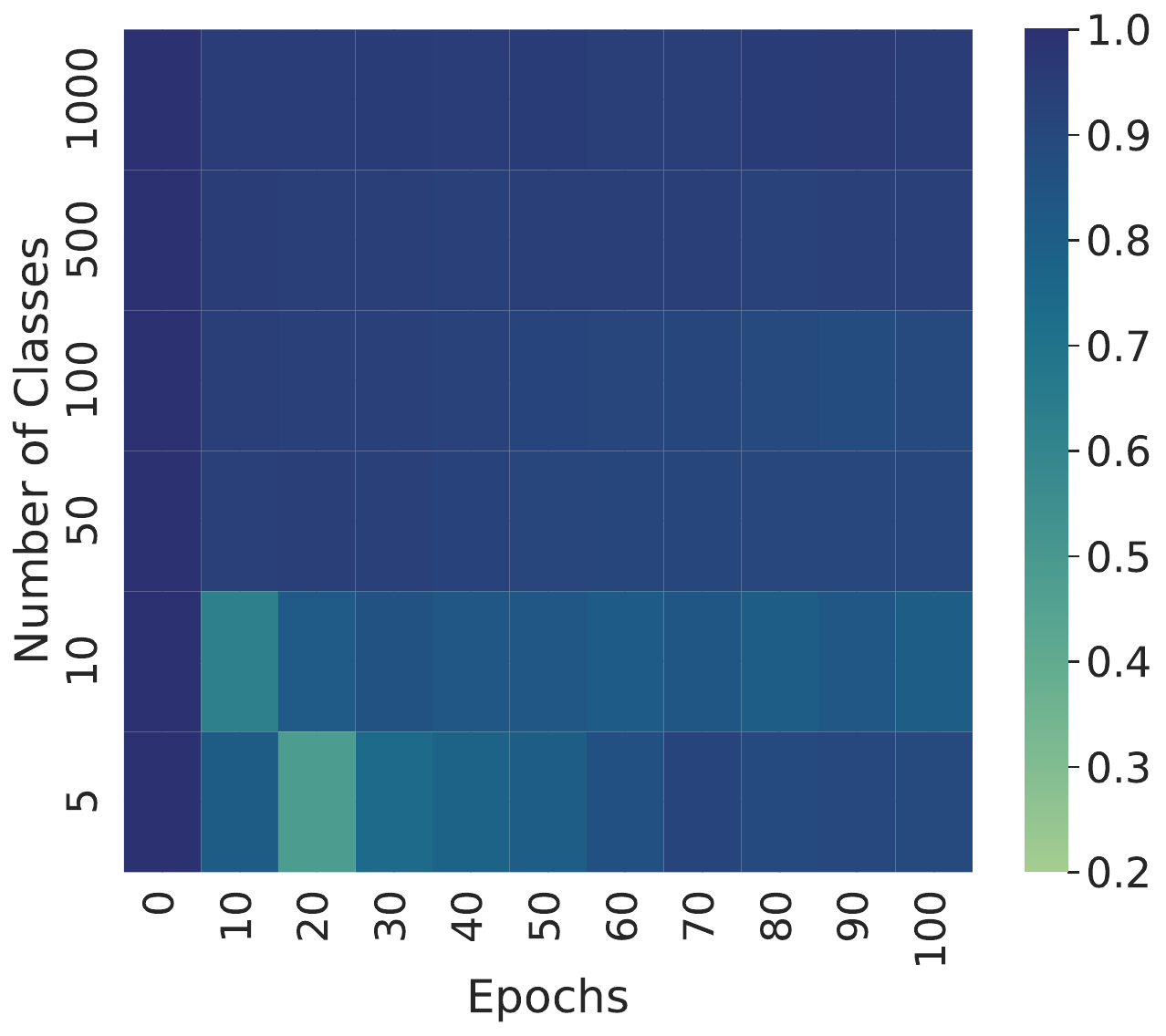} \\
\includegraphics[width=0.195\linewidth]{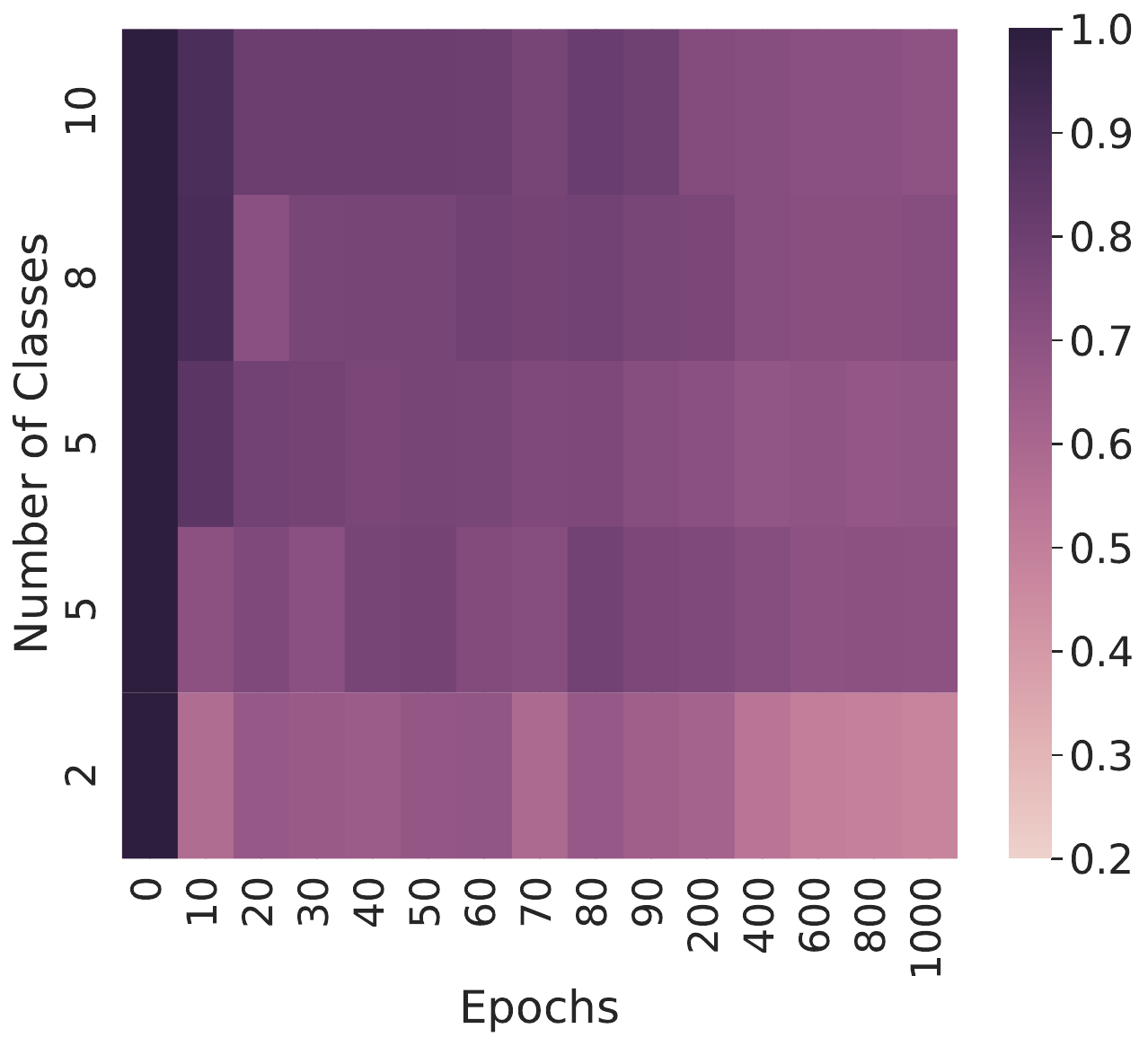} &
\includegraphics[width=0.195\linewidth]{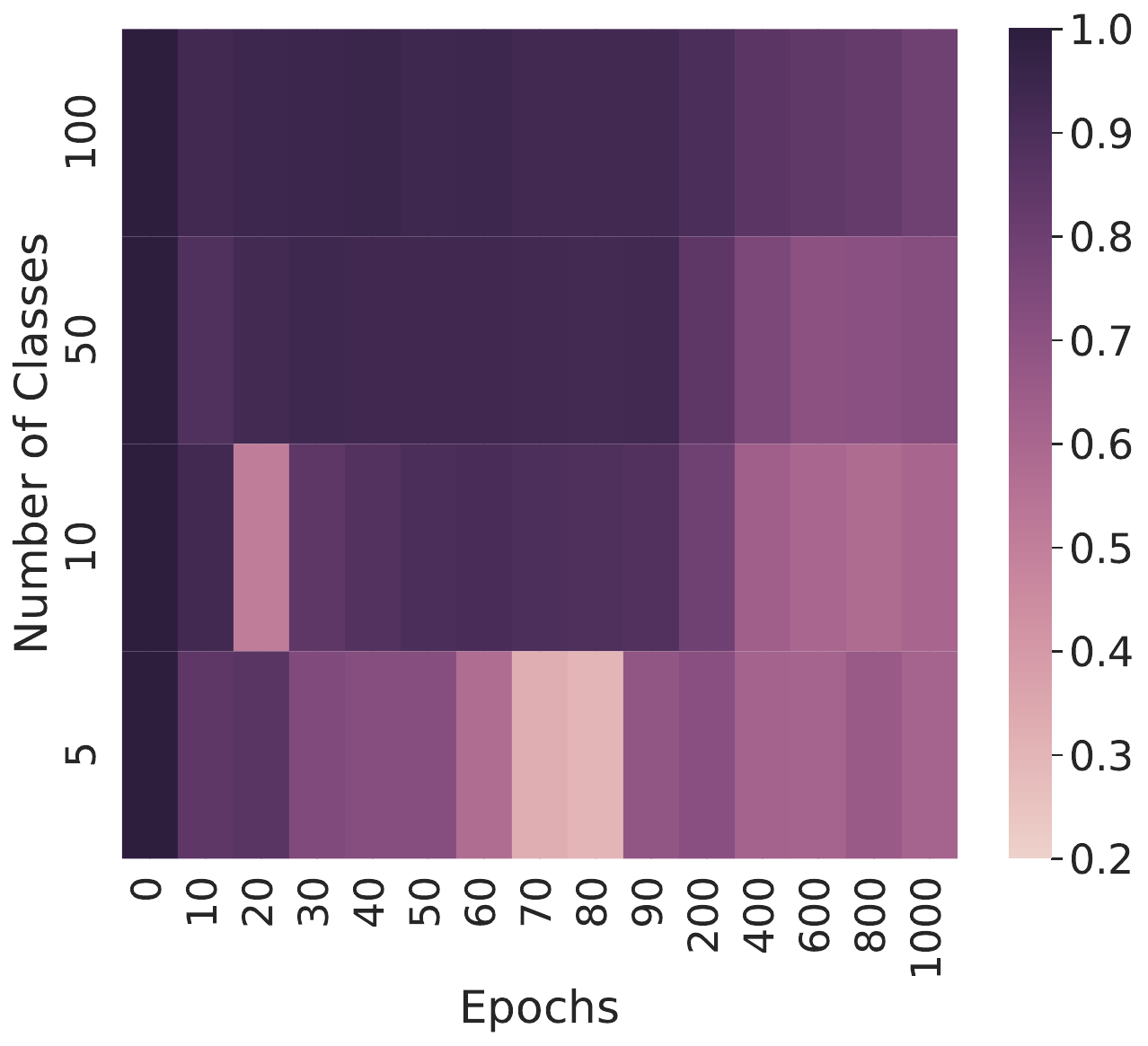} &
\includegraphics[width=0.195\linewidth]{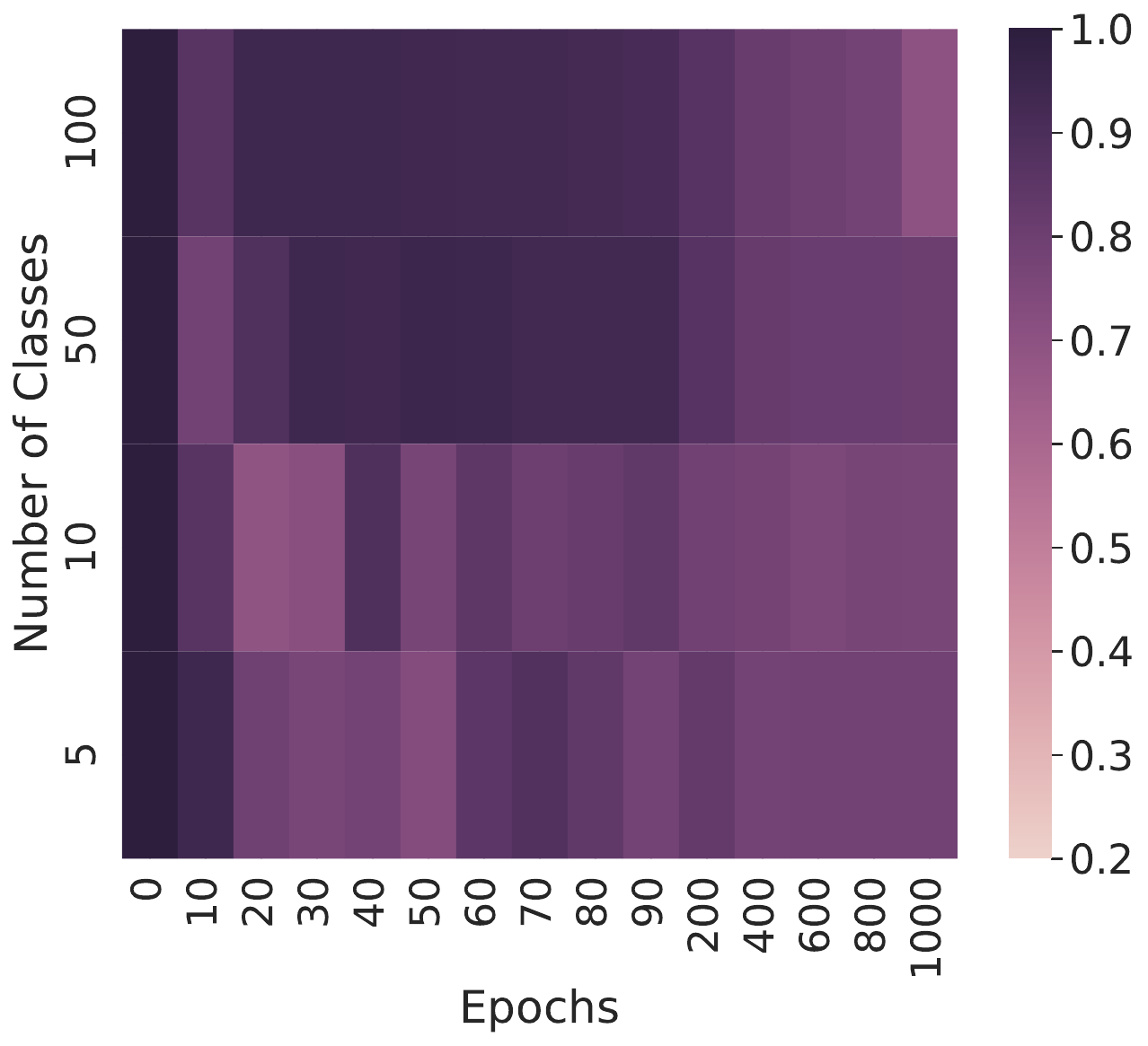} &
\includegraphics[width=0.195\linewidth]{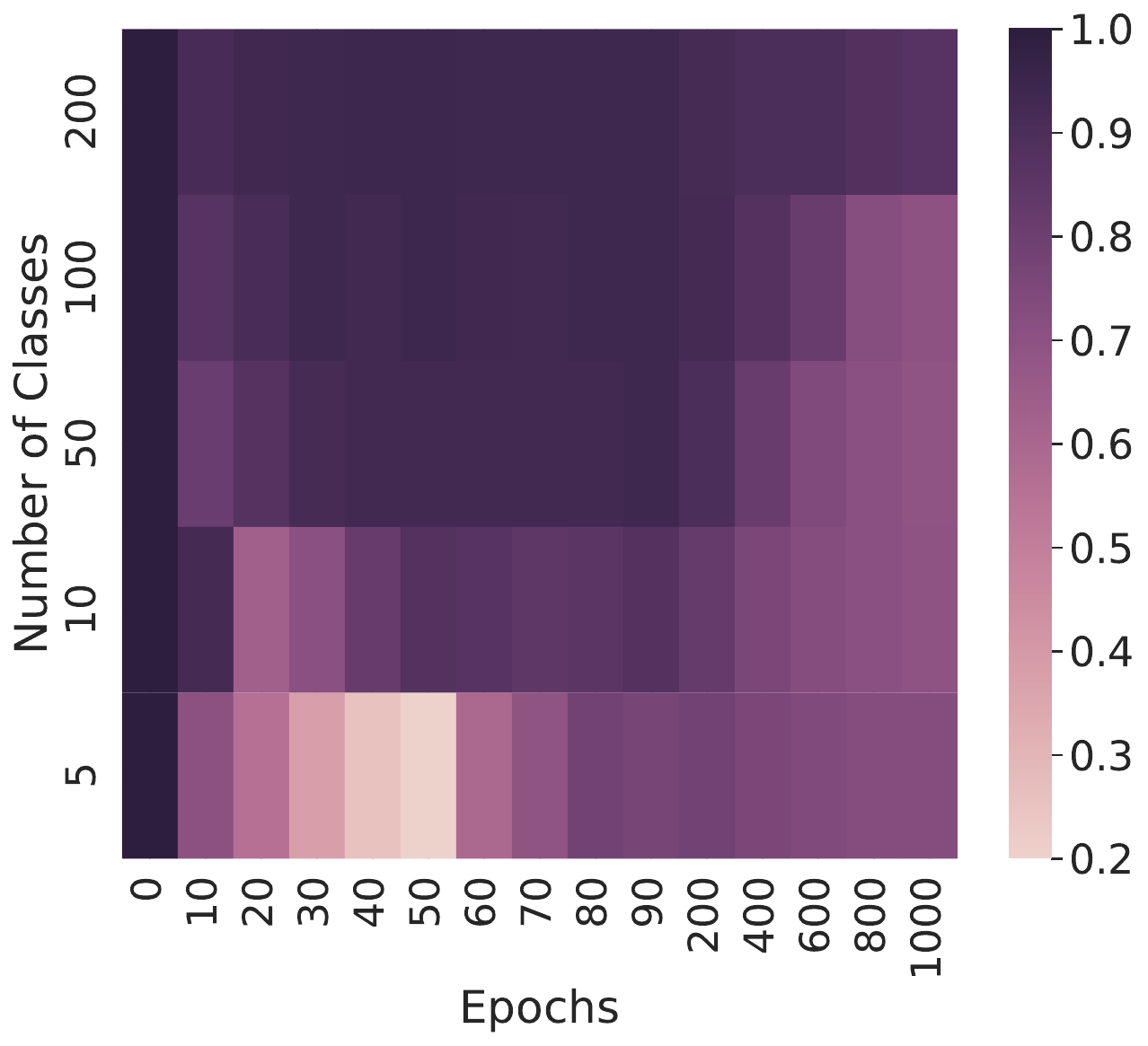} &
\includegraphics[width=0.195\linewidth]{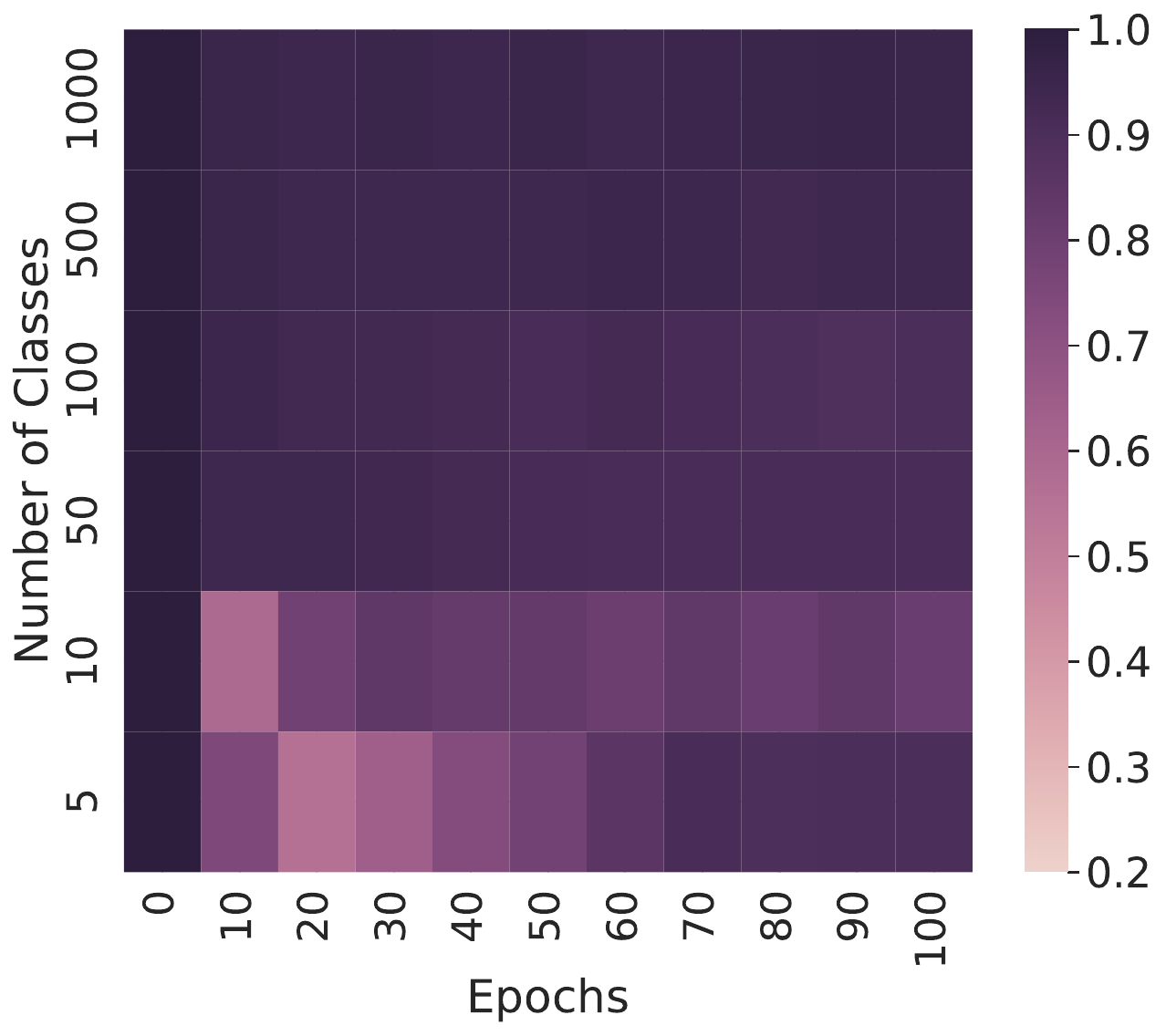} \\
{\small CIFAR-10} & {\small CIFAR-100} & {\small Mini-ImageNet} & {\small Tiny-ImageNet} & {\small ImageNet-1k}
\end{tabular}
\vspace{-0.4em}
  \caption{{\bf CL–NSCL alignment (linear CKA) increases with the number of training classes.} The heatmaps show the linear CKA between CL and NSCL models. For each dataset, we visualize alignment on the training (top row, green) and test (bottom row, purple) sets. The y-axis indicates the number of classes ($N$) used for training, and the x-axis represents the training epoch. While alignment is consistently higher for larger $N$, it also tends to decrease as training progresses for any fixed $N$.}
  \label{fig:alignment_classes}
\end{figure}

\begin{table}[t]
    \centering
    \setlength{\tabcolsep}{8pt} 
    \renewcommand{\arraystretch}{1.15} 
    \begin{tabular}{l|cc|cc|cc|cc}
        \toprule
        & \multicolumn{2}{c|}{CIFAR-10} & \multicolumn{2}{c|}{CIFAR-100} & \multicolumn{2}{c|}{Mini-ImageNet} & \multicolumn{2}{c}{Tiny-ImageNet} \\
        \cmidrule(lr){2-3} \cmidrule(lr){4-5} \cmidrule(lr){6-7} \cmidrule(lr){8-9}
        & NCCC & LP & NCCC & LP & NCCC & LP & NCCC & LP \\
        \midrule
        CL   & 88.37 & 90.16 & 54.62 & 65.65 & 60.78 & 65.30 & 40.59 & 44.61 \\
        NSCL & 94.47 & 94.09 & 60.14 & 68.38 & 63.92 & 72.60 & 40.76 & 45.79 \\
        SCL  & 94.93 & 94.67 & 64.06 & 69.52 & 74.78 &   76.00  & 48.63 & 48.73 \\
        CE   & 92.97 & 93.39 & 67.35 & 68.04 & 75.20 &   74.00  & 48.28 & 52.57 \\
        \bottomrule
    \end{tabular}
    \caption{Nearest Class-Center Classifier (NCCC) and Linear Probe (LP) test accuracies (\%). We report the accuracies against the all-way classification task in each dataset. The models (also used in Fig.~\ref{fig:alignment_epochs}) were pre-trained on their respective datasets: CIFAR10, CIFAR100, Mini-ImageNet and Tiny-ImageNet.}
    \label{tab:acc_full_test_only}
\end{table}

\begin{figure}[t]
    \centering
    \begin{tabular}{c@{\hspace{6pt}}c@{\hspace{6pt}}c@{\hspace{6pt}}c}
          \includegraphics[width=0.235\linewidth]{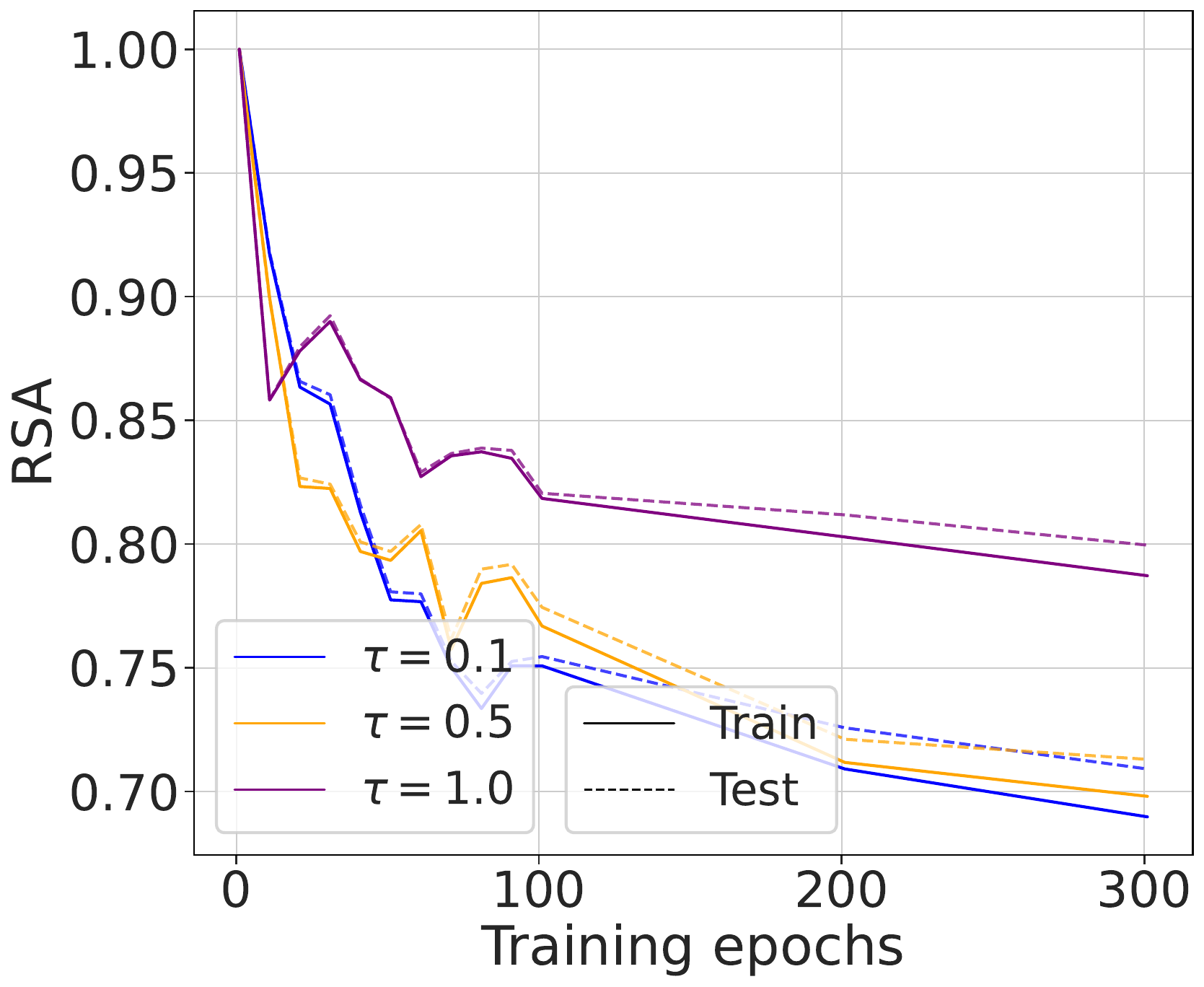} &
          \includegraphics[width=0.235\linewidth]{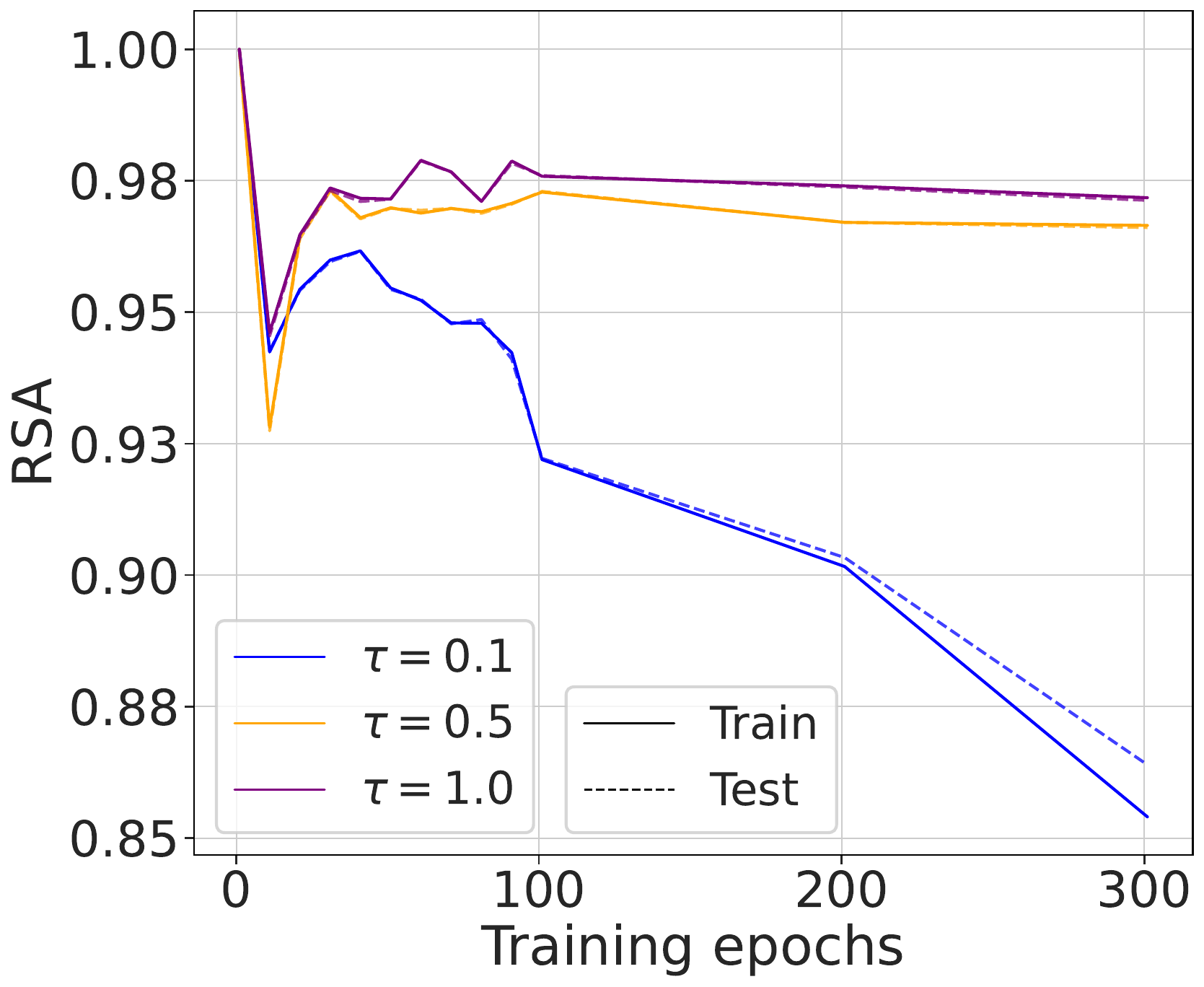} &
          \includegraphics[width=0.235\linewidth]{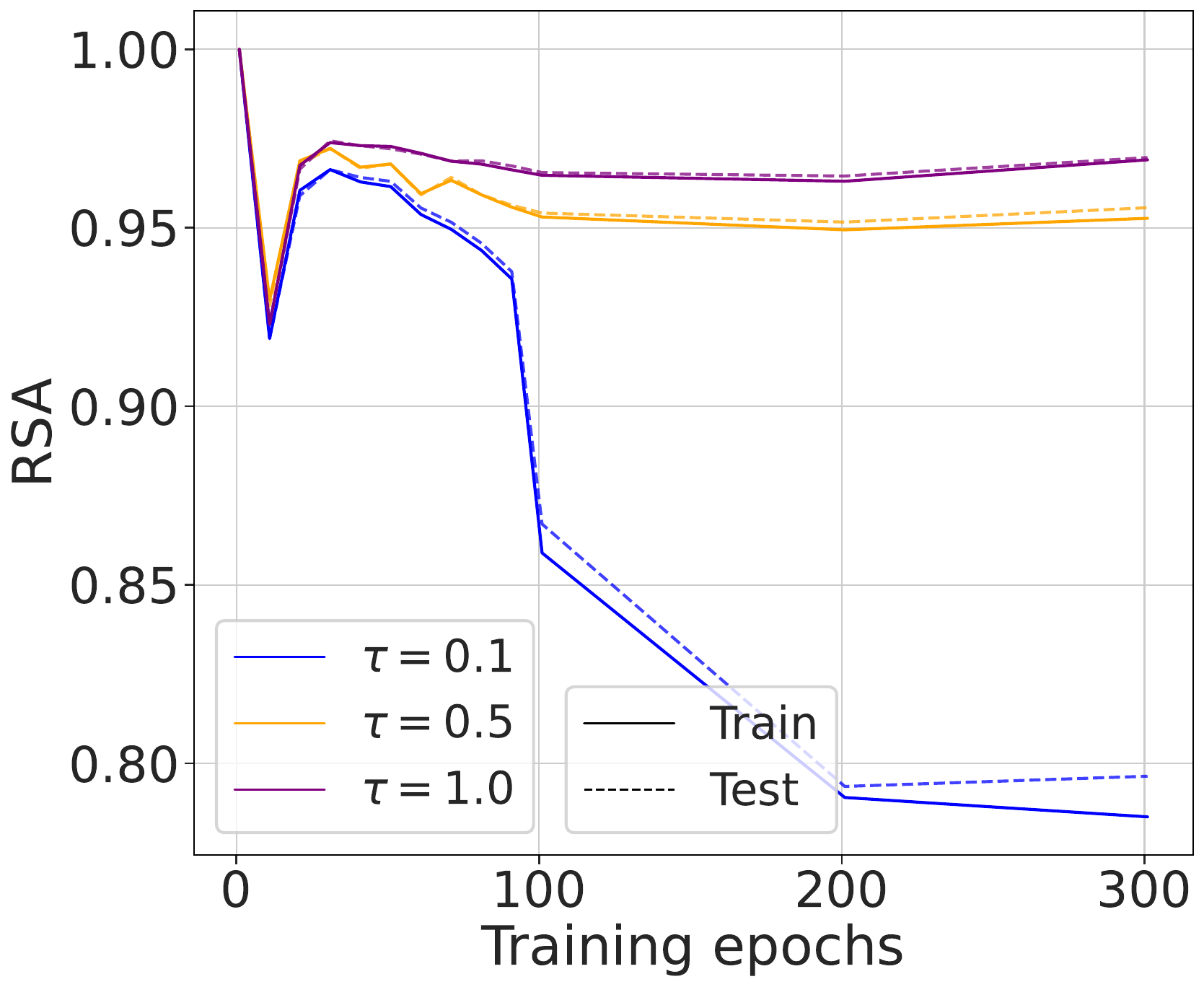} &
          \includegraphics[width=0.235\linewidth]{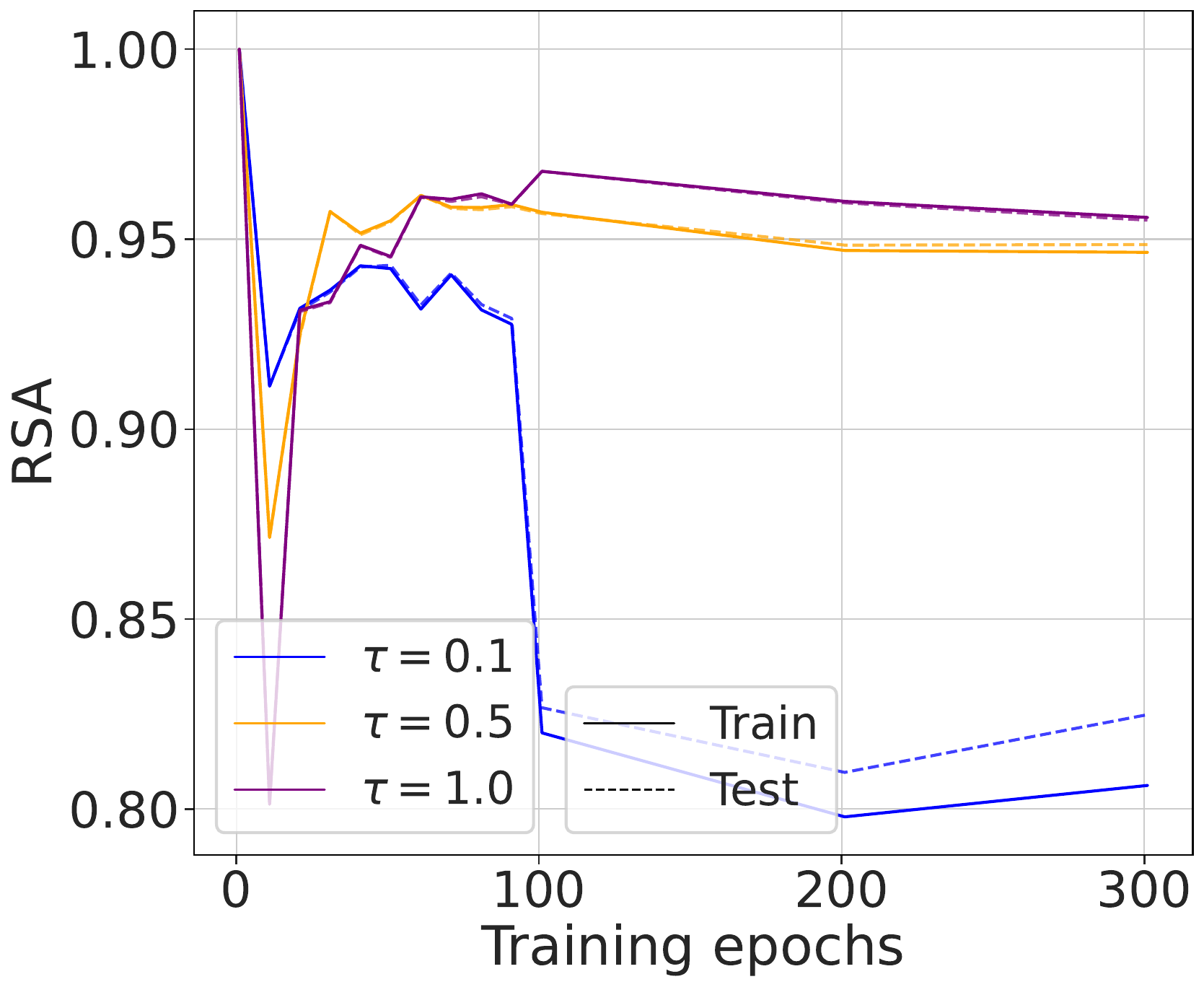} \\
          \includegraphics[width=0.235\linewidth]{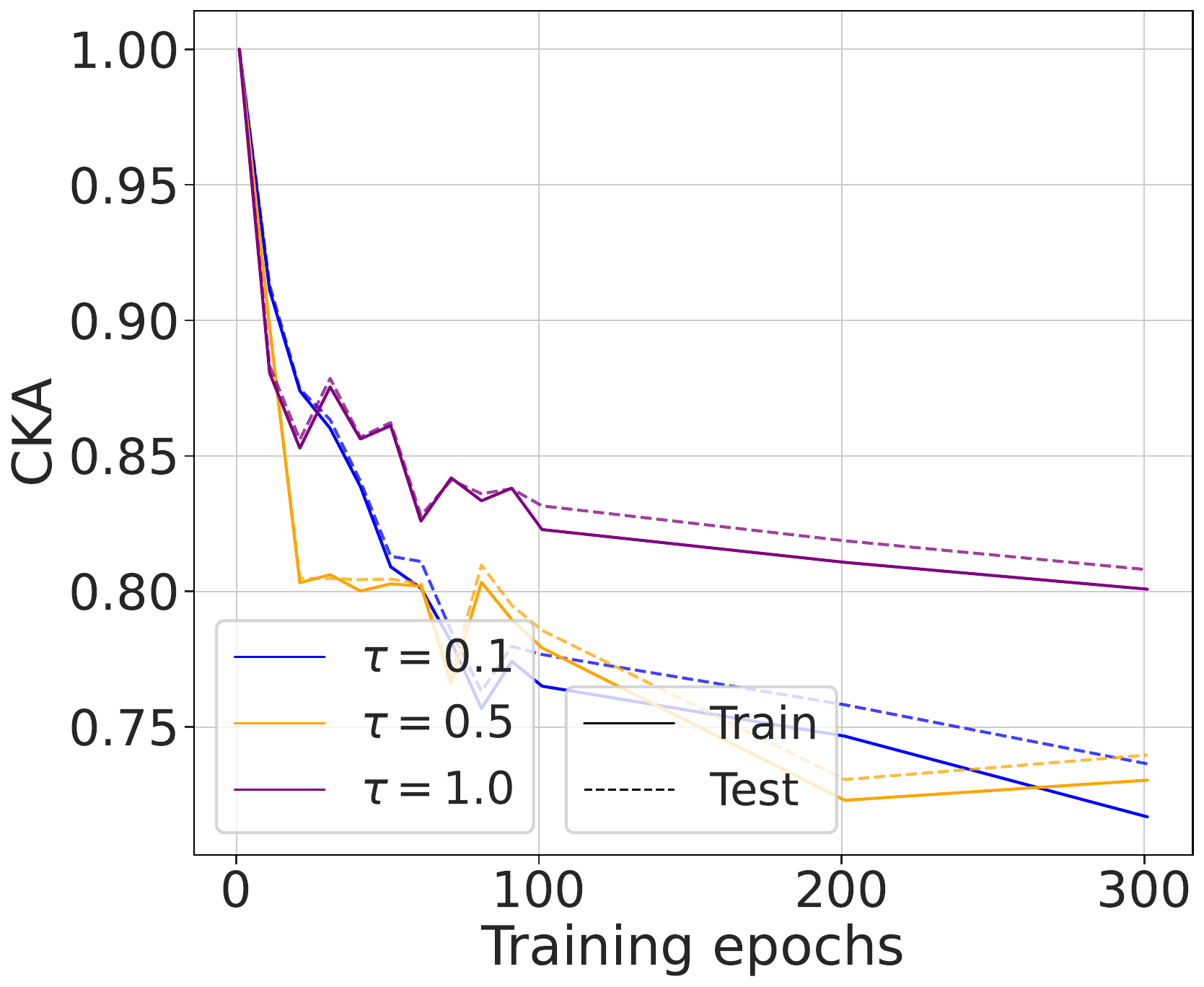} &
          \includegraphics[width=0.235\linewidth]{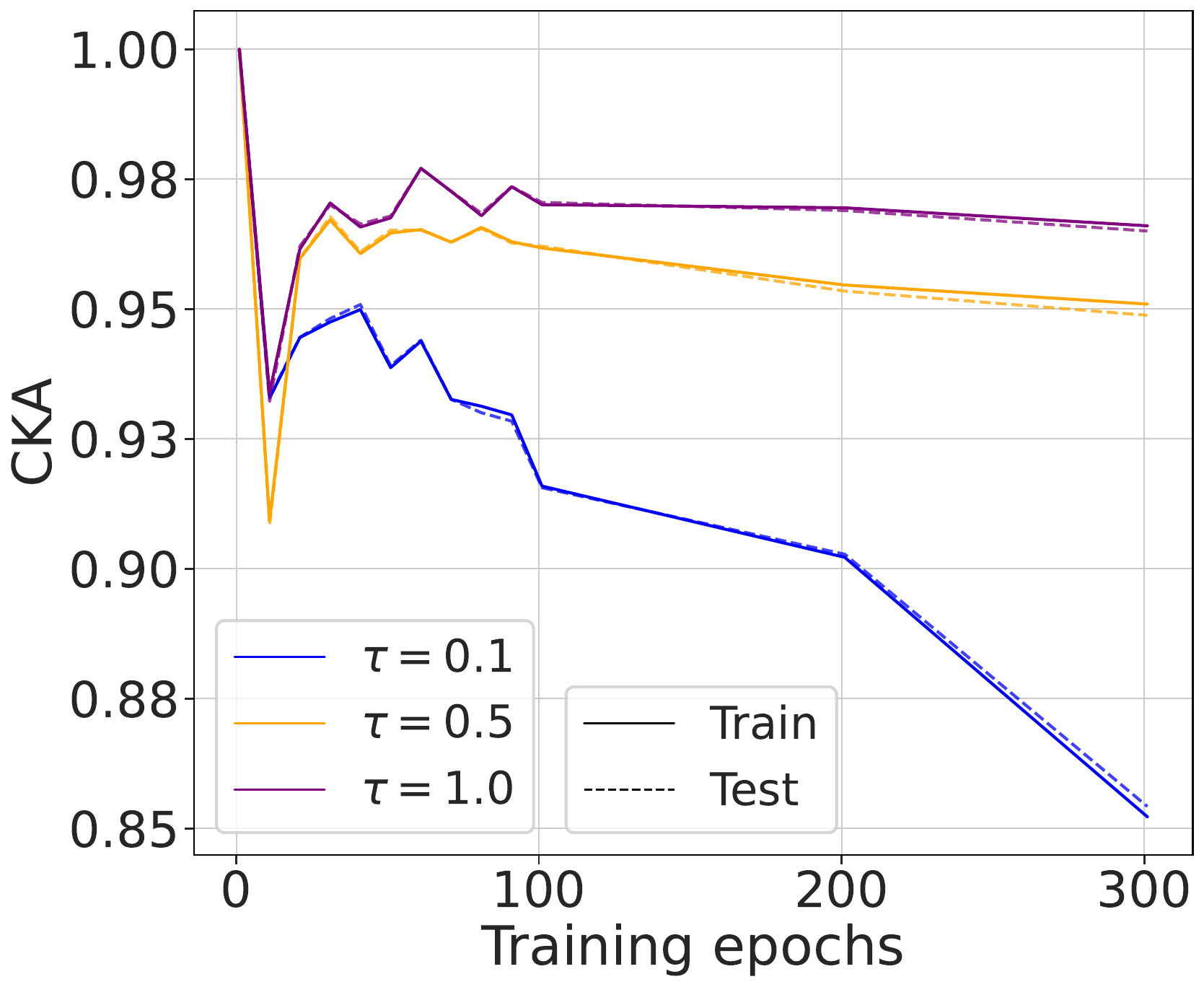} &
          \includegraphics[width=0.235\linewidth]{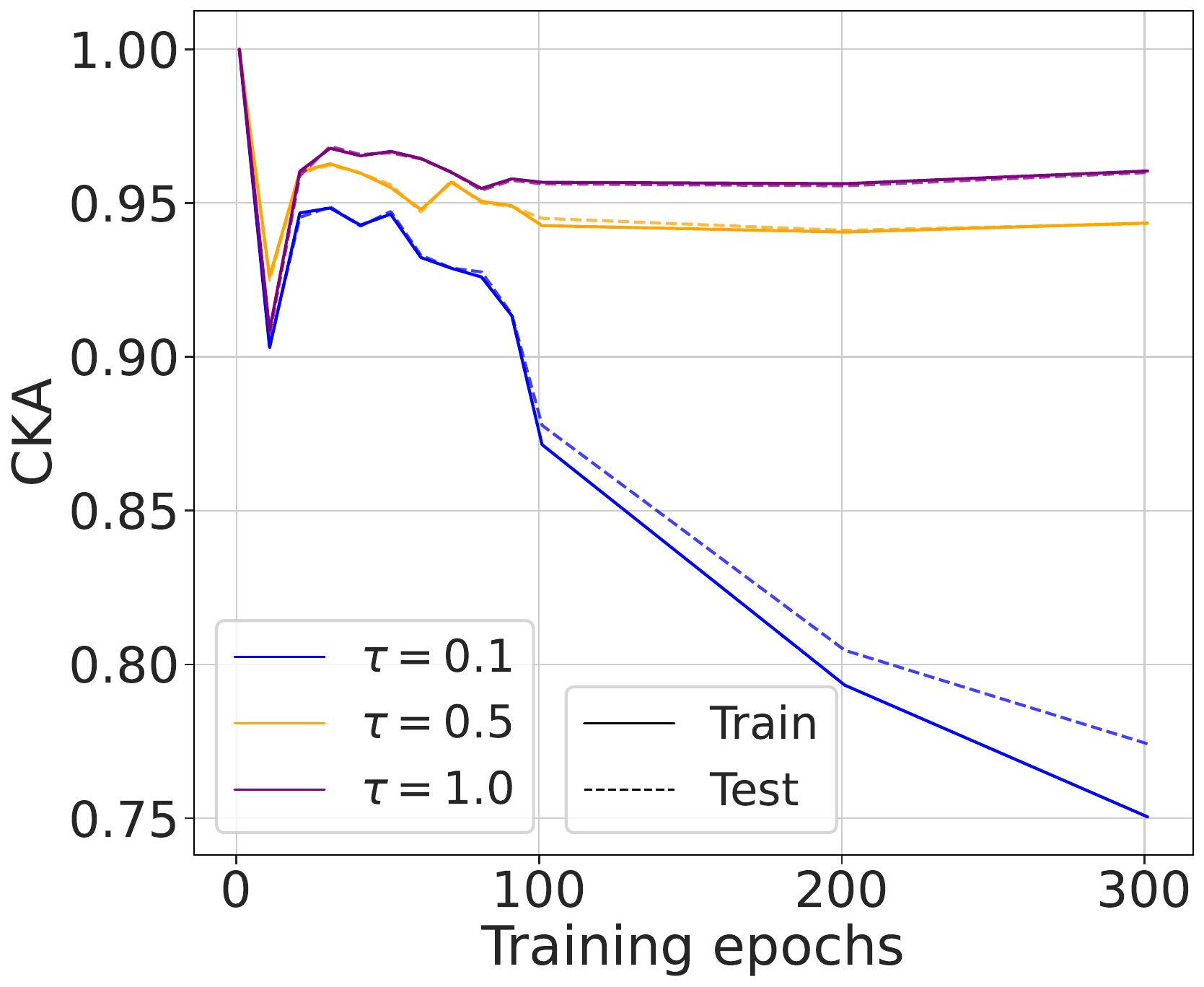} &
          \includegraphics[width=0.235\linewidth]{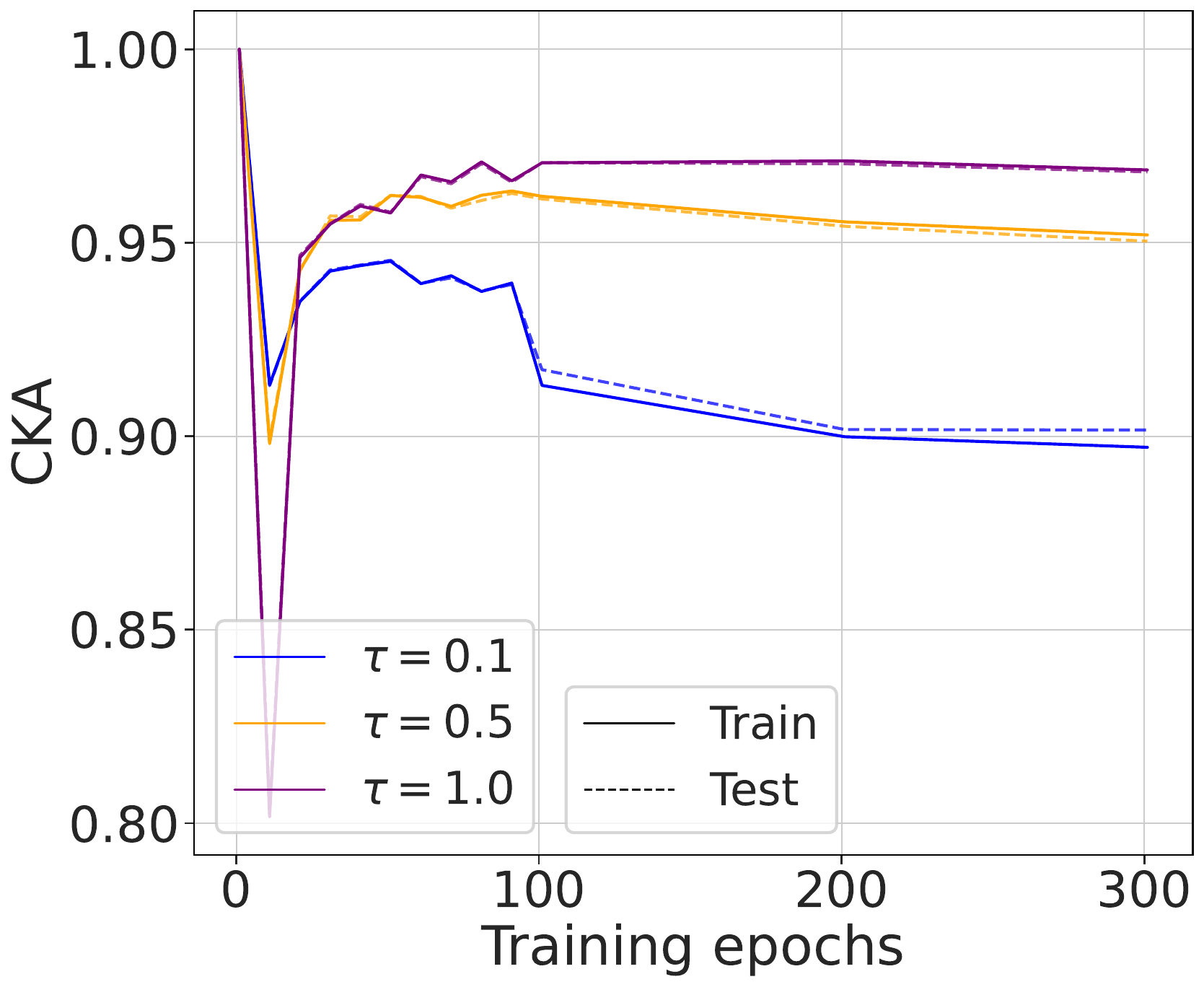} \\
          {\small {\bf (a)} CIFAR10} & {\small {\bf (b)} CIFAR100} & {\small {\bf (c)} Mini-ImageNet} & {\small {\bf (d)} Tiny-ImageNet}
    \end{tabular}
    \caption{{\bf Higher $\tau$ increases the CL-NSCL alignment.} The plots show RSA (top row) and CKA (bottom row) over 300 epochs. We trained CL and NSCL models with varying temperatures ($\tau \in \{ 0.1, 0.5, 1.0\}$) on four datasets. Across all datasets, a higher temperature $\tau = 1.0$ (shown in purple) evidently results in the highest alignment.} 
    \label{fig:alignment_tau}
\end{figure}

{\bf Alignment analysis as a function of epochs. \enspace} To understand how representational similarity evolves, we trained a model with a CL objective and monitored its alignment (via CKA/RSA) against supervised models trained with NSCL, CE, and SCL. We find that NSCL consistently achieves the highest alignment with CL throughout training across multiple datasets compared to CE and SCL (see Fig.~\ref{fig:alignment_epochs}). For example, after 1k epochs on Tiny-ImageNet, the CL-NSCL alignment reaches a CKA of $0.87$, in contrast to just $0.043$ for CL-SCL.  

Intuitively, this behavior stems from the geometric constraints induced by the respective loss functions. NSCL's objective structurally mimics CL's. Both losses attract a single positive sample to an anchor for each instance, encouraging a similar geometric arrangement focused on instance-level discrimination. SCL enforces a much stronger, class-level constraint. By pulling all positive samples in a batch together, it aggressively reduces intra-class variance to form tight clusters. This induces a geometric structure different from the one learned by CL's instance-based task~\citep{luthra2025selfsupervisedcontrastivelearningapproximately}. Cross-Entropy (CE) represents an intermediate case. It focuses on class separability via a decision boundary rather than explicit feature clustering. This explains its modest alignment trajectory, which dips initially before rising as both CL and CE converge toward learning semantic features.

{\bf Validating Thm.~\ref{thm:sim-coupling} as a function of class count.\enspace} Our theory (Thm.~\ref{thm:sim-coupling}) predicts that using more classes yields stronger CL--NSCL alignment. We test this via $C'$-way training: for each $C'\in[2,C]$, we train CL and NSCL on random $C'$-class subsets for 1,000 epochs (except 100 epochs for IM-1K). As shown in Fig.~\ref{fig:alignment_classes}, representation similarity (RSA/CKA) increases with $C'$ across all datasets. 

{\bf Effect of temperature on alignment.\enspace} According to our theoretical analysis (Cors.~\ref{cor:cka-main}–\ref{cor:rsa-main}), CL-NSCL alignment improves with higher values of temperature ($\tau$). We empirically verify this claim by training CL and NSCL models for 300 epochs, over three different values of $\tau \in \{ 0.1, 0.5, 1.0 \}$. Both models--CL and NSCL--are trained with same $\tau$ in each run. As shown in Fig.~\ref{fig:alignment_tau}, models trained with $\tau = 1.0$ achieve higher alignment compared to models trained with lower temperatures. 

{\bf Effect of batch size on alignment.\enspace} Our framework also predicts that alignment varies with batch-size $B$ depending on how the learning rate scales. To investigate this, we vary $\eta$ with $B$ across four cases: $\eta=\tfrac{0.3B}{256}$, $\eta=\tfrac{0.3\sqrt{B}}{256}$, $\eta=\tfrac{0.3\sqrt[4]{B}}{256}$, and $\eta=0.3$. Under $\mathcal{O}(B)$ scaling, CL–NSCL alignment decreases as $B$ grows, matching the theorem’s implication for that scaling; for the other three cases, alignment increases with $B$, again consistent with the bound under those dependencies (see Fig.~\ref{fig:alignment_batches}).

\begin{figure}[t]
    \centering
    \begin{tabular}{c@{\hspace{6pt}}c@{\hspace{6pt}}c@{\hspace{6pt}}c}
         \includegraphics[width=0.235\linewidth]{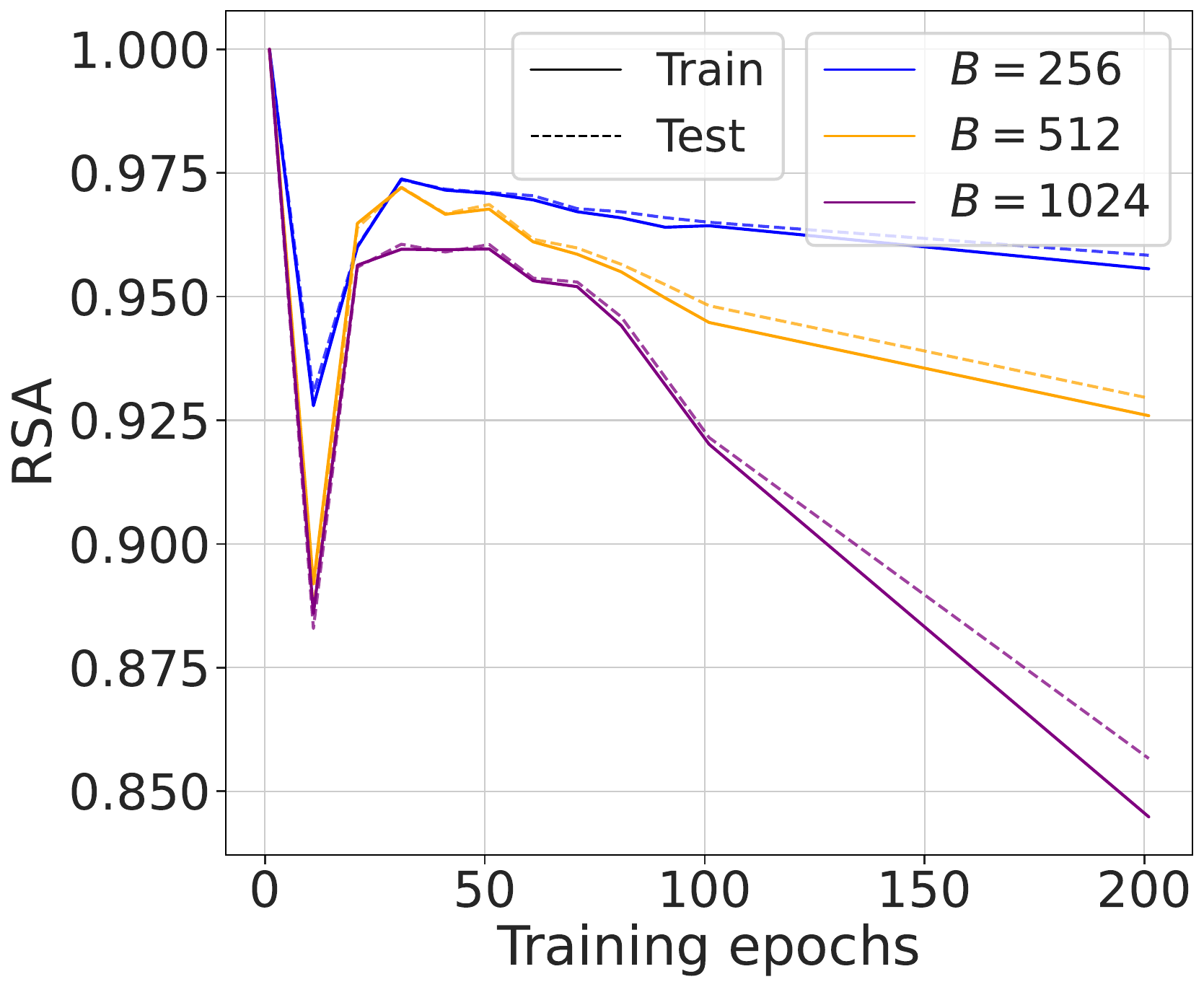} &
         \includegraphics[width=0.235\linewidth]{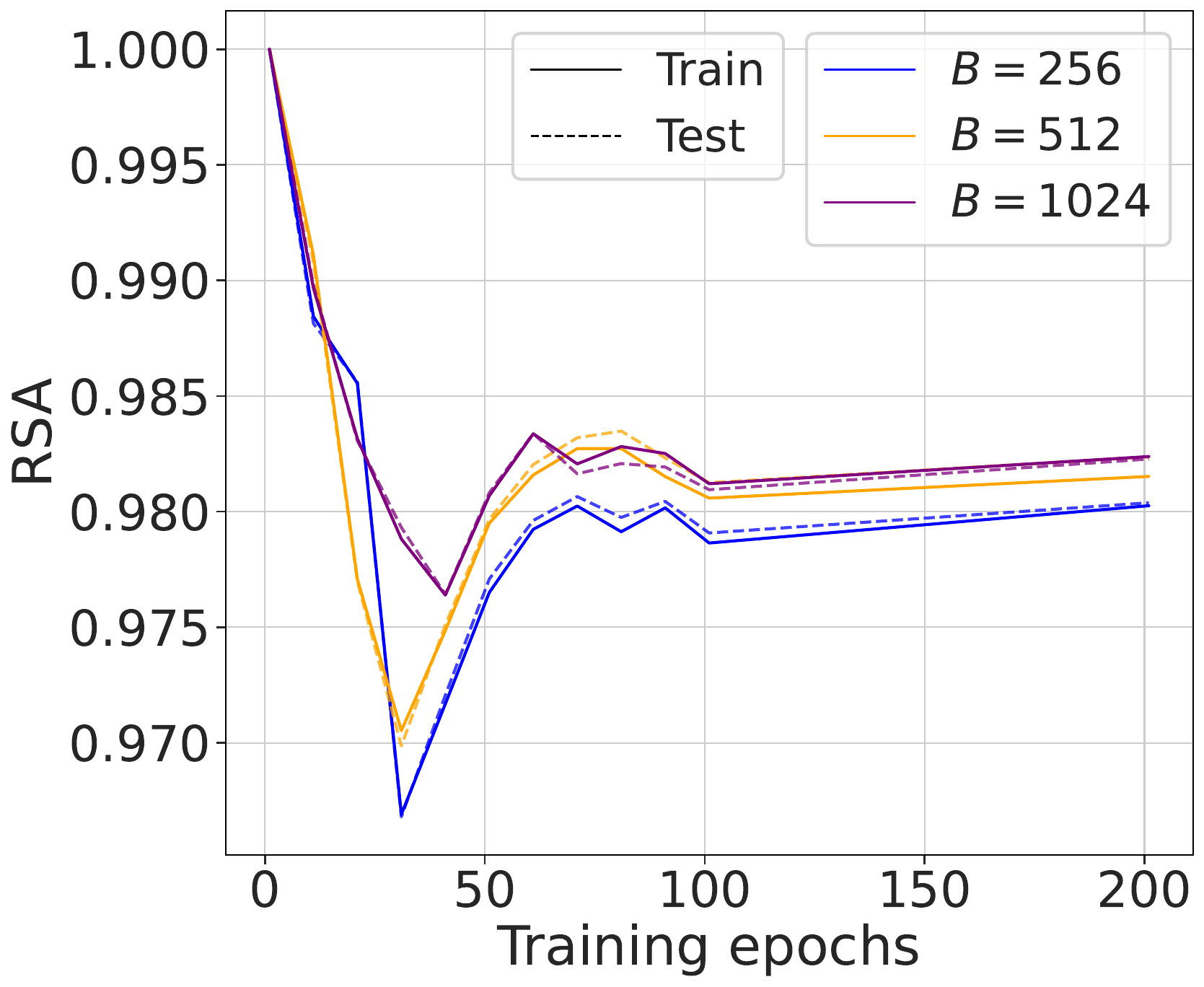} & \includegraphics[width=0.235\linewidth]{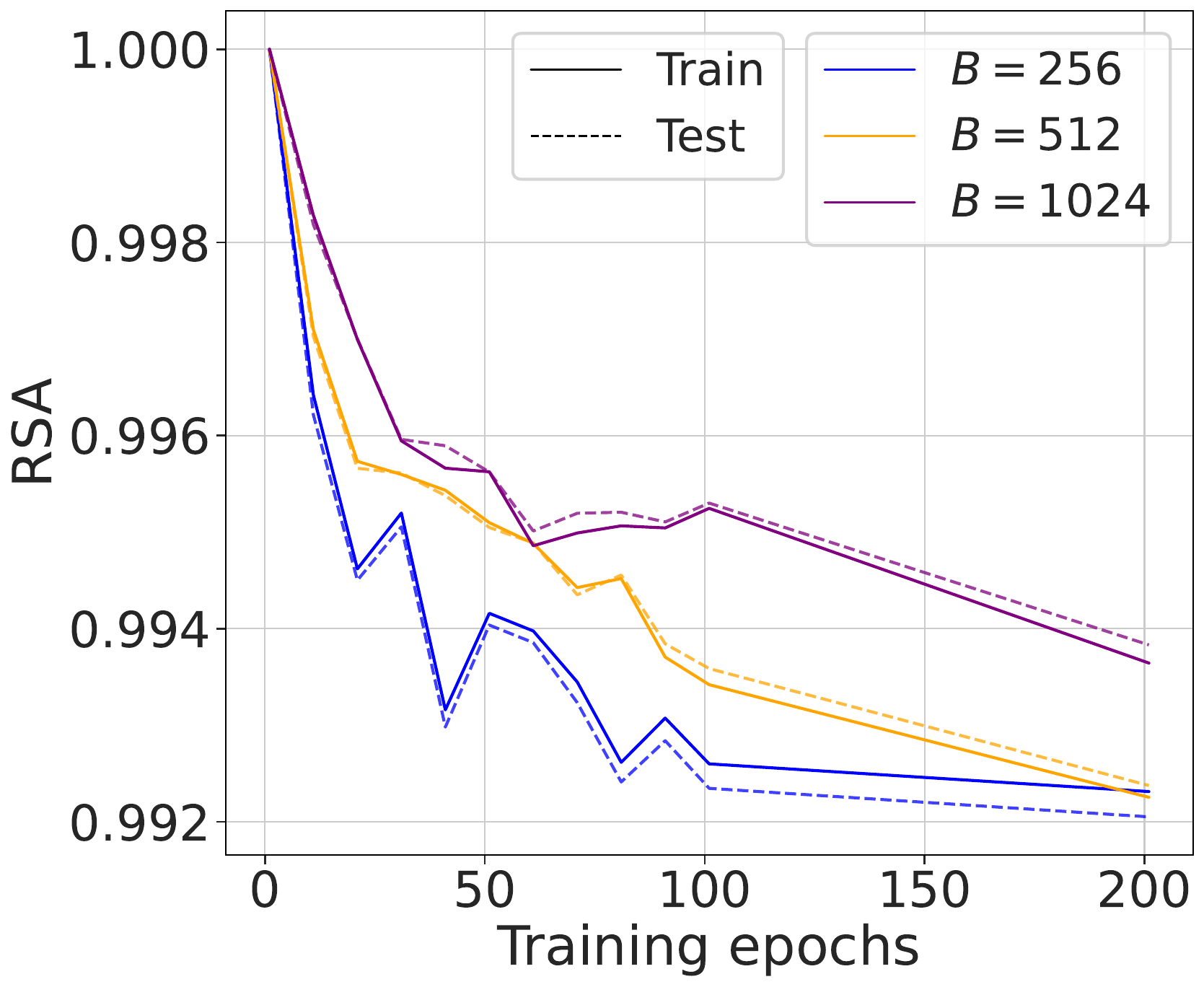} &
         \includegraphics[width=0.235\linewidth]{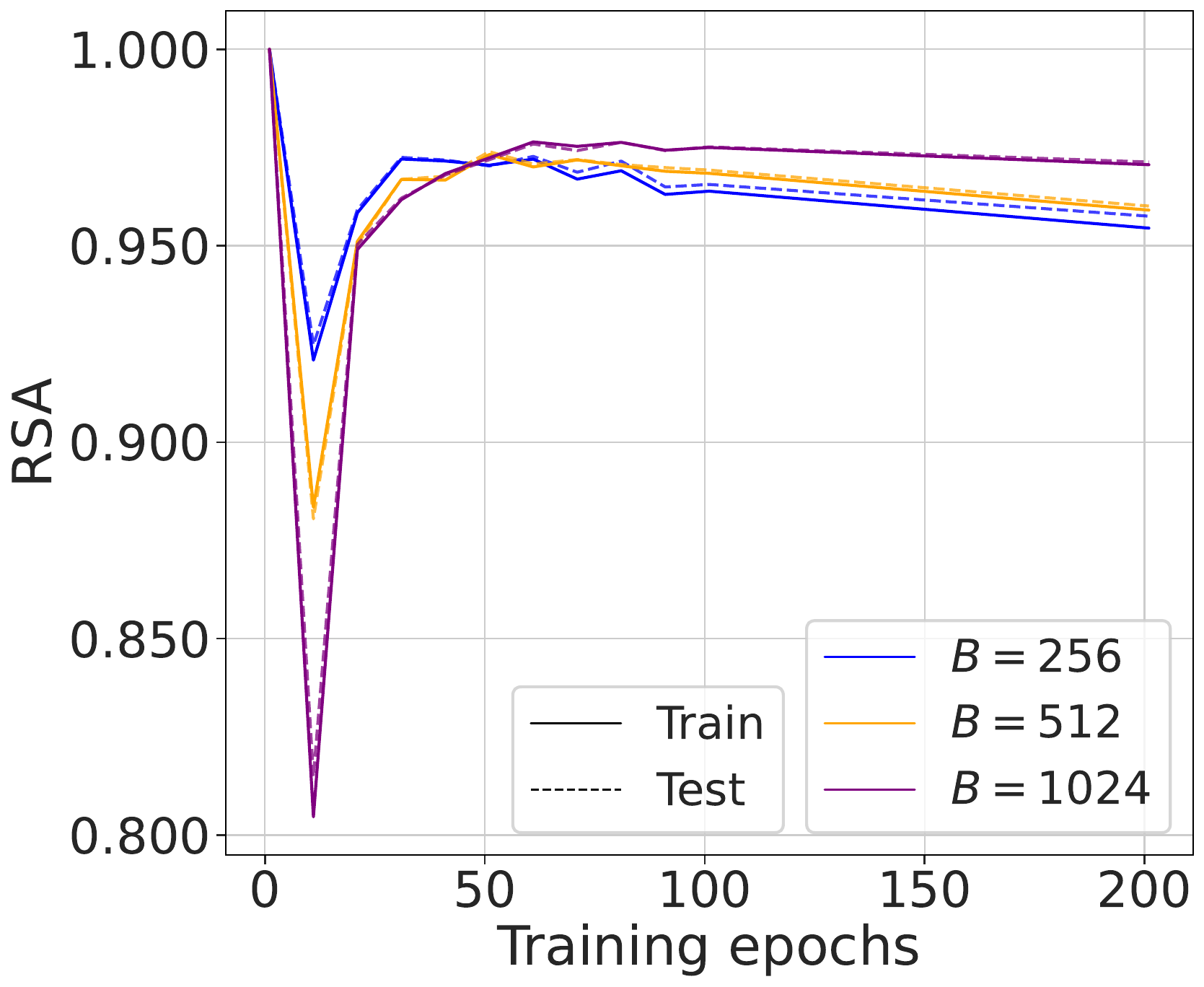} \\
         \includegraphics[width=0.235\linewidth]{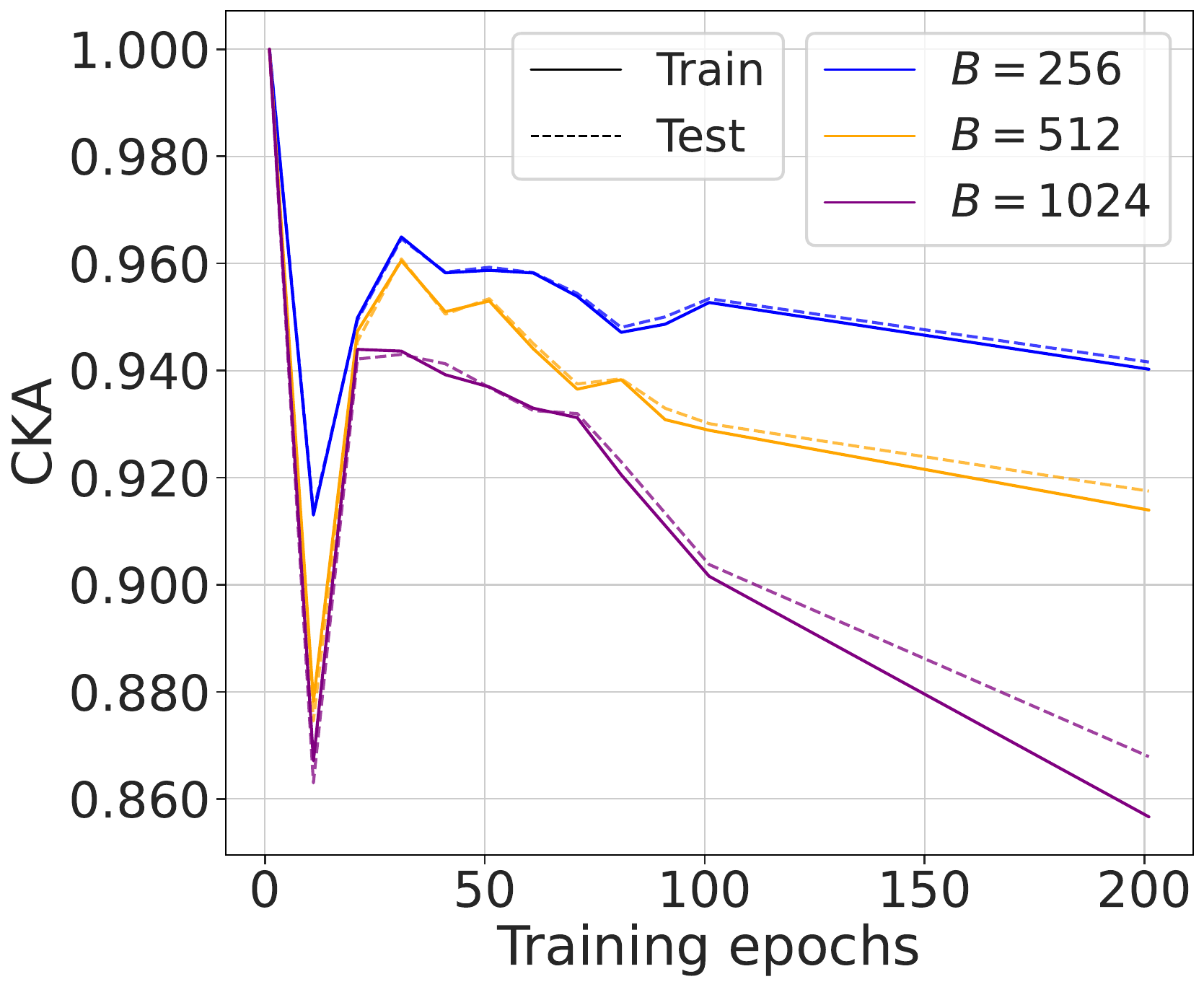} &
         \includegraphics[width=0.235\linewidth]{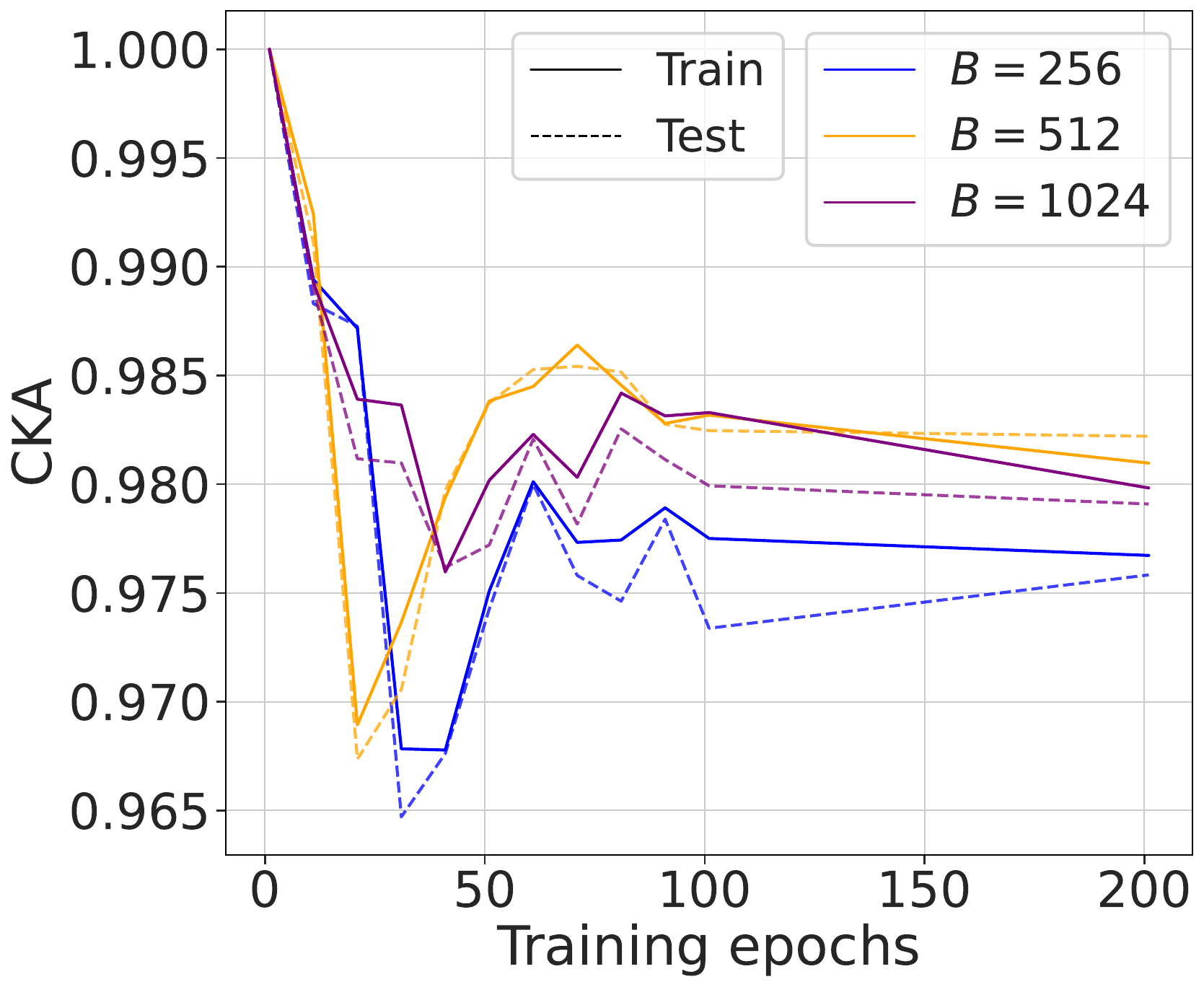} & \includegraphics[width=0.235\linewidth]{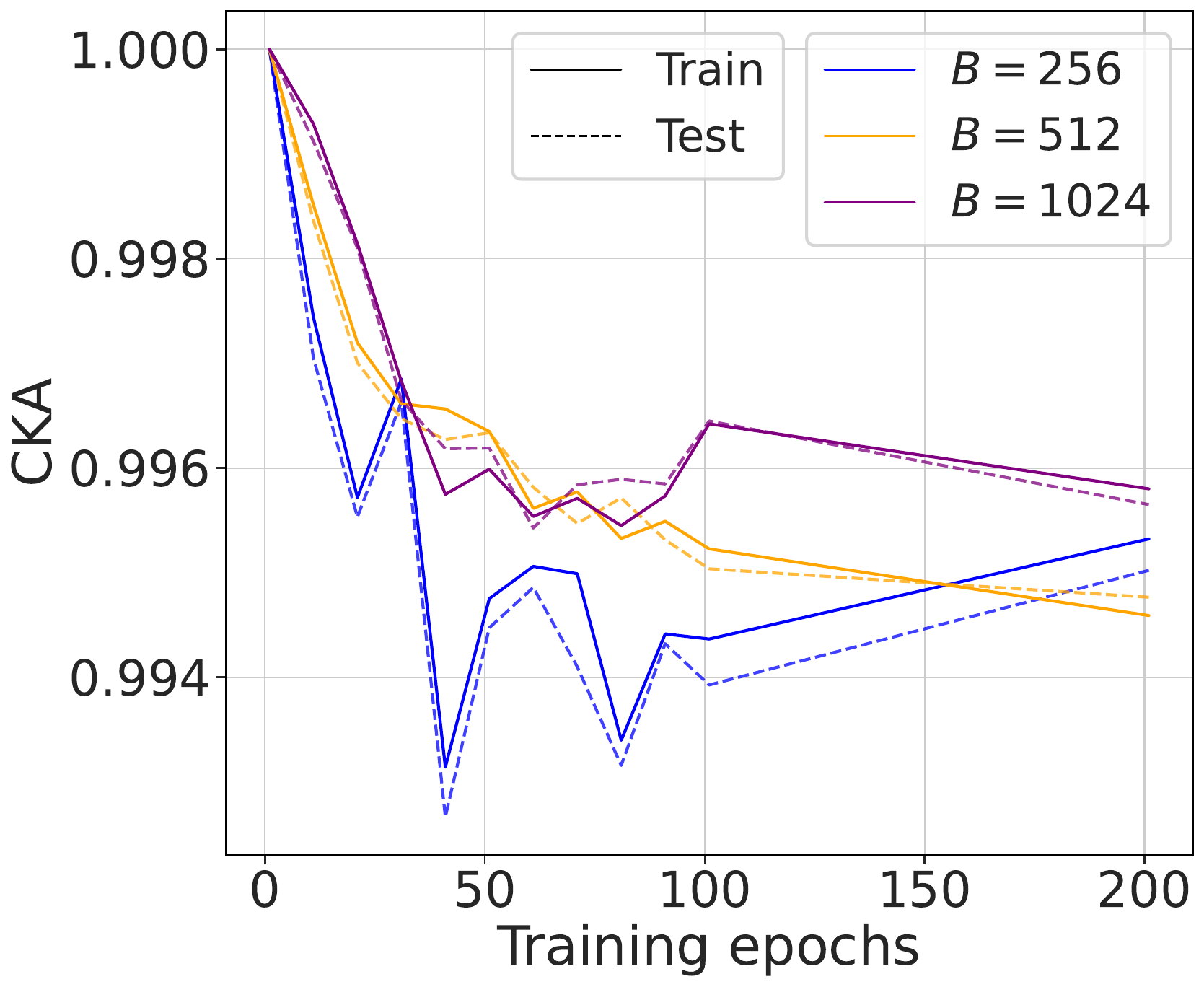} &
         \includegraphics[width=0.235\linewidth]{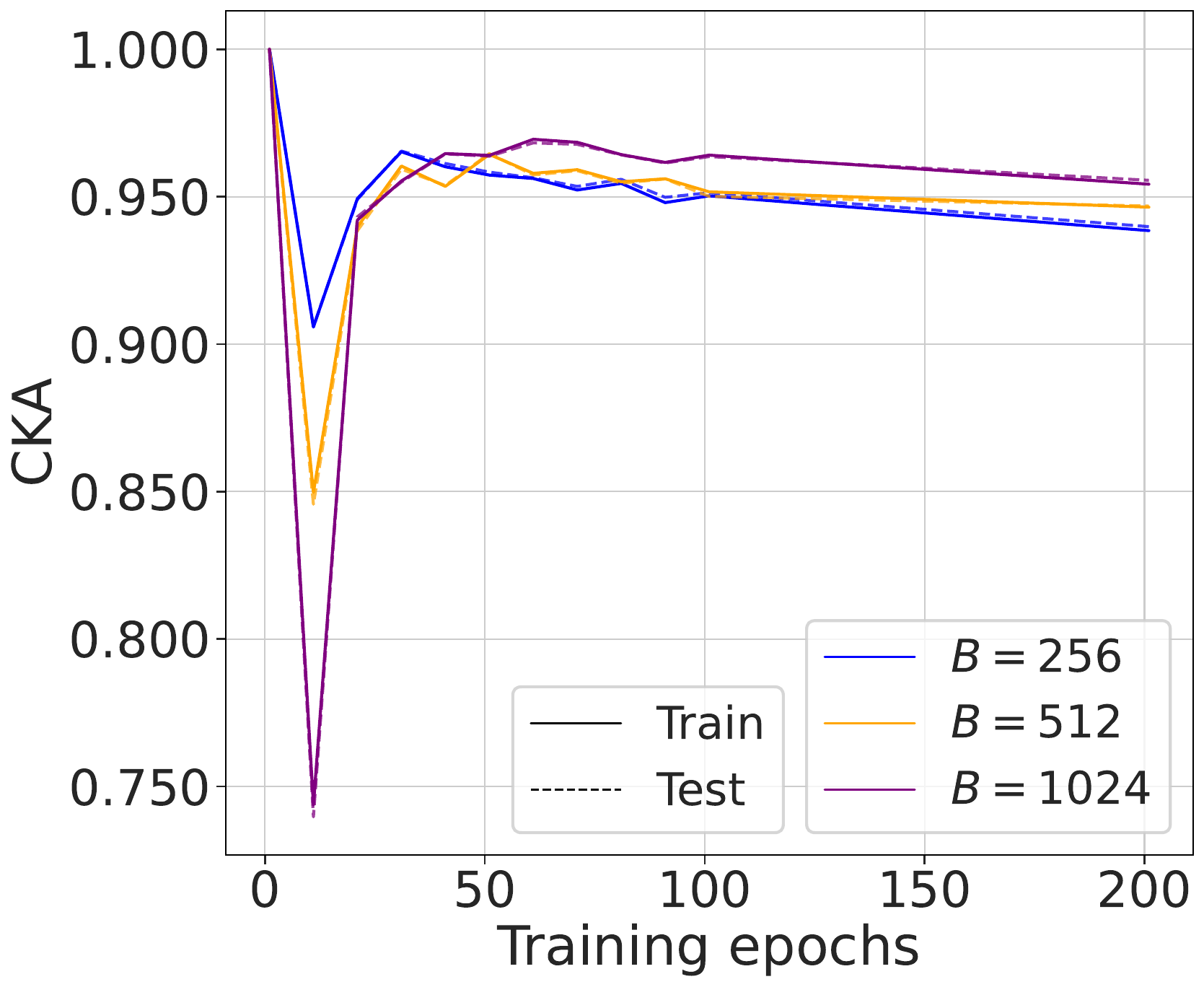} \\
         {\small {\bf (a)} $\mathcal{O}(B)$} & {\small {\bf (b)} $\mathcal{O}(B^{1/2})$} & {\small {\bf (c)} $\mathcal{O}(B^{1/4})$} & {\small {\bf (d)} $\mathcal{O}(1)$}
    \end{tabular}   
    \caption{\textbf{Effect of batch size with scaled learning rates.} We trained CL, and NSCL models for 300 epochs on Mini-ImageNet, with varying batch-sizes ($B \in \{256,\, 512,\, 1024\}$). For each experiment, the learning rate $\eta$ is scaled as a function of batch-size, as mentioned under each panel. For instance, the results shown in panel (b) use a learning rate of $\eta = \frac{0.3\,\sqrt{B}}{256}$.}
    \label{fig:alignment_batches}
\end{figure}

{\bf Weight-space coupling.\enspace} We next study whether the observed alignment between representations of contrastive and supervised models is also reflected directly in their parameters. For this, we measure the average weight difference between a contrastive model and two supervised counterparts as follows: $\sum_l\tfrac{\| w_{\CL}^l - W_{\text{sup}}^l\|_F}{0.5\,(\|w_{\CL}^l\|_F + \|w_{\text{sup}}^l\|_F)} $ where $w_{\CL}^l$ and $w_{\text{sup}}^l$ are weights corresponding to $l^{\text{th}}$ layer of self-supervised and supervised models respectively, and $\|\cdot \|_F$ denotes Frobenius norm. As we show in Fig.~\ref{fig:weight_gap_epochs}, for each dataset, we observe a significant divergence in weight space: both supervised models (NSCL and SCL) increasingly separate from the contrastive model as training progresses.

\begin{figure}[t]
    \centering
    \begin{tabular}{c@{\hspace{6pt}}c@{\hspace{6pt}}c@{\hspace{6pt}}c}
         \includegraphics[width=0.235\linewidth]{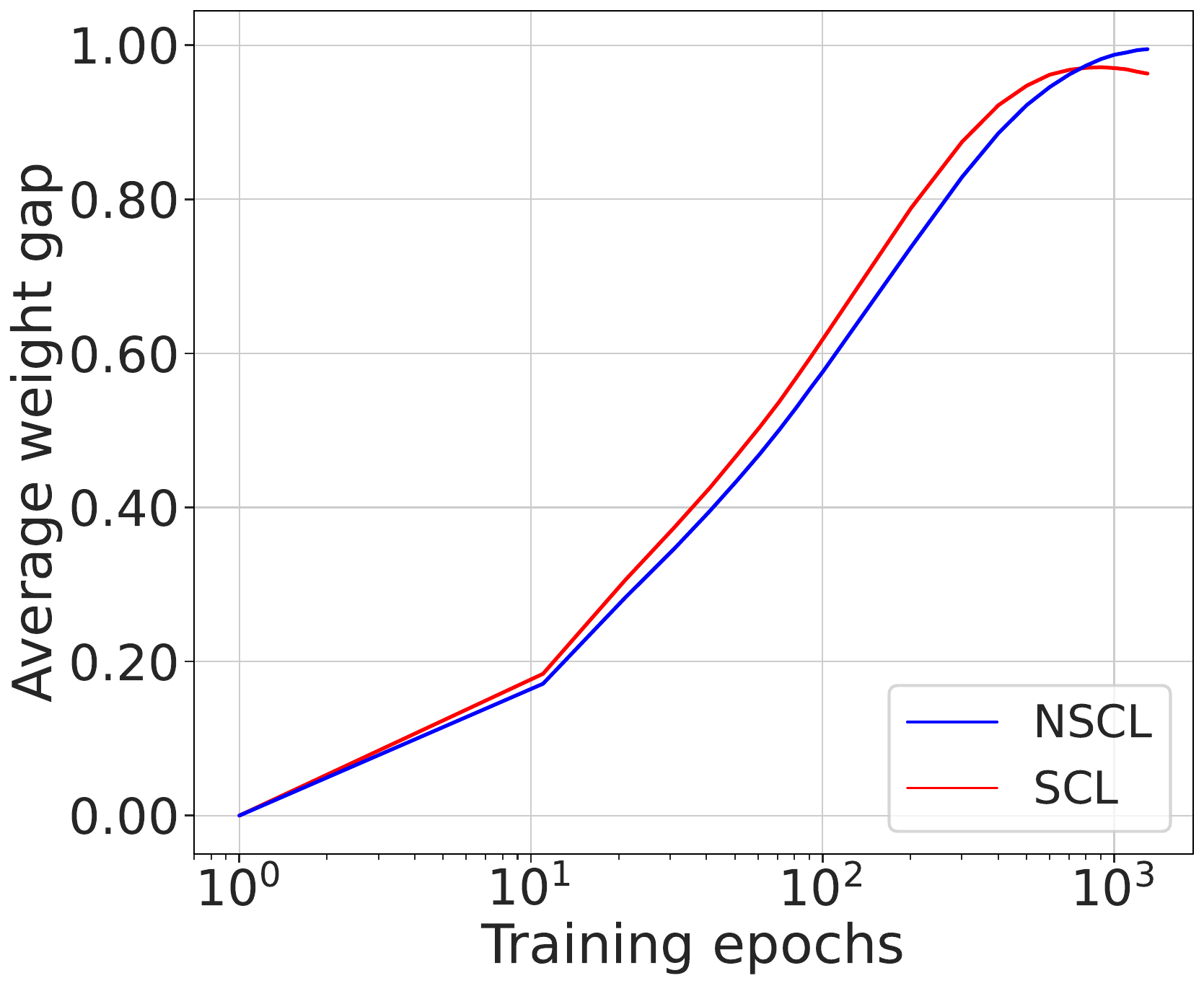} & 
         \includegraphics[width=0.235\linewidth]{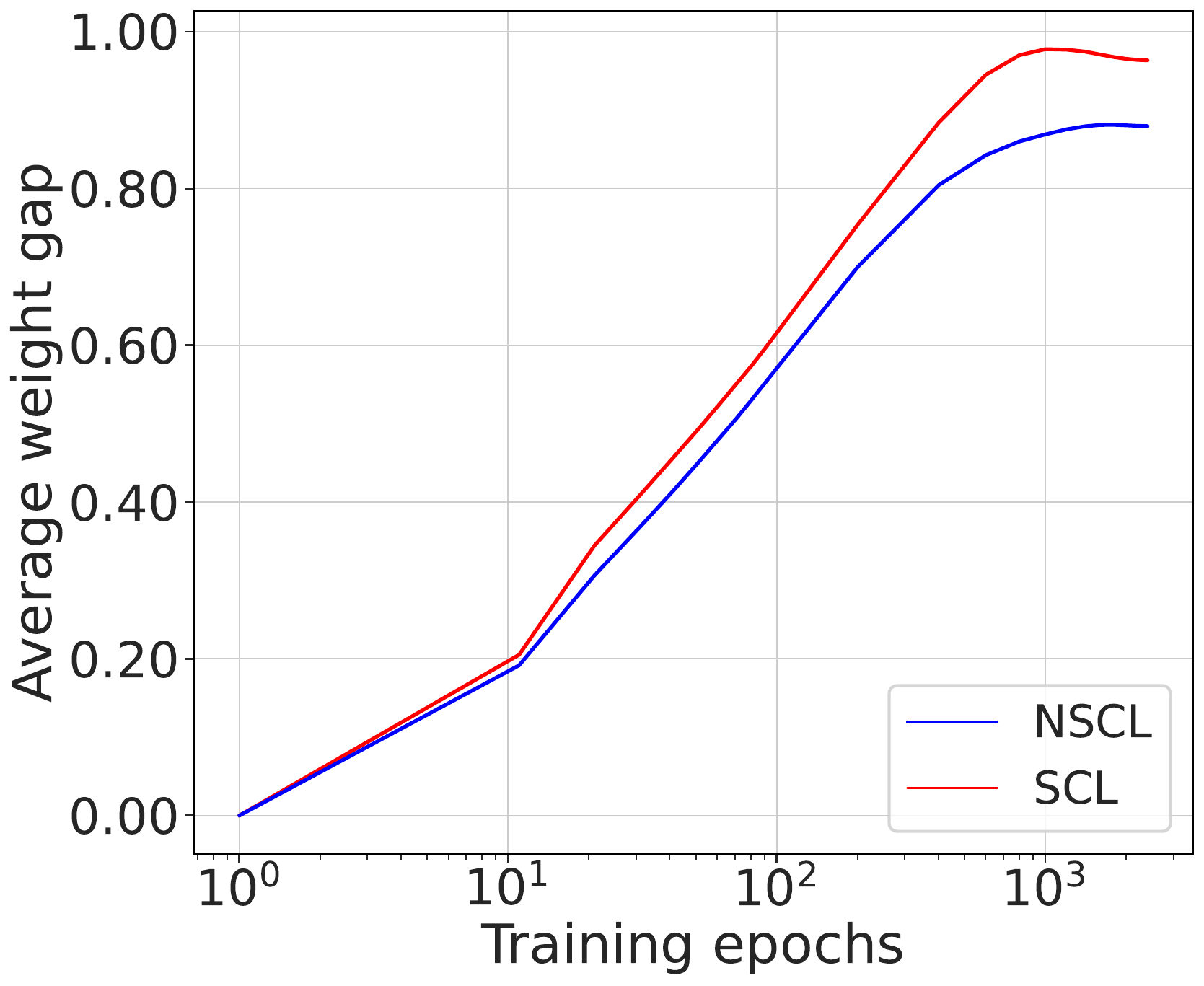} &
         \includegraphics[width=0.235\linewidth]{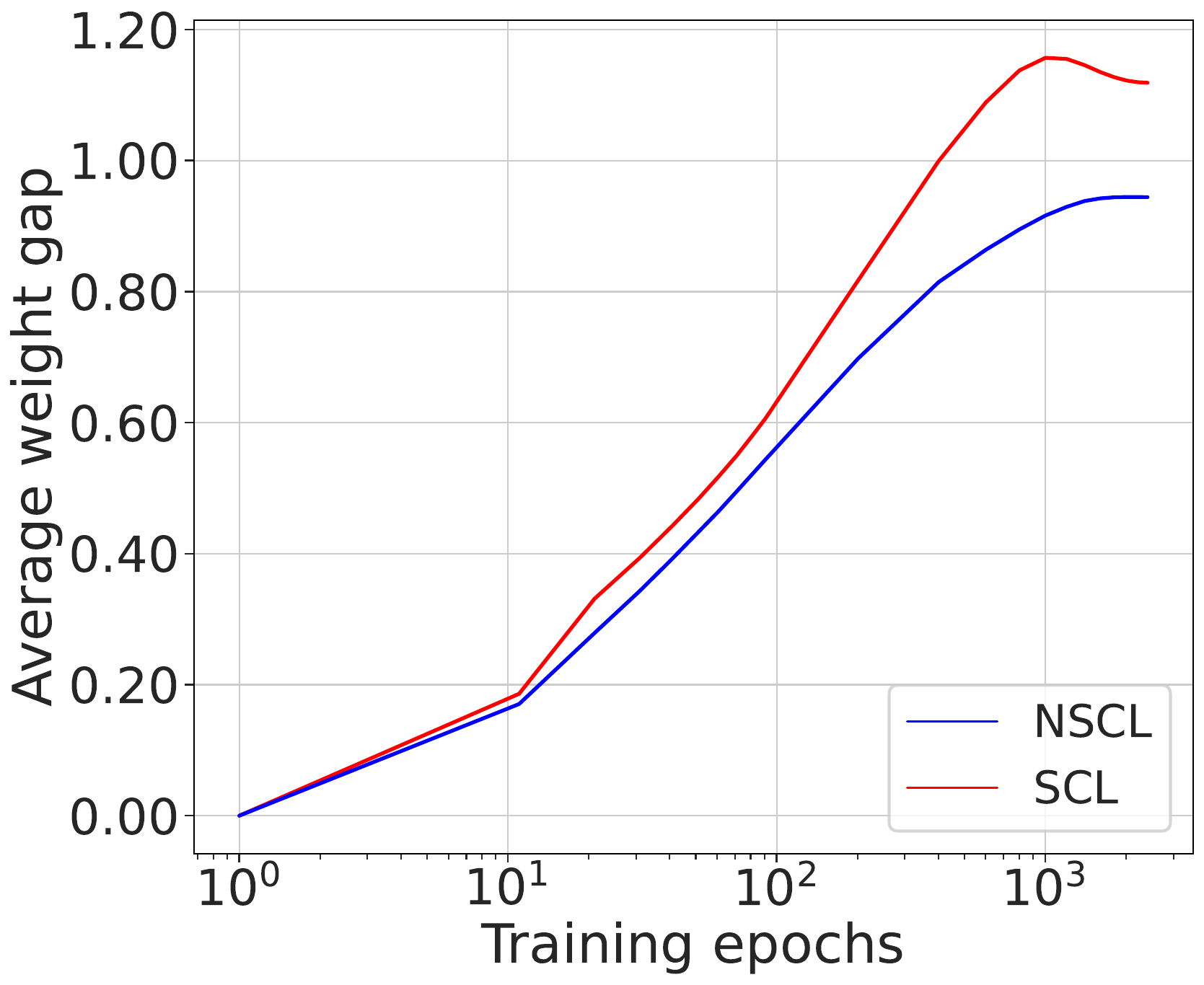} &
         \includegraphics[width=0.235\linewidth]{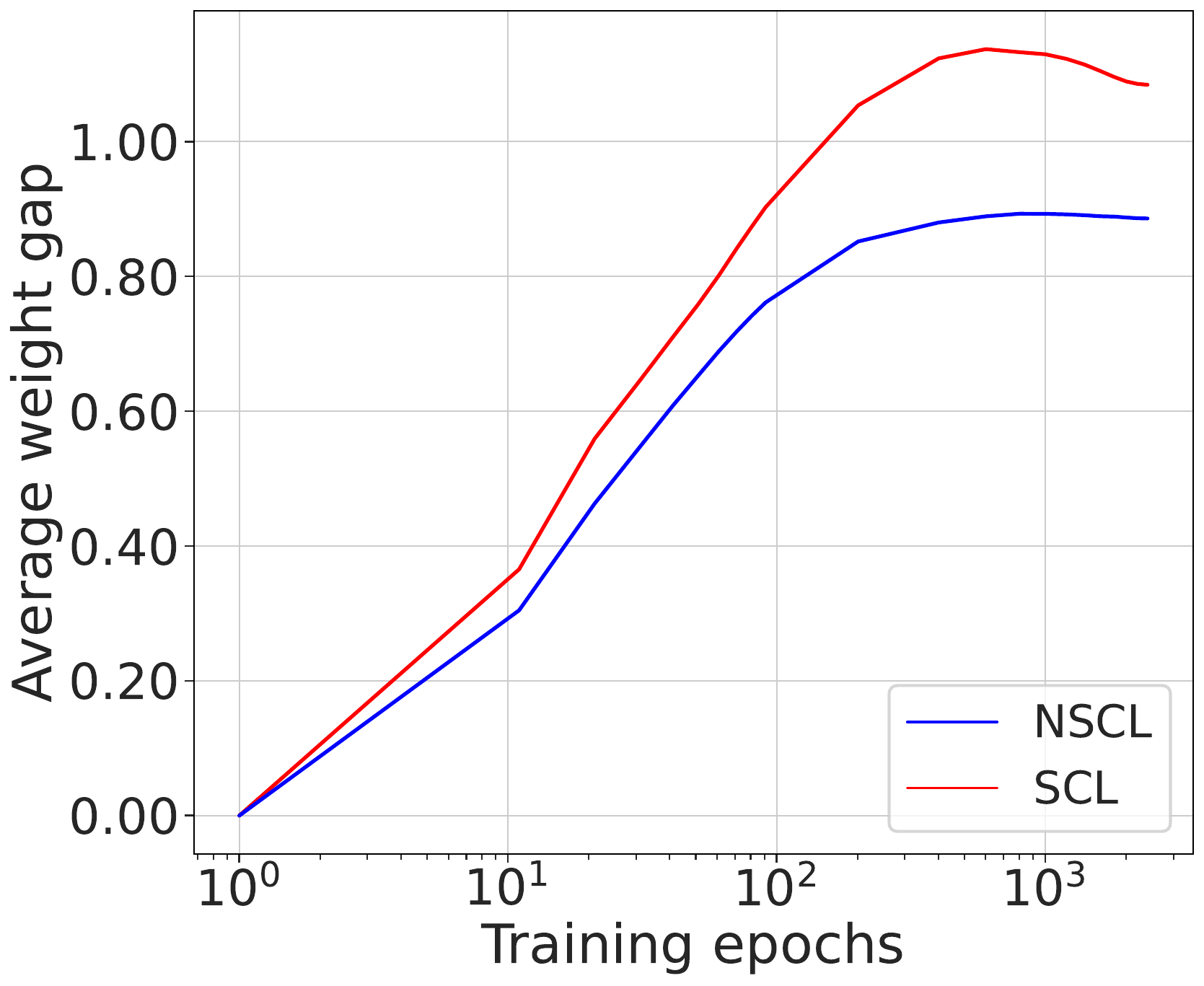} \\
         {\small {\bf (a)} CIFAR10} & {\small {\bf (b)} CIFAR100} & {\small {\bf (c)} Mini-ImageNet} & {\small {\bf (d)} Tiny-ImageNet}
    \end{tabular}   
    \caption{{\bf Weight-space alignment quickly deteriorates.} Using the same ResNet-50 instances as in Fig.~\ref{fig:alignment_epochs}, we plot the average weight gap between CL and the supervised models (NSCL and SCL) across training epochs. Both supervised variants diverge from the CL model, with SCL showing a wider separation.}
    \label{fig:weight_gap_epochs}
\end{figure}

\section{Conclusions, Limitations and Future Work}


{\bf Conclusions.\enspace} We studied the dynamic alignment between contrastive learning (CL) and its supervised counterpart (NSCL). By analyzing coupled SGD under shared randomness, we showed that while parameter-space trajectories may diverge exponentially, representation-space dynamics are far more stable: the similarity matrices induced by CL and NSCL remain close throughout training. This yields high-probability lower bounds on alignment metrics such as CKA and RSA, directly certifying representational coupling. Empirically, our experiments confirmed these trends across datasets and architectures. Together, our results highlight that the implicit supervised signal in CL is not confined to its loss function but extends throughout the entire optimization trajectory.

{\bf Limitations.\enspace} Our theoretical bounds, while structurally informative, are quantitatively loose in large-scale, long-horizon regimes. The exponential factor can quickly make the bounds vacuous beyond the first few epochs of large-scale training. This limitation stems from worst-case arguments and uniform high-probability concentration bounds that favor generality over sharpness. Nonetheless, the analysis offers qualitative insights into how the temperature parameter, the number of hidden classes, the batch size and training duration affect the alignment between CL and NSCL models. 

{\bf Future directions.\enspace} We view our results as a first step toward a more refined theory of self-supervised representation alignment. Future work could (i) derive tighter constants by exploiting data-dependent structure rather than worst-case bounds, and (ii) extend the framework to other SSL paradigms (e.g., non-contrastive methods). Improving these guarantees while retaining their stability properties would provide an even stronger theoretical bridge between supervised and self-supervised learning.



\bibliography{preprint_neurips}
\bibliographystyle{iclr2026_conference}

\newpage
\appendix


\section{Additional Experiments}\label{app:experiments}
{\bf Datasets and augmentations.\enspace} CIFAR10 and CIFAR100 both consist of 50000 training images and 10000 validation images with 10 classes and 100 classes, respectively, uniformly distributed across the dataset, i.e., CIFAR10 has 5000 samples per class and CIFAR100 has 500 samples per class. Mini-ImageNet also has 5000 test images on top of 50000 train and 10000 validation images, with 100 of 1000 classes from ImageNet-1K~\citep{5206848} (at the original resolution). Tiny-ImageNet contains 100000 images downsampled to $64\times64$, with total 200 classes from IM-1K. Each class has 500 training, 50 validation, and 50 test images.

We use standard augmentations as proposed in SimCLR~\citep{pmlr-v119-chen20j}. For experiments on Mini-ImageNet, we use the following pipeline: random resized cropping to \(224 \times 224\), random horizontal flipping, color jittering (brightness, contrast, saturation: $0.8$; hue: $0.2$), random grayscale conversion (\(p=0.2\)), and Gaussian blur (applied with probability $0.1$ using a \(3 \times 3\) kernel and \(\sigma = 1.5\)). For Tiny-ImageNet, we drop saturation to $0.4$ and hue to $0.1$ due to low resolution images. For CIFAR datasets, we adopt a similar pipeline with appropriately scaled parameters. The crop size is adjusted to \(32 \times 32\), and the color jitter parameters are scaled to saturation $0.4$, and hue $0.1$.

\subsection{Effect of number of classes on alignment}\label{app:n_way_align}
In addition to the linear CKA results reported in the main text (Fig.~\ref{fig:alignment_classes}), we also evaluate representational similarity using RSA. The corresponding RSA values are presented in Fig.~\ref{fig:alignment_classes_app}, providing a complementary perspective on alignment across varying numbers of classes.

\begin{figure}[h]
  \centering
  \setlength{\tabcolsep}{1.5pt} 
\begin{tabular}{@{}ccccc@{}}
\includegraphics[width=0.195\linewidth]{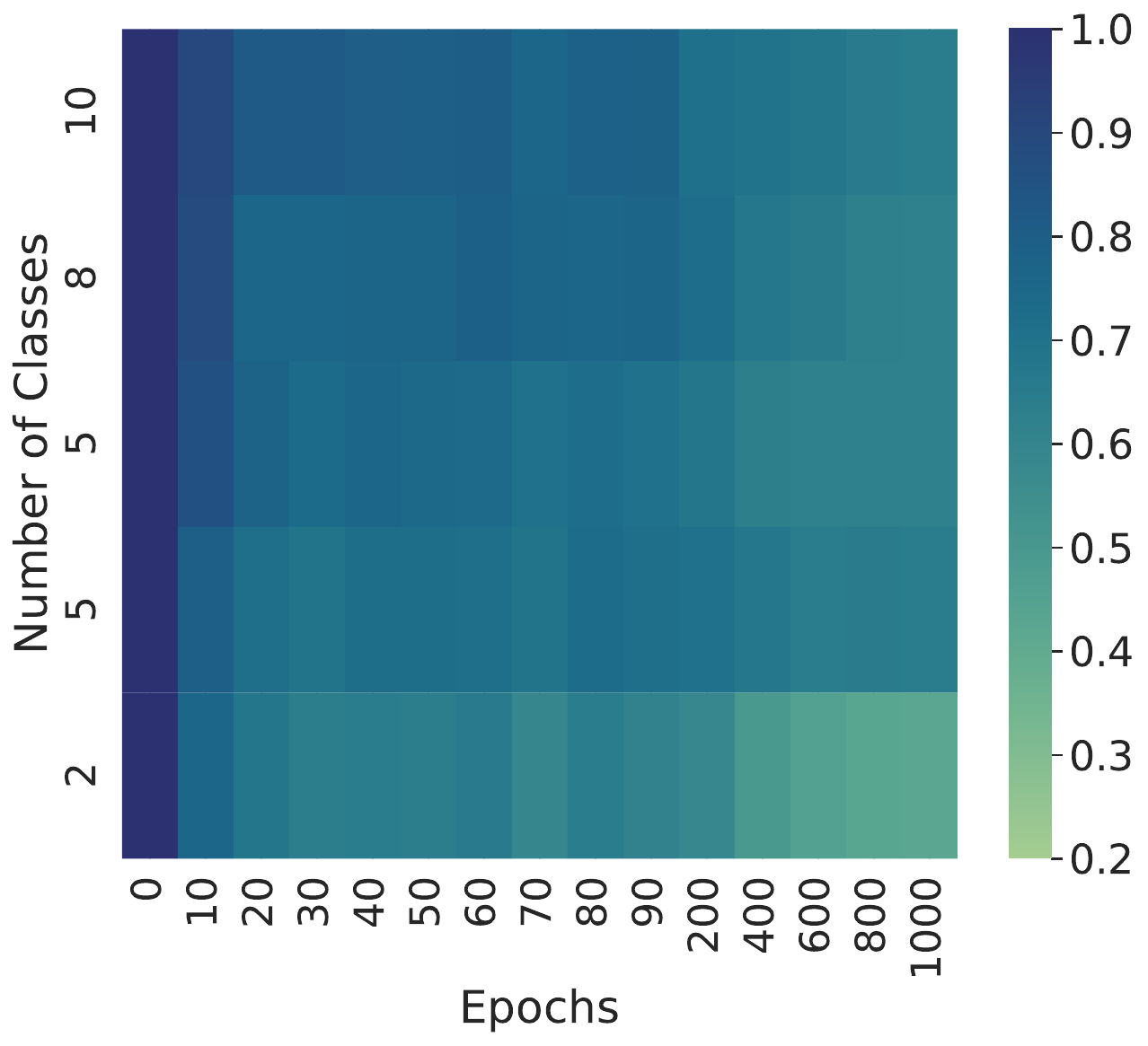} &
\includegraphics[width=0.195\linewidth]{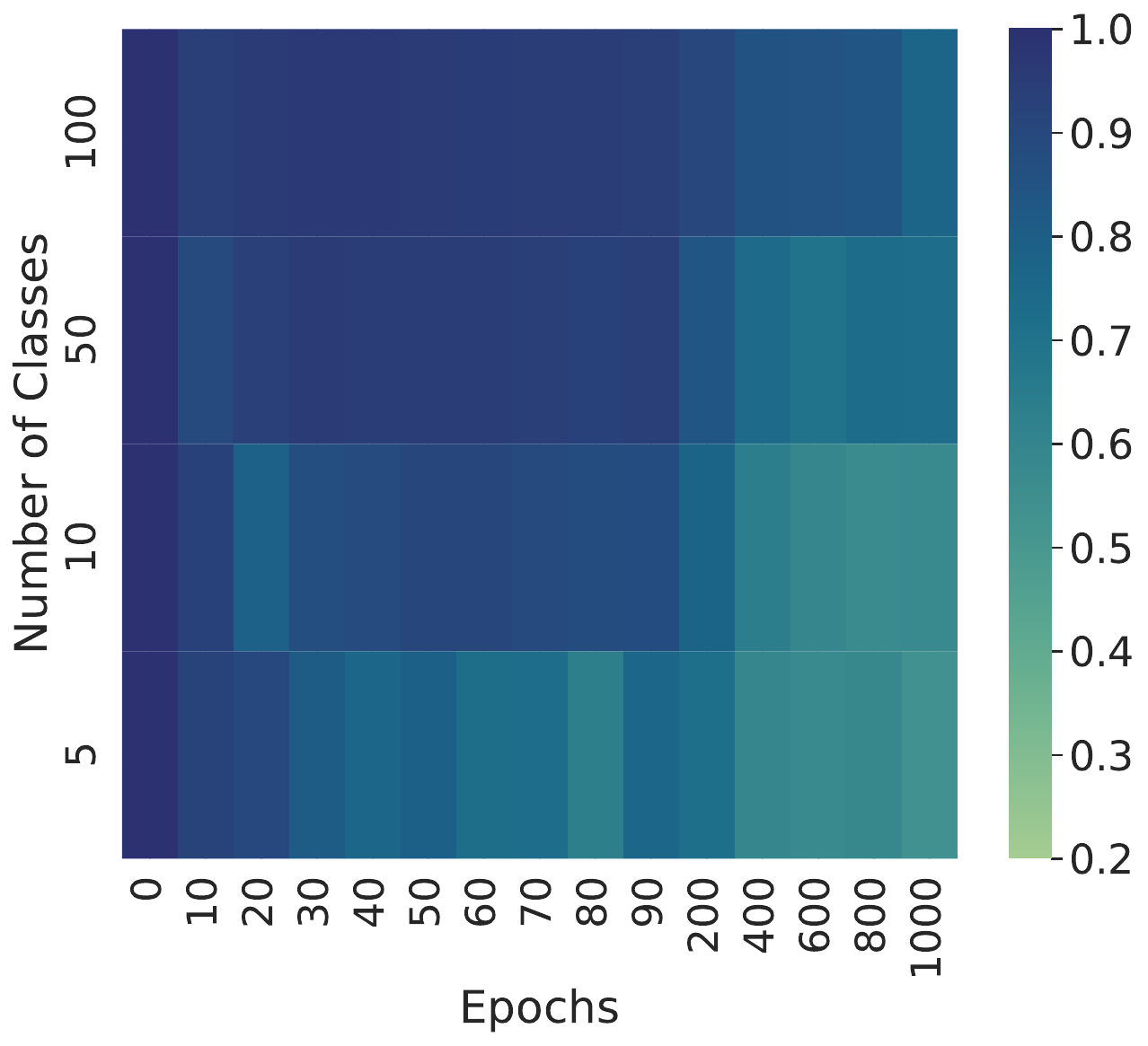} &
\includegraphics[width=0.195\linewidth]{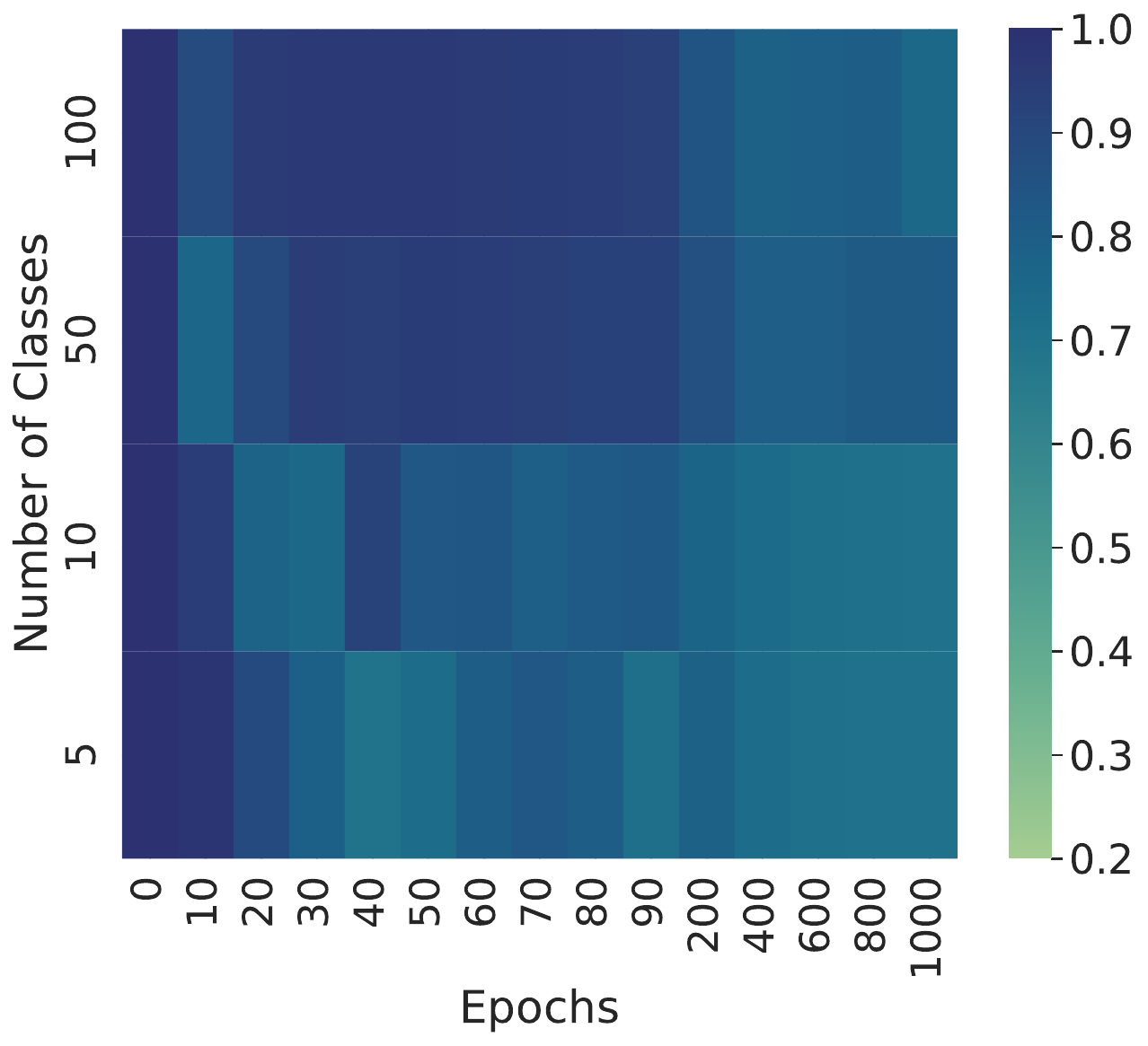} &
\includegraphics[width=0.195\linewidth]{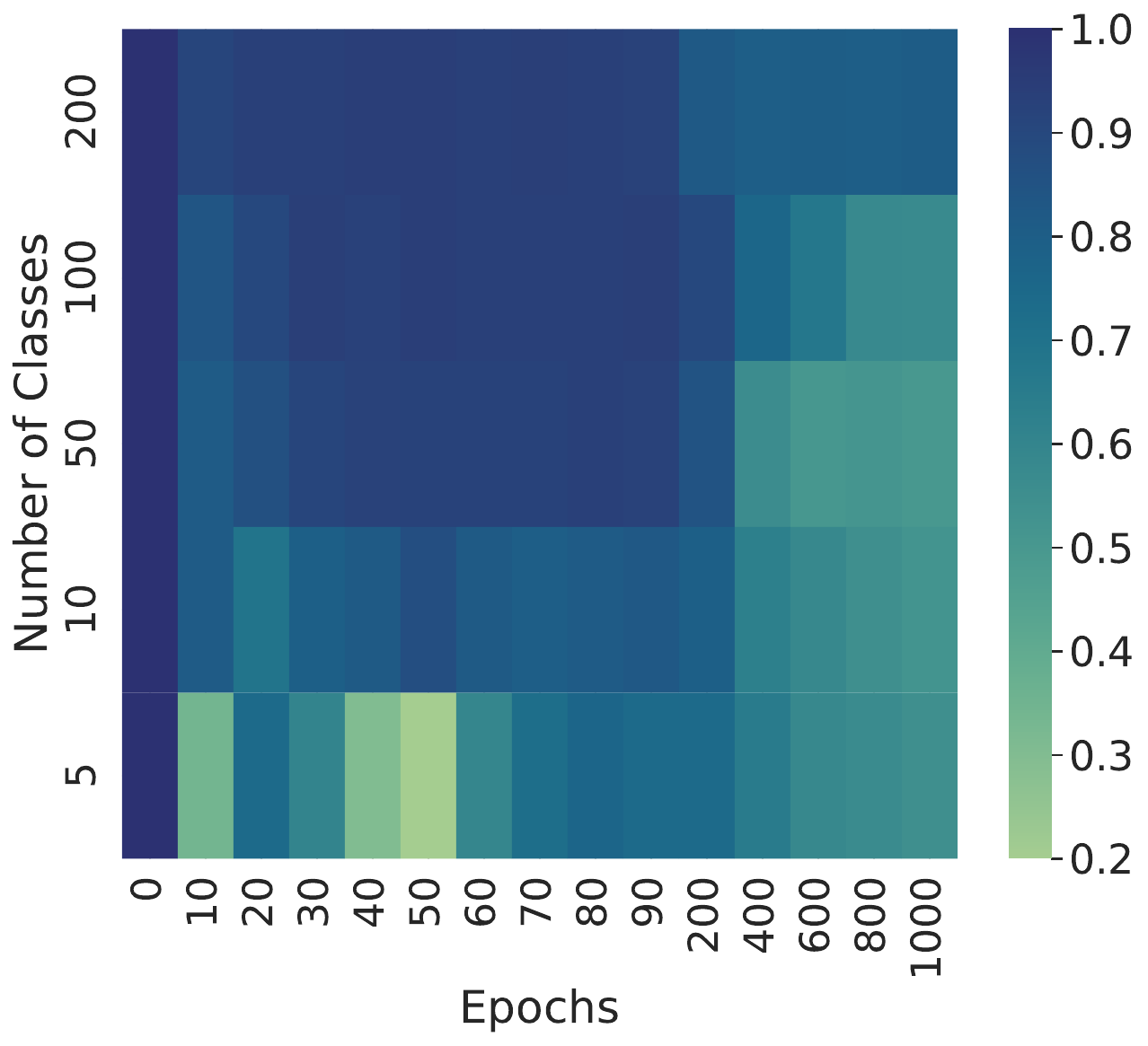} &
\includegraphics[width=0.195\linewidth]{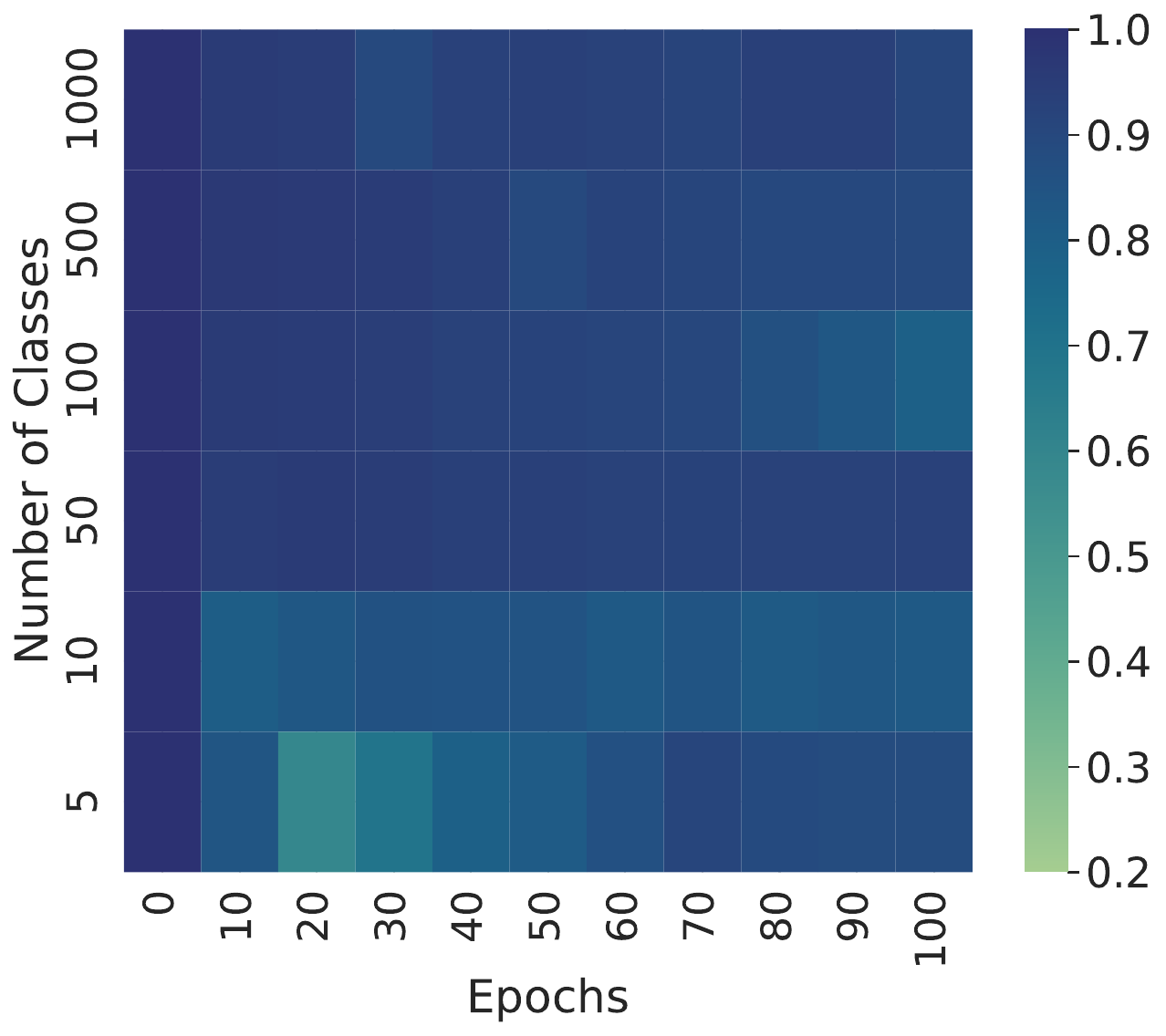} \\
\includegraphics[width=0.195\linewidth]{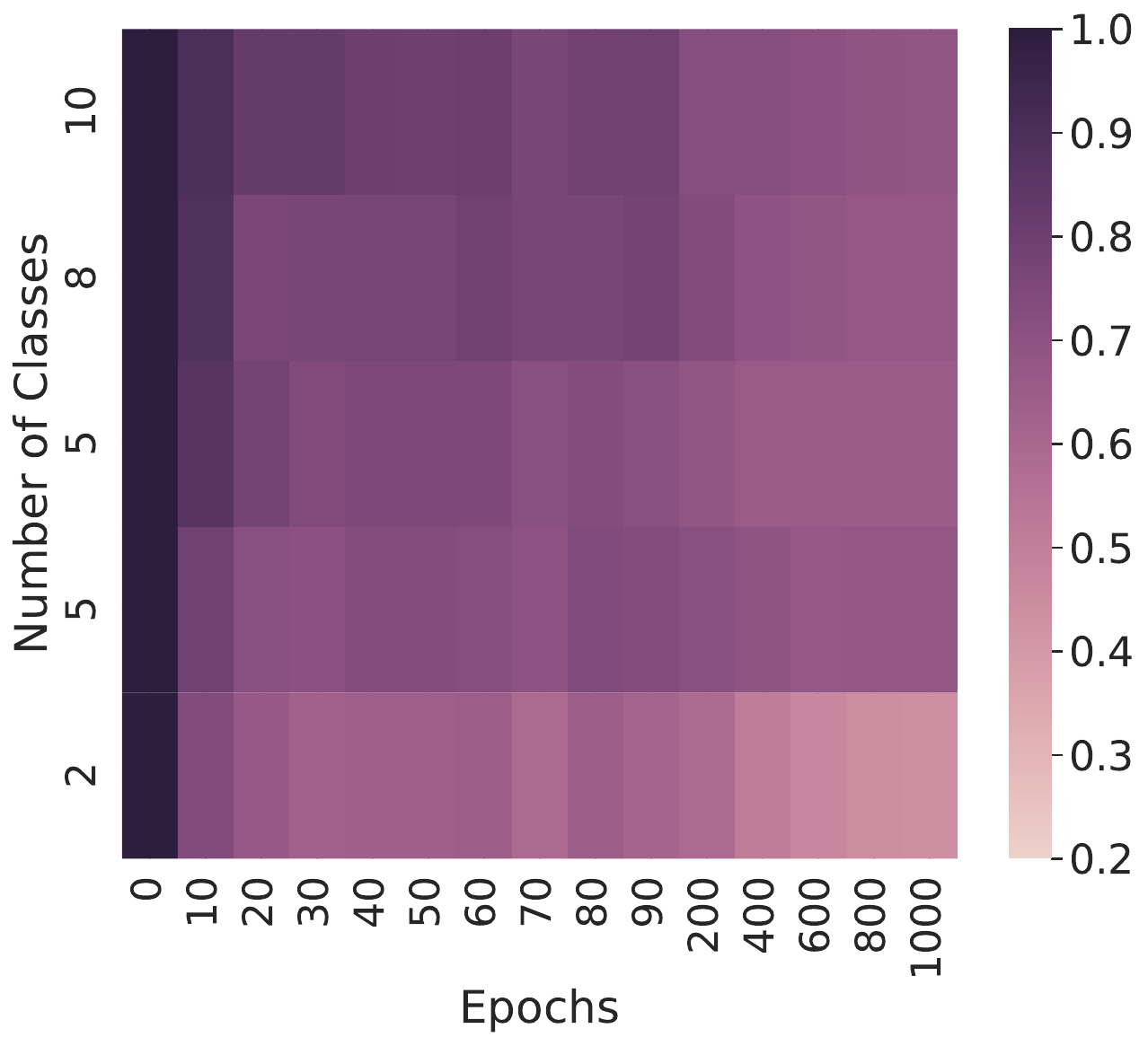} &
\includegraphics[width=0.195\linewidth]{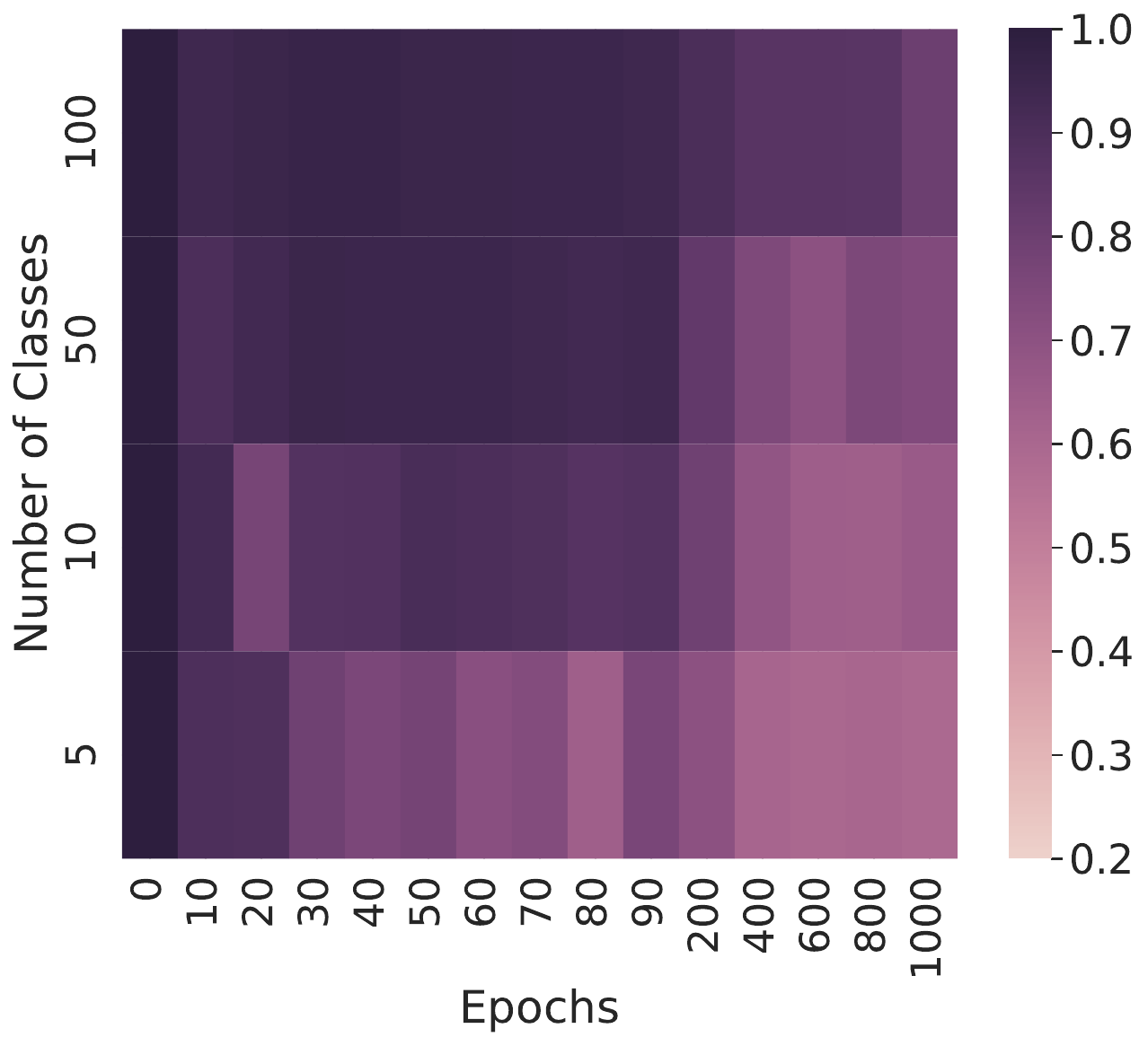} &
\includegraphics[width=0.195\linewidth]{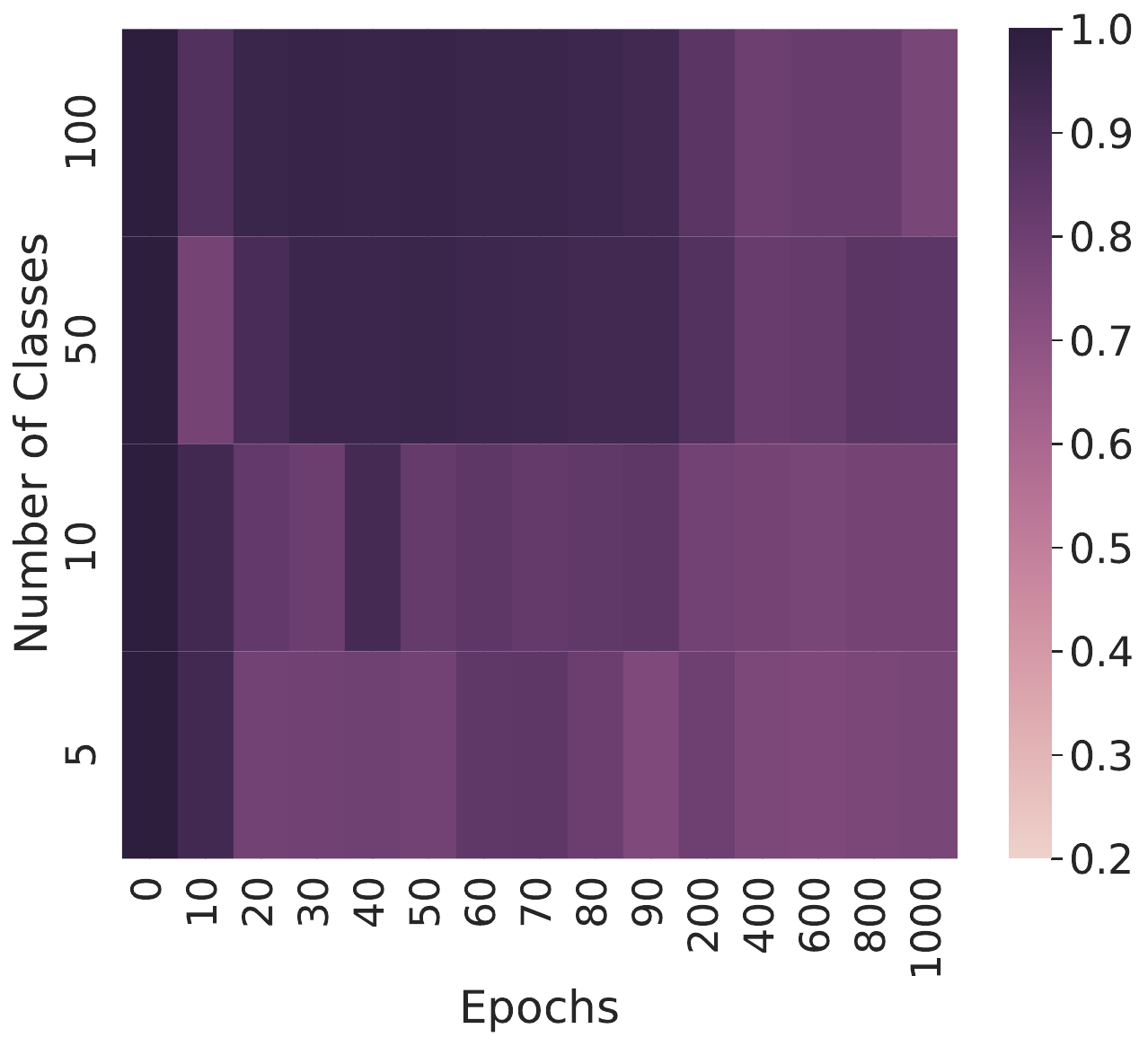} &
\includegraphics[width=0.195\linewidth]{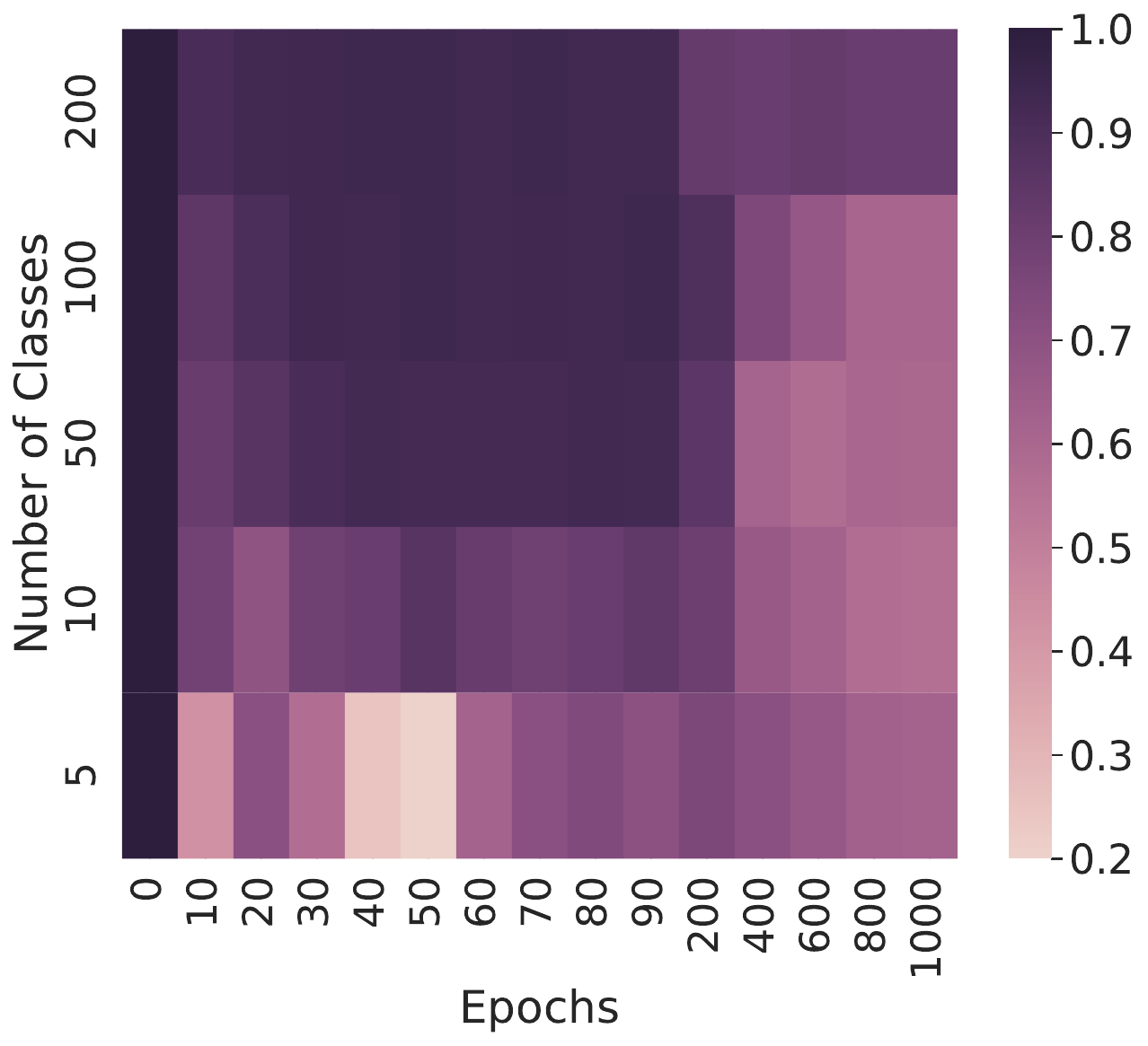} &
\includegraphics[width=0.195\linewidth]{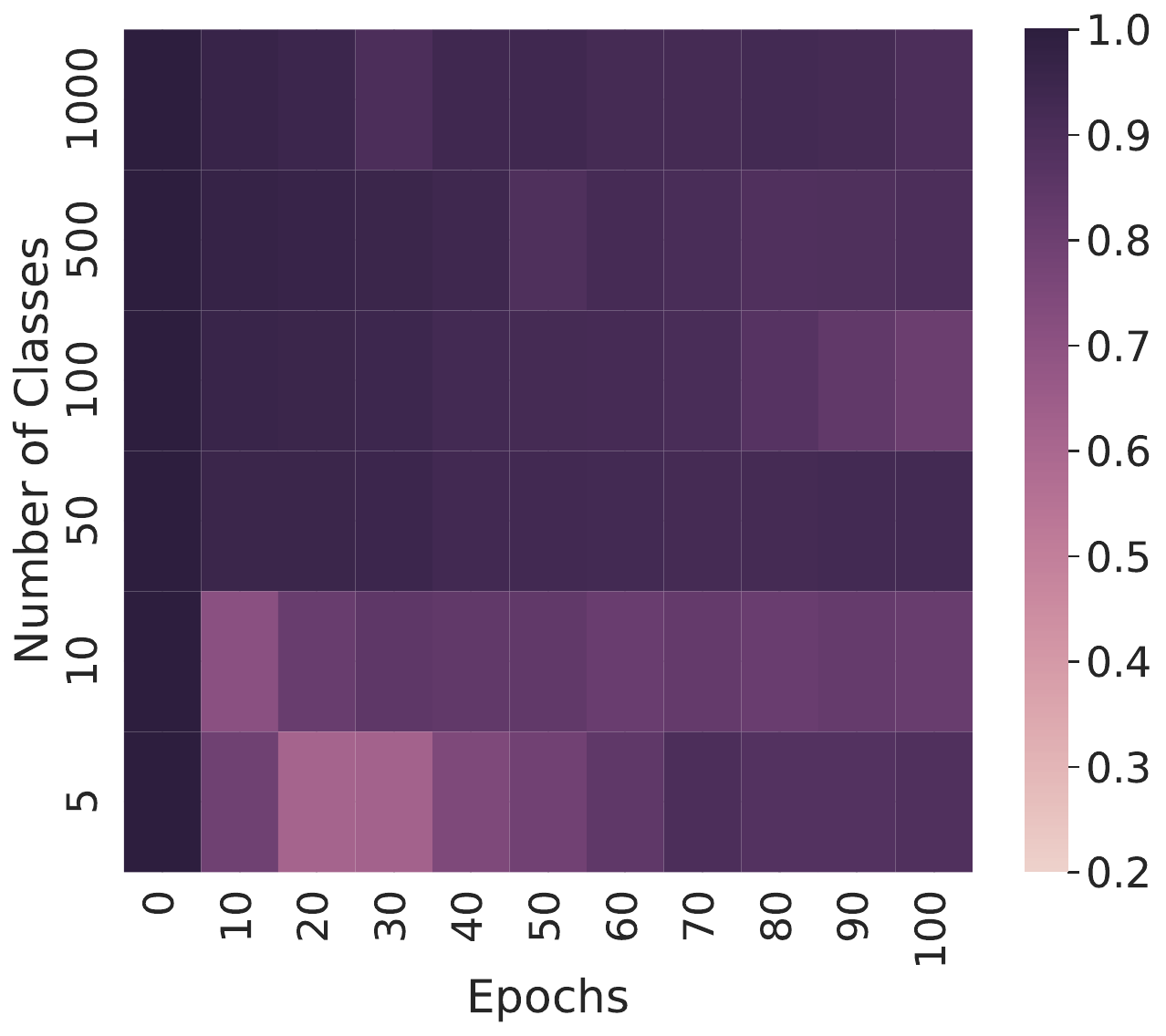} \\
{\small CIFAR-10} & {\small CIFAR-100} & {\small Mini-ImageNet} & {\small Tiny-ImageNet} & {\small ImageNet-1k}
\end{tabular}
  \vspace{-0.5em}
\caption{{\bf CL–NSCL alignment (RSA) increases with the number of training classes.} See Sec.~\ref{sec:exp_results} and Fig.~\ref{fig:alignment_classes} for experimental details.}
  \label{fig:alignment_classes_app}
\end{figure}

\subsection{Fig.~\ref{fig:teaser} Methodology}
We explain how to generate the plots comparing alignment in weight-space and representation-space. The two plots on the left visualize the direction of learning for each model. Each vector represents the change in model's state from initialization (epoch 0) to epoch 1000. 

{\bf Model states.\enspace} We consider CL and NSCL models trained on CIFAR100, corresponding to epoch 0 and epoch 1000-a total of four models. 

{\bf Weight space.\enspace} This plot shows how the raw parameters evolve during training. For all four models, we first flatten all the weights into a massive vector which gives us four points in a very high dimensional space (order of $10^7$). To visualize these points, we perform Principal Component Analysis (PCA) on all four vectors combined and fit them to a 3D space. This creates a shared 3D coordinate system. We transform all four points into this space and we get ${p_{\CL}^0, p_{\CL}^{1000}, p_{\NS}^0, p_{\NS}^{1000}}$. Using these points, we create two vectors: ($v_{\CL}, v_{\NS}$), and create polar plot using the final vectors and the calculated angle between them ($85.7^\circ$).

{\bf Representation space.\enspace} This plot shows how model's alignment for a specific class evolved. We pick one class from our dataset (CIFAR100) and randomly sample 100 images. We use the same samples for all four models to extract their corresponding features, say $\mathrm{Z} \in \mathbb R^{100 \times d}$, where $d$ is the projection dimension. We concatenate total 400 representations (100 from each model) and perform PCA to learn a shared 3D coordinate system. The representations are transformed to this shared space ($\mathbb R^{100 \times d} \rightarrow \mathbb R^{100 \times 3}$) and averaged to a single 3D point for each model. Just like before, a polar plot is created using the vectors and angle between them ($27.8^\circ$).

{\bf Similarity metrics.\enspace} We report RSA and CKA values computed between DCL and NSCL models trained on CIFAR100. Additionally, we show their average weight gap as detailed in Sec.~\ref{sec:exp_results}. It is evident that models stay aligned in representation space but diverge in weight space.

\section{Parameter-Space Coupling}\label{app:params}

To complement the analysis in Sec.~\ref{sec:theory}, we compare the two trajectories in parameter space. Let $e_t = \|w^{\CL}_t - w^{\NS}_t\|$ denote the parameter drift at step $t$. We would like to bound it as a function of the number of training iterations, batch size, and learning rate scheduling. We use classic techniques that can be found at~\citep{10.1162/153244302760200704,10.5555/3045390.3045520,pmlr-v75-mou18a,KuzborskijLampert2017DataDependentStability}.

{\bf Optimization.\enspace} In order to isolate the effect of the loss, we optimize both objectives (CL and NSCL) with standard mini-batch SGD under a single coupled protocol: at step \(t\) we draw a batch \(\mathcal{B}_t=\{(x_{j},x'_{j},y_{j})\}_{j=1}^{B}\) with replacement, where each \(x'_j\sim\alpha(x_j)\) (e.g., random crop/resize, horizontal flip, color jitter, Gaussian blur); we average per-anchor terms to form either \(\bar{\ell}^{\mathrm{CL}}_{\mathcal{B}_t}(w)\) or \(\bar{\ell}^{\mathrm{NS}}_{\mathcal{B}_t}(w)\) using cosine similarity (optionally temperature-scaled), hence bounded in \([-1,1]\); and we update \(w_{t+1}=w_t-\eta_t\nabla \bar{\ell}_{\mathcal{B}_t}(w_t)\) with prescribed \(\eta_t>0\). We then run two coupled SGD trajectories from the same initialization \(w^{\CL}_0=w^{\NS}_0\) that share the \emph{same} batches and augmentations \((\mathcal B_t)_{t=0}^{T-1}\) and differ only by NSCL’s exclusion of same-class negatives:
\[
\begin{aligned}
w^{\CL}_{t+1}&=w^{\CL}_t-\eta_t\,\nabla \bar{\ell}^{\CL}_{\mathcal B_t}\bigl(w^{\CL}_t\bigr),\qquad
w^{\NS}_{t+1}=w^{\NS}_t-\eta_t\,\nabla \bar{\ell}^{\NS}_{\mathcal B_t} \bigl(w^{\NS}_t\bigr),\;\; t=0,\dots,T-1.
\end{aligned}
\]

Throughout the analysis, we make standard assumptions on the smoothness of the loss functions and the scale of gradients.

{\bf Assumptions.\enspace} To control the dynamics, we impose two standard conditions on the geometry of the batch objectives and the scale of pairwise gradients.

\begin{assumption}[Uniform smoothness]\label{asm:smoothness}
For every batch $\mathcal{B}$, the functions $w\mapsto \bar{\ell}^{\CL}_{\mathcal B}(w)$ and $w\mapsto \bar{\ell}^{\NS}_{\mathcal B}(w)$ are $\beta$-smooth with the same constant $\beta>0$:
\[
\left\|\nabla \phi(w)-\nabla \phi(v)\right\| ~\le~ \beta \,\|w-v\|
\quad\text{for all }v,w\in\mathbb{R}^p,\ \phi\in\{\bar{\ell}^{\CL}_{\mathcal B},\bar{\ell}^{\NS}_{\mathcal B}\}.
\]
\end{assumption}

\begin{assumption}[Bounded pairwise gradients]\label{asm:pairwise-G}
There exists $G>0$, independent of $\mathcal B$ and $t$, such that for all $w$ and all pairs $(u,v)$ appearing in any denominator term,
\[
\left\|\nabla_w\,\mathrm{sim}\left(f_w(u),f_w(v)\right)\right\| ~\le~ G.
\]
\end{assumption}

We quantify drift between the coupled trajectories under shared randomness in the \emph{nonconvex $\beta$-smooth} regime. Throughout, the only data-dependent term is $\Delta_{C,\delta}(B;\tau)$, which decreases with more classes and larger batches.

\begin{restatable}{theorem}{nonconvex}\label{thm:nonconvex}
Fix $B,T\in\mathbb N$, $\delta\in(0,1)$, and temperature $\tau>0$. 
Suppose Assumptions~\ref{asm:smoothness}--\ref{asm:pairwise-G} hold. Then, with probability at least $1-\delta$,
\[
e_T ~\le~ \frac{G}{\beta\,\tau}\,\Delta_{C,\delta}(B;\tau)\,
\Bigl(\exp\Bigl(\beta\sum_{t=0}^{T-1}\eta_t\Bigr)-1\Bigr).
\]
\end{restatable}

The bound scales linearly with $G$ and $\Delta_{C,\delta}(B;\tau)$, but crucially it is amplified by the exponential factor $\exp(\beta\sum_t\eta_t)$. Unless the step sizes are aggressively annealed, this term grows rapidly with training time. Even though $\Delta_{C,\delta}(B;\tau)$ improves with $C$ and $B$ (e.g., for $C{=}1000$, $B{=}512$, $\delta{=}0.01$, we obtain $\Delta_{C,\delta}(B;\tau)\approx 0.01$ so that the reweightings of the steps differ by about one percent), the exponential accumulation can still overwhelm this small per-step gap. 

In other words, parameter-space coupling guarantees only that the two runs do not drift apart too quickly in weight space. But because the weights may follow very different trajectories even when representations remain similar, this control is too weak to yield meaningful statements about representational alignment. This motivates our next step: shifting the analysis to \emph{similarity space}, where we can obtain bounds that remain stable throughout training and translate directly into guarantees on metrics such as CKA and RSA.

{\bf Proof idea.\enspace} With high probability over batches (Cor.~\ref{cor:comp-hp}), every anchor’s denominator is dominated by negatives up to $\epsilon_{B,\delta}$ fluctuations. This keeps the (temperature–$\tau$) softmax reweighting gap between CL and NSCL small. In particular, Lem.~\ref{lem:gap-param} shows that the per-batch parameter gradients differ uniformly as
\[
\bigl\|\nabla\bar\ell^{\CL}_{\mathcal B_t}(w)-\nabla\bar\ell^{\NS}_{\mathcal B_t}(w)\bigr\|
~\le~ \frac{G}{\tau}\,\Delta_{C,\delta}(B;\tau).
\]

By $\beta$-smoothness of each batch loss, each step can expand distances by at most a factor 
$(1+\beta\eta_t)$. Combining this smoothness expansion with the uniform gradient-gap bound 
yields the following recurrence:
\[
e_{t+1}~\le~(1+\beta\eta_t)\,e_t  ~+~\eta_t\,\frac{G}{\tau}\,\Delta_{C,\delta}(B;\tau),
\]
where the first term propagates the previous error (with amplification controlled by curvature), 
and the second injects the new discrepancy introduced by the CL–NSCL gap at temperature $\tau$.

Unrolling over $T$ steps and applying the discrete Gr\"onwall inequality gives the exponential-type bound
\[
e_T ~\le~ \frac{G}{\beta\,\tau}\,\Delta_{C,\delta}(B;\tau)\,
\Bigl(\exp\Bigl(\beta\sum_{t=0}^{T-1}\eta_t\Bigr)-1\Bigr).
\]
Thus, cumulative drift scales with the reweighting gap and is amplified exponentially 
with the total step size; smaller $\tau$ tightens the softmax and increases the constants (via both $1/\tau$ and $\mathrm e^{2/\tau}$ inside $\Delta_{C,\delta}$), so keeping $\sum_t \eta_t$ moderate is especially important.

\section{Why gradient descent in similarity space is a faithful surrogate}
\label{app:sigma-just}

We now explain why running gradient descent directly in similarity space closely tracks the dynamics induced by gradient descent in parameter space.  
 
When parameters move from $w_t$ to $w_{t+1}$, the induced change in the similarity matrix can be approximated by a linear expansion:
\begin{equation}
\Sigma(w_{t+1}) - \Sigma(w_t) ~\approx~ J_t (w_{t+1} - w_t), 
\qquad J_t ~:=~ J(w_t),
\end{equation}
where $J(w) := \partial \Sigma / \partial w$ is the Jacobian. The error in this expansion, denoted $R_t$, is quadratic in the step size:
\begin{equation}\label{eq:sigma-lin}
\Sigma(w_{t+1}) - \Sigma(w_t) ~=~ J_t (w_{t+1} - w_t) + R_t.
\end{equation}
  
By the chain rule, the gradient in parameter space can be written as follows:
\[
\nabla_w \bar\ell(w_t) ~=~ J_t^\top \nabla_\Sigma \bar\ell(\Sigma(w_t)) ~=~ J_t^\top \widehat G_t,
\]
where $\widehat G_t := \nabla_\Sigma \bar\ell(\Sigma(w_t))$. Substituting this into the update rule gives
\begin{equation}\label{eq:precond}
\Sigma(w_{t+1}) - \Sigma(w_t)
~=~ -\eta_t\, P_t \widehat G_t + R_t,
\qquad P_t ~:=~ J_t J_t^\top ~\succeq~ 0.
\end{equation}
Thus, parameter descent acts like similarity descent, but with a preconditioning matrix $P_t$, plus the remainder $R_t$.
 
Assume there exist constants $L_\Sigma, M_\Sigma > 0$ such that
\[
\|J(w)\|_{2\to 2} ~\le~ L_\Sigma, 
\qquad 
\|\Sigma(w+\Delta w) - \Sigma(w) - J(w)\Delta w\|_F ~\le~ \frac{M_\Sigma}{2}\|\Delta w\|_2^2.
\]
Then $\|P_t\|_{2\to2} \le L_\Sigma^2$ and, with $\Delta w_t := -\eta_t \nabla_w \bar\ell(w_t)$,
\begin{equation}\label{eq:R-bound-appendix}
\|R_t\|_F ~\le~ \frac{M_\Sigma}{2}\eta_t^2 \|\nabla_w \bar\ell(w_t)\|_2^2 ~=:~ \frac{M_\Sigma}{2}\eta_t^2 \Xi_t.
\end{equation}
 
Let $\widehat\Sigma_t := \Sigma(w_t)$ be the similarity trajectory induced by parameter descent. Define $\widetilde\Sigma_t$ as the trajectory of explicit similarity descent:
\[
\widetilde\Sigma_{t+1} ~=~ \widetilde\Sigma_t - \eta_t \widetilde G_t,
\qquad 
\widetilde G_t ~:=~ \nabla_\Sigma \bar\ell(\widetilde\Sigma_t),
\]
with $\widehat\Sigma_0 = \widetilde\Sigma_0$.  
Let $E_t := \|\widehat\Sigma_t - \widetilde\Sigma_t\|_F$ and $C_\Sigma := \sup_t \|P_t - I\|_{2\to2} \le L_\Sigma^2 + 1$.  
Using \eqref{eq:precond}, adding and subtracting $-\eta_t \widehat G_t$, and applying the temperature–$\tau$ bounds \eqref{eq:hess-lip} and \eqref{eq:R-bound-appendix}, one obtains
\begin{equation}\label{eq:E-rec}
E_{t+1} \;\le\; \Bigl(1 + \frac{\eta_t}{2\tau^2 B}\Bigr) \, E_t  ~+~ \eta_t C_\Sigma \|\widehat G_t\|_F  ~+~ \frac{M_\Sigma}{2}\eta_t^2 \Xi_t.
\end{equation}
Unrolling this recursion from $E_0=0$ and using $\prod_u (1+\alpha_u) \le \exp(\sum_u \alpha_u)$ yields
\begin{equation}\label{eq:E-unrolled}
\|\widehat\Sigma_T - \widetilde\Sigma_T\|_F
\;\le\; \exp\Bigl(\frac{1}{2\tau^2 B}\sum_{t=0}^{T-1} \eta_t\Bigr)
\left[
C_\Sigma \sum_{t=0}^{T-1} \eta_t \|\widehat G_t\|_F
~+~ \frac{M_\Sigma}{2}\sum_{t=0}^{T-1} \eta_t^2 \Xi_t
\right].
\end{equation}
By bounding $\|\widehat G_t\|_F$ via \eqref{eq:GhatF}, namely $\|\widehat G_t\|_F \le \frac{1}{\tau}\sqrt{\frac{2}{B}}$, this simplifies to
\begin{equation}\label{eq:E-unrolled-bs}
\|\widehat\Sigma_T - \widetilde\Sigma_T\|_F
\;\le\; \exp\Bigl(\frac{1}{2\tau^2 B}\sum^{T-1}_{t=0} \eta_t\Bigr)
\left[
\frac{\sqrt{2}\,C_\Sigma}{\tau\sqrt{B}} \sum^{T-1}_{t=0} \eta_t
~+~ \frac{M_\Sigma}{2} \sum^{T-1}_{t=0} \eta_t^2 \Xi_t
\right].
\end{equation}

To summarize, the difference between the two trajectories stays small when step sizes are not too aggressive (moderate cumulative step size), the schedule is square-summable (so $\sum_t \eta_t^2$ remains finite), and the batch size $B$ is not too small. The temperature $\tau$ modulates both curvature (via $1/\tau^2$) and gradient magnitudes (via $1/\tau$); smaller $\tau$ tightens the softmax and makes the coupling more sensitive to step size.

\section{Technical Tools and Proofs}\label{app:appendix-main}

\subsection{Notation and basic softmax facts}\label{app:notation-softmax}
Let $S=\{(x_i,y_i)\}_{i=1}^N$ be class-balanced with $C$ classes and $N=Cn$ (each class has $n$ points). 
For parameters $w$, let $z_i=f_w(x_i)$ and define the bounded similarity matrix  
\[
\Sigma(w)_{ij} ~:=~ \mathrm{sim}\bigl(z_i,z_j\bigr)\in[-1,1].
\]
At step $t$, draw a mini-batch $\mathcal B_t=\{(x_{j_s},x'_{j_s},y_{j_s})\}_{s=1}^B$ with replacement, using independent augmentations $x'_{j_s}\sim\alpha(x_{j_s})$. 
For an \emph{anchor} $i\in\{j_1,\dots,j_B\}$, let $D_i$ be its denominator index set, and let $D_i^{\negative} := \{k\in D_i:\ y_k\neq y_i\}$ (and similarly $D^{\positive}_i$) denote the subset restricted to negatives 
(e.g., in two-view SimCLR, $D_i$ consists of all $2B$ views except the anchor itself).  

Define the anchor’s logit vector $s_i(w) := \bigl(\Sigma(w)_{i,k}\bigr)_{k\in D_i}$ and the corresponding softmax distributions with temperature $\tau>0$ (default $1$):
\[
p_i=\softmax \bigl(s_i(w)/\tau\bigr),
\qquad
q_i=\softmax \bigl((s_i(w))_{D_i^{\negative}}/\tau\bigr).
\]
Let $i'$ denote the positive (augmented) index for anchor $i$.  

For contrastive learning (CL) and negatives-only supervised contrastive learning (NSCL), the per-anchor and batch losses are
\[
\ell^{\CL}_i(s_i)~=~-\log p_{i,i'},
\qquad
\ell^{\NS}_i(s_i)~=~-\log q_{i,i'},
\]
\[
\bar\ell^{\CL}_{\mathcal B_t}~=~\frac{1}{B}\sum_{i\in\{j_1,\dots,j_B\}}\ell^{\CL}_i(s_i),
\qquad
\bar\ell^{\NS}_{\mathcal B_t}~=~\frac{1}{B}\sum_{i\in\{j_1,\dots,j_B\}}\ell^{\NS}_i(s_i).
\]

Since $\Sigma(w)_{ij}\in[-1,1]$, each exponential term inside the softmax lies in
\[
\exp\bigl(\Sigma(w)_{ij}/\tau\bigr)\in[\mathrm e^{-1/\tau},\mathrm e^{1/\tau}],
\]
a fact used below to control softmax mass ratios.

\begin{lemma}[Anchor-block orthogonality]\label{lem:block-orth}
Fix a step $t$ and batch $\mathcal B_t$. For each anchor $i\in\mathcal B_t$, let $D_i$ be the set of indices appearing in $i$'s denominator and define the per-anchor gradient $g_i\in\mathbb{R}^{\mathcal I_t}$ by
\[
g_i ~:=~ \nabla_{s_i}\ell_i \quad\text{placed on the coordinates } \{(i,k):k\in D_i\}\subset \mathcal I_t,
\]
with zeros elsewhere (here $\mathcal I_t$ is the set of all coordinates touched at step $t$). If $i\neq j$, then $g_i$ and $g_j$ have disjoint supports, and hence
\[
\langle g_i,\,g_j\rangle_F  ~=~ 0 .
\]
Consequently, for the batch gradient $G=\frac{1}{B}\sum_{i\in\mathcal B_t} g_i$,
\begin{equation}\label{eq:block-orth-sum}
\|G\|_F^2 ~=~ \frac{1}{B^2}\sum_{i\in\mathcal B_t}\|g_i\|_F^2 .
\end{equation}
\end{lemma}

\begin{proof}
By construction, $g_i$ is supported only on coordinates $\{(i,k):k\in D_i\}$, while $g_j$ is supported only on $\{(j,k):k\in D_j\}$. For $i\neq j$ these sets are disjoint, so every coordinatewise product is zero, yielding $\langle g_i,g_j\rangle_F=0$. Expanding the square for $G$,
\[
\|G\|_F^2
~=~\Bigl\langle \frac{1}{B}\sum_i g_i,\; \frac{1}{B}\sum_j g_j\Bigr\rangle_F
~=~\frac{1}{B^2}\sum_i \|g_i\|_F^2 + \frac{1}{B^2}\sum_{i\neq j}\langle g_i,g_j\rangle_F
~=~\frac{1}{B^2}\sum_i \|g_i\|_F^2,
\]
where the cross terms vanish by orthogonality.
\end{proof}

\begin{lemma}[Softmax Hessian and gradient Lipschitzness]\label{lem:sigma-smooth}
Fix a step $t$ and batch $\mathcal B_t$. Let $\mathcal I_t$ be the set of coordinates $(i,k)$ that appear in any anchor’s denominator at step $t$, and view $\bar\ell_{\mathcal B_t}$ (either CL or NSCL) as a function of the restricted similarity entries $\Sigma\in\mathbb{R}^{\mathcal I_t}$. For each anchor $i$, write $s_i=\{\Sigma(i,k):(i,k)\in \mathcal I_t\}$ and $p_i=\softmax(s_i/\tau)$. Then:
\[
\nabla^2_{s_i}\ell_i(s_i)~=~\frac{1}{\tau^2}\,J(s_i),\quad J(s_i):=\mathrm{Diag}(p_i)-p_i p_i^\top,
\qquad
\bigl\|\nabla^2 \bar\ell_{\mathcal B_t}(\Sigma)\bigr\|_{2\to 2} ~\le~ \frac{1}{2\tau^2 B}.
\]
Consequently, for all $\Sigma,\widetilde\Sigma\in\mathbb{R}^{\mathcal I_t}$,
\begin{equation}\label{eq:hess-lip}
\bigl\|\nabla_\Sigma \bar\ell_{\mathcal B_t}(\Sigma)-\nabla_\Sigma \bar\ell_{\mathcal B_t}(\widetilde\Sigma)\bigr\|_F ~\le~ \frac{1}{2\tau^2 B}\,\|\Sigma-\widetilde\Sigma\|_F.
\end{equation}
\end{lemma}

\begin{proof}
With temperature $\tau>0$, for an anchor $i$ we have $p_i=\softmax(s_i/\tau)$ and
\[
\nabla_{s_i}\ell_i(s_i) ~=~\frac{1}{\tau}\,(p_i-e_{i'})\quad\Longrightarrow\quad
\nabla^2_{s_i}\ell_i(s_i) ~=~\frac{1}{\tau^2}\,\nabla_{s_i}p_i ~=~\frac{1}{\tau^2}\,J(s_i),
\]
where $J(s_i):=\mathrm{Diag}(p_i)-p_i p_i^\top$. Bound $\|J(s_i)\|_{2\to 2}$ via the infinity norm:
\begin{align*}
\|J(s_i)\|_{2\to 2}
~&\le~ \|J(s_i)\|_\infty \\
~&=~ \max_r \sum_{\ell} |J_{r\ell}| \\
~&=~ \max_r \Bigl(p_{i,r}(1-p_{i,r})+\sum_{\ell\neq r} p_{i,r}p_{i,\ell}\Bigr) \\
~&=~ \max_r 2p_{i,r}(1-p_{i,r}) ~\le~ \tfrac12,
\end{align*}
since $x(1-x)\le 1/4$ for $x\in[0,1]$.

The batch loss is an average over anchors, so its Hessian is block-diagonal across anchors with a prefactor $1/B$:
\[
\nabla^2 \bar\ell_{\mathcal B_t}(\Sigma)
~=~\frac{1}{B}\,\mathrm{blkdiag} \Bigl(\tfrac{1}{\tau^2}J(s_i)\Bigr)_{i\in\mathcal B_t}
~=~\frac{1}{\tau^2 B}\,\mathrm{blkdiag}\bigl(J(s_i)\bigr)_{i\in\mathcal B_t}.
\]
Hence
\[
\bigl\|\nabla^2 \bar\ell_{\mathcal B_t}(\Sigma)\bigr\|_{2\to 2}
=\frac{1}{\tau^2 B}\max_{i}\|J(s_i)\|_{2\to 2}
\;\le\; \frac{1}{2\tau^2 B}.
\]
By the mean-value (integral) form for vector fields,
\[
\nabla_\Sigma \bar\ell_{\mathcal B_t}(\Sigma)-\nabla_\Sigma \bar\ell_{\mathcal B_t}(\widetilde\Sigma)
~=~ \int_0^1 \nabla^2 \bar\ell_{\mathcal B_t}\bigl(\widetilde\Sigma+\theta(\Sigma-\widetilde\Sigma)\bigr)\,[\Sigma-\widetilde\Sigma]\;d\theta,
\]
and therefore
\[
\bigl\|\nabla_\Sigma \bar\ell_{\mathcal B_t}(\Sigma)-\nabla_\Sigma \bar\ell_{\mathcal B_t}(\widetilde\Sigma)\bigr\|_F
~\le~ \sup_{\theta\in[0,1]}\bigl\|\nabla^2 \bar\ell_{\mathcal B_t}(\Sigma_\theta)\bigr\|_{2\to 2}\,\|\Sigma-\widetilde\Sigma\|_F
~\le~ \frac{1}{2\tau^2 B}\,\|\Sigma-\widetilde\Sigma\|_F,
\]
as claimed.
\end{proof}

\begin{lemma}[Per-anchor gradient norm and batch average]\label{lem:anchor-grad-and-batch}
For an anchor $i$, let $s_i$ be the vector of logits in its denominator and $p_i=\softmax(s_i/\tau)$.
Let $i'$ denote the (unique) positive index (for NSCL, if $i'$ is not in the denominator, set $p_{i,i'}:=0$ in the display below). Then
\begin{equation}\label{eq:anchor-grad-norm}
\|\nabla_{s_i}\ell_i\|_2^2
~=~ \frac{1}{\tau^2}\Bigl[(1-p_{i,i'})^2+\sum_{k\ne i'} p_{i,k}^2\Bigr]
\;\le\; \frac{2}{\tau^2},
\end{equation}
hence $\|\nabla_{s_i}\ell_i\|_2\le \sqrt{2}/\tau$. Moreover, by block orthogonality across anchors,
\begin{equation}\label{eq:GhatF}
\Bigl\|\frac{1}{B}\sum_{i\in\mathcal B_t}\nabla_{s_i}\ell_i\Bigr\|_F^2
~=~\frac{1}{B^2}\sum_{i\in\mathcal B_t}\|\nabla_{s_i}\ell_i\|_2^2
\;\le\; \frac{2}{\tau^2 B}
\quad\Longrightarrow\quad
\Bigl\|\frac{1}{B}\sum_{i\in\mathcal B_t}\nabla_{s_i}\ell_i\Bigr\|_F \;\le\; \frac{1}{\tau}\sqrt{\frac{2}{B}}.
\end{equation}
\end{lemma}

\begin{proof}
For CL, the loss is $-\log p_{i,i'}$ with $p_i=\softmax(s_i/\tau)$. By the standard softmax–cross-entropy derivative with temperature,
\[
\nabla_{s_i}\ell_i  ~=~ \frac{1}{\tau}\,(p_i - e_{i'}),
\]
so
\[
\|\nabla_{s_i}\ell_i\|_2^2
= \frac{1}{\tau^2} \left[(1-p_{i,i'})^2+\sum_{k\ne i'} p_{i,k}^2\right]
\le \frac{1}{\tau^2} \left[(1-p_{i,i'})^2+\Big(\sum_{k\ne i'} p_{i,k}\Big)^2\right]
= \frac{2}{\tau^2}(1-p_{i,i'})^2
\le \frac{2}{\tau^2},
\]
since $p_i$ is a probability vector and $\sum_{k\ne i'}p_{i,k}=1-p_{i,i'}$.

For NSCL, two cases. If $i'\in D_i$, the same computation applies (the target index is present), hence the same bound holds. If $i'\notin D_i$ (negatives-only denominator), then the loss is $-\log q_{i,i'}$ with $q_i=\softmax\bigl((s_i)_{D_i^{\neg}}/\tau\bigr)$ supported only on $D_i^{\neg}$, and
\[
\nabla_{s_i}\ell_i  ~=~ \frac{1}{\tau}\,q_i \quad\text{on }D_i^{\negative}\quad(\text{and }0\text{ on }D_i^{\positive}),
\]
so
\[
\|\nabla_{s_i}\ell_i\|_2^2  ~=~ \frac{1}{\tau^2}\sum_{j\in D_i^{\neg}} q_{i,j}^2
\;\le\; \frac{1}{\tau^2}\Big(\sum_{j\in D_i^{\neg}} q_{i,j}\Big)^2
 ~=~ \frac{1}{\tau^2}
\;\le\; \frac{2}{\tau^2}.
\]
Thus in all cases $\|\nabla_{s_i}\ell_i\|_2\le \sqrt{2}/\tau$, establishing \eqref{eq:anchor-grad-norm}.

For the batch bound \eqref{eq:GhatF}, gradients from different anchors have disjoint supports over coordinates $\{(i,k):k\in D_i\}$, so they are orthogonal in Frobenius inner product (Lem.~\ref{lem:block-orth}). Therefore,
\[
\Bigl\|\frac{1}{B}\sum_{i\in\mathcal B_t}\nabla_{s_i}\ell_i\Bigr\|_F^2
=\frac{1}{B^2}\sum_{i\in\mathcal B_t}\|\nabla_{s_i}\ell_i\|_2^2
\le \frac{1}{B^2}\cdot B\cdot \frac{2}{\tau^2}
= \frac{2}{\tau^2 B},
\]
which also implies
\(
\bigl\|\frac{1}{B}\sum_{i\in\mathcal B_t}\nabla_{s_i}\ell_i\bigr\|_F
\le \frac{1}{\tau}\sqrt{2/B}.
\)
\end{proof}

\begin{lemma}[Bounded logits imply bounded softmax masses]\label{lem:bounded-logits}
Fix a step $t$ and an anchor $i$. Suppose all active logits satisfy $\Sigma(i,k)\in[-1,1]$. For any index subset $S$ in the anchor’s denominator, define
\[
Z_S \;:=\; \sum_{k\in S}\exp\bigl(\Sigma(i,k)/\tau\bigr)
\quad\text{with temperature }\tau>0.
\]
Then
\[
|S|\,\mathrm e^{-1/\tau}\ \le\ Z_S\ \le\ |S|\,\mathrm e^{1/\tau}.
\]
In particular, if $S_{\mathrm{pos}}$ and $S_{\mathrm{neg}}$ are the positive and negative index sets with sizes $n_{\mathrm{pos}}$ and $n_{\mathrm{neg}}$, and $Z_{\mathrm{pos}}:=Z_{S_{\mathrm{pos}}}$, $Z_{\mathrm{neg}}:=Z_{S_{\mathrm{neg}}}$, then
\[
n_{\mathrm{pos}}\,\mathrm e^{-1/\tau} ~\le~ Z_{\mathrm{pos}} ~\le~ n_{\mathrm{pos}}\,\mathrm e^{1/\tau},
\qquad
n_{\mathrm{neg}}\,\mathrm e^{-1/\tau} ~\le~ Z_{\mathrm{neg}} ~\le~ n_{\mathrm{neg}}\,\mathrm e^{1/\tau},
\]
and hence
\[
\frac{Z_{\mathrm{pos}}}{Z_{\mathrm{neg}}} ~\le~ \mathrm e^{2/\tau}\,\frac{n_{\mathrm{pos}}}{n_{\mathrm{neg}}}
\quad\text{and}\quad
\frac{Z_{\mathrm{pos}}}{Z_{\mathrm{neg}}} ~\ge~ \mathrm e^{-2/\tau}\,\frac{n_{\mathrm{pos}}}{n_{\mathrm{neg}}}.
\]
\end{lemma}

\begin{proof}
Since $\Sigma(i,k)\in[-1,1]$, we have $\exp(\Sigma(i,k)/\tau)\in[\mathrm e^{-1/\tau},\mathrm e^{1/\tau}]$ for every active $k$. Summing over $k\in S$ yields
$|S|\,\mathrm e^{-1/\tau}\le Z_S\le |S|\,\mathrm e^{1/\tau}$. Apply this with $S=S_{\mathrm{pos}}$ and $S=S_{\mathrm{neg}}$ and take ratios to obtain the stated bounds.
\end{proof}

\subsection{High-probability batch composition}\label{app:batch-comp}

Fix $T,B\in\mathbb N$ and $\epsilon>0$. For step $t$ and anchor $i\in\mathcal B_t$, let $Y^{(i)}_{t,s}=\mathbf 1\{y_{j_s}\ne y_i\}$ for $s=1,\dots,B$.

\begin{lemma}[Batch-composition event]\label{lem:comp}
For a class-balanced population with $C$ classes, the $Y^{(i)}_{t,s}$ are i.i.d.\ Bernoulli with mean $1-\tfrac1C$. For any $\epsilon>0$,
\[
\mathbb P\left[\exists (t,i):~ \frac{1}{B}\sum_{s=1}^B Y^{(i)}_{t,s} ~<~ 1-\frac{1}{C}-\epsilon\right]
~\le~ TB\,\mathrm e^{-2B\epsilon^2}.
\]
Equivalently, with probability $\ge 1-TB\,\mathrm e^{-2B\epsilon^2}$, every anchor sees at least $B(1-\tfrac1C-\epsilon)$ negatives.
\end{lemma}

\begin{proof}
Fix any step $t$ and anchor $i$. Because batches are drawn with replacement from a class-balanced population with $C$ classes, for each position $s\in\{1,\dots,B\}$ the indicator
\(
Y^{(i)}_{t,s}=\mathbf 1\{y_{j_s}\neq y_i\}
\)
is Bernoulli with mean
\(
\mathbb E[Y^{(i)}_{t,s}]=1-\tfrac{1}{C},
\)
and $\{Y^{(i)}_{t,s}\}_{s=1}^B$ are i.i.d. across $s$. By Hoeffding’s inequality, for any $\epsilon>0$,
\[
\mathbb P\left[\frac{1}{B}\sum_{s=1}^B Y^{(i)}_{t,s} ~<~ 1-\frac{1}{C}-\epsilon\right]
~=~\mathbb P\left[\frac{1}{B}\sum_{s=1}^B\bigl(Y^{(i)}_{t,s}-\mathbb E Y^{(i)}_{t,s}\bigr)~<~-\epsilon\right]
~\le~ \exp(-2B\epsilon^2).
\]
There are at most $TB$ anchor–step pairs $(t,i)$ over $t=0,\dots,T-1$ and $i\in\mathcal B_t$. A union bound gives
\[
\mathbb P\left[\exists (t,i):\ \frac{1}{B}\sum_{s=1}^B Y^{(i)}_{t,s} ~<~ 1-\frac{1}{C}-\epsilon\right]
~\le~ TB\,\mathrm e^{-2B\epsilon^2}.
\]
Equivalently, with probability at least $1-TB\,\mathrm e^{-2B\epsilon^2}$, every anchor in every step has at least $B(1-\tfrac{1}{C}-\epsilon)$ negatives in its denominator.
\end{proof}

\begin{corollary}\label{cor:comp-hp}
For $\delta\in(0,1)$, set $\epsilon_{B,\delta}:=\sqrt{\tfrac{1}{2B}\log(\tfrac{TB}{\delta})}$. With probability $\ge 1-\delta$, every anchor $i$ has at least $B(1-\tfrac1C-\epsilon_{B,\delta})$ negatives and at most $B(\tfrac1C+\epsilon_{B,\delta})$ positives in its denominator. Using bounded logits, the ratio of total positive to negative softmax mass (at temperature $\tau>0$) satisfies, for all anchors and steps,
\begin{equation}\label{eq:pos-neg-ratio}
\frac{Z_i^{\positive}}{Z_i^{\negative}}
~\le~ \frac{\mathrm e^{2/\tau}\bigl(\tfrac1C+\epsilon_{B,\delta}\bigr)}{1-\tfrac1C-\epsilon_{B,\delta}}
~=~ \tfrac12\,\Delta_{C,\delta}(B;\tau),
\end{equation}
where
\[
\Delta_{C,\delta}(B;\tau)\;:=\;\frac{2\,\mathrm e^{2/\tau}\bigl(\tfrac{1}{C}+\epsilon_{B,\delta}\bigr)}{1-\tfrac{1}{C}-\epsilon_{B,\delta}}.
\]
\end{corollary}

\begin{proof}
Set $\epsilon=\epsilon_{B,\delta}:=\sqrt{\tfrac{1}{2B}\log(\tfrac{TB}{\delta})}$ and 
$\Delta_{C,\delta}(B;\tau):=\tfrac{2\,\mathrm e^{2/\tau}\bigl(\tfrac{1}{C}+\epsilon_{B,\delta}\bigr)}{1-\tfrac{1}{C}-\epsilon_{B,\delta}}$.
Apply Lem.~\ref{lem:comp} with this $\epsilon$: with probability at least $1-\delta$, for every step $t$ and every anchor $i$,
\[
|D_i^{\negative}| \;\ge\; B\Bigl(1-\tfrac{1}{C}-\epsilon_{B,\delta}\Bigr),
\qquad
|D_i^{\positive}| \;\le\; B\Bigl(\tfrac{1}{C}+\epsilon_{B,\delta}\Bigr).
\]
In two-view SimCLR, each sampled point contributes two denominator entries, so the denominator contains at least $2|D_i^{\negative}|$ negative entries and at most $2|D_i^{\positive}|$ positive entries; the factor $2$ cancels in the ratio below.

Because similarities are bounded in $[-1,1]$, each logit lies in $[-1,1]$ and hence each exponential term at temperature $\tau$ lies in $[\mathrm e^{-1/\tau},\mathrm e^{1/\tau}]$. Therefore, for any anchor and step,
\[
Z_i^{\positive} \;\le\; \mathrm e^{1/\tau} \cdot (2|D_i^{\positive}|),
\qquad
Z_i^{\negative} \;\ge\; \mathrm e^{-1/\tau}\cdot (2|D_i^{\negative}|),
\]
and thus
\[
\frac{Z_i^{\positive}}{Z_i^{\negative}}
\;\le\; \mathrm e^{2/\tau}\,\frac{|D_i^{\positive}|}{|D_i^{\negative}|}
\;\le\; \frac{\mathrm e^{2/\tau}\,\bigl(\tfrac{1}{C}+\epsilon_{B,\delta}\bigr)}
{1-\tfrac{1}{C}-\epsilon_{B,\delta}}
~=~\tfrac12\,\Delta_{C,\delta}(B;\tau).
\]
The bound is meaningful whenever $\epsilon_{B,\delta}<1-\tfrac{1}{C}$ so that the denominator is positive. This proves the corollary.
\end{proof}

\begin{lemma}[Per-anchor reweighting gap]\label{lem:softmax-reweight}
On the event of Cor.~\ref{cor:comp-hp}, let $p$ be the CL softmax (temperature $\tau>0$) over an anchor’s full denominator, and $q$ the NSCL softmax (same $\tau$) that removes same-class entries and renormalizes over negatives. Then
\[
\|p-q\|_1 ~\le~ \Delta_{C,\delta}(B;\tau),
\qquad
\|p-q\|_2 ~\le~ \|p-q\|_1 ~\le~ \Delta_{C,\delta}(B;\tau).
\]
\end{lemma}

\begin{proof}
Fix an anchor $i$ and let $D_i^{\positive}, D_i^{\negative}$ be its positive and negative index sets in the CL denominator. Write $s_k:=\Sigma(i,k)$ and define
\[
Z_i^{\positive} ~:=~ \sum_{k\in D_i^{\positive}}\exp\bigl(s_k/\tau\bigr),\qquad
Z_i^{\negative} ~:=~ \sum_{j\in D_i^{\negative}}\exp\bigl(s_j/\tau\bigr),\qquad
\alpha ~:=~ \frac{Z_i^{\positive}}{Z_i^{\positive}+Z_i^{\negative}}.
\]
Let $p$ be the CL softmax on $D_i^{\positive}\cup D_i^{\negative}$ and let $q$ be the NSCL softmax that zeros positive entries and renormalizes on negatives: $q(k)=0$ for $k\in D_i^{\positive}$ and $q(j)=p(j)/(1-\alpha)$ for $j\in D_i^{\negative}$. Then
\[
\|p-q\|_1
= \sum_{k\in D_i^{\positive}}p_k
+ \sum_{j\in D_i^{\negative}}\Bigl|p_j-\frac{p_j}{1-\alpha}\Bigr|
= \alpha + (1-\alpha)\frac{\alpha}{1-\alpha}
= 2\alpha
\le \frac{2Z_i^{\positive}}{Z_i^{\negative}}.
\]
On the high-probability event of Cor.~\ref{cor:comp-hp}, since $s\in[-1,1]\Rightarrow \exp(s/\tau)\in[\mathrm e^{-1/\tau},\mathrm e^{1/\tau}]$,
\[
Z_i^{\positive} ~\le~ \mathrm e^{1/\tau}\,|D_i^{\positive}|,
\qquad
Z_i^{\negative} ~\ge~ \mathrm e^{-1/\tau}\,|D_i^{\negative}|.
\]
Moreover,
\[
|D_i^{\positive}| ~\le~ 2B\Bigl(\tfrac{1}{C}+\epsilon_{B,\delta}\Bigr),
\qquad
|D_i^{\negative}| ~\ge~ 2B\Bigl(1-\tfrac{1}{C}-\epsilon_{B,\delta}\Bigr),
\]
(each sampled point contributes two keys, so the factor $2$ cancels in the ratio). Hence
\[
\frac{2Z_i^{\positive}}{Z_i^{\negative}}
~\le~ 2\,\mathrm e^{2/\tau}\,
\frac{|D_i^{\positive}|}{|D_i^{\negative}|}
~\le~ \frac{2\,\mathrm e^{2/\tau}\bigl(\tfrac{1}{C}+\epsilon_{B,\delta}\bigr)}
{1-\tfrac{1}{C}-\epsilon_{B,\delta}}
~=:~ \Delta_{C,\delta}(B;\tau).
\]
Therefore $\|p-q\|_1 \le \Delta_{C,\delta}(B;\tau)$. Finally, $\|p-q\|_2 \le \|p-q\|_1$ yields the second claim.
\end{proof}

\subsection{Parameter-space coupling: supporting lemmas and proofs}\label{app:param-space}


\begin{lemma}[Per-batch parameter-gradient gap]\label{lem:gap-param}
On the event of Cor.~\ref{cor:comp-hp}, for any step $t$ and any $w$,
\[
\bigl\|\nabla \bar{\ell}^{\CL}_{\mathcal B_t}(w)-\nabla \bar{\ell}^{\NS}_{\mathcal B_t}(w)\bigr\|
~\le~ \frac{G}{\tau}\,\Delta_{C,\delta}(B;\tau).
\]
\end{lemma}

\begin{proof}
Fix $t$ and $w$. For an anchor $i\in\mathcal B_t$, let $D_i$ be its denominator index set, split as $D_i=\mathrm{pos}_i\cup\mathrm{neg}_i$, where $\mathrm{pos}_i$ collects all same-class indices (including the designated positive $i'$) and $\mathrm{neg}_i$ the rest. Write the logits $s_{ik}=\Sigma(i,k)$, the CL softmax $p_{ik}=\exp(s_{ik}/\tau)\big/\sum_{\ell\in D_i}\exp(s_{i\ell}/\tau)$, and the NSCL softmax over negatives $q_{ij}=p_{ij}/(1-\alpha_i)$ for $j\in\mathrm{neg}_i$, with $q_{k}=0$ for $k\in\mathrm{pos}_i$, where $\alpha_i:=\sum_{k\in\mathrm{pos}_i}p_{ik}$. Define $v_{ik}:=\nabla_w s_{ik}=\nabla_w\,\mathrm{sim}\bigl(f_w(x_i),f_w(x_k)\bigr)$; by Assumption~\ref{asm:pairwise-G}, $\|v_{ik}\|\le G$ for all $(i,k)$.

For the per-anchor losses,
\[
\nabla_w \ell^{\CL}_{i,\mathcal B_t}
~=~\frac{1}{\tau}\Bigl(\sum_{k\in D_i} p_{ik}\,v_{ik}-v_{ii'}\Bigr),
\qquad
\nabla_w \ell^{\NS}_{i,\mathcal B_t}
~=~\frac{1}{\tau}\Bigl(\sum_{j\in \mathrm{neg}_i} q_{ij}\,v_{ij}-v_{ii'}\Bigr).
\]
Hence the per-anchor gradient difference is
\[
\Delta g_i
~:=~\nabla_w \ell^{\CL}_{i,\mathcal B_t}-\nabla_w \ell^{\NS}_{i,\mathcal B_t}
~=~\frac{1}{\tau}\left(\underbrace{\sum_{k\in \mathrm{pos}_i} p_{ik}\,v_{ik}}_{(\mathrm{A})}
 ~+~\underbrace{\sum_{j\in \mathrm{neg}_i} (p_{ij}-q_{ij})\,v_{ij}}_{(\mathrm{B})}\right).
\]
By the triangle inequality and $\|v_{ik}\|\le G$,
\[
\|\Delta g_i\|
~\le~ \frac{G}{\tau}\Bigl(\sum_{k\in \mathrm{pos}_i} p_{ik}
+\sum_{j\in \mathrm{neg}_i} |p_{ij}-q_{ij}|\Bigr).
\]
Since $q_{ij}=p_{ij}/(1-\alpha_i)$ for $j\in\mathrm{neg}_i$,
\[
\sum_{j\in \mathrm{neg}_i}|p_{ij}-q_{ij}|
=\sum_{j\in \mathrm{neg}_i} p_{ij}\,\frac{\alpha_i}{1-\alpha_i}
=\alpha_i.
\]
Therefore $\|\Delta g_i\|\le \frac{G}{\tau}( \alpha_i+\alpha_i)=\frac{2G}{\tau}\alpha_i$. Writing $r_i:=\tfrac{Z_{\mathrm{pos}}}{Z_{\mathrm{neg}}}$ with
\(
Z_{\mathrm{pos}}=\sum_{k\in\mathrm{pos}_i}\exp(s_{ik}/\tau),\;
Z_{\mathrm{neg}}=\sum_{j\in\mathrm{neg}_i}\exp(s_{ij}/\tau),
\)
we have $\alpha_i=\frac{r_i}{1+r_i}$, hence $2\alpha_i=\frac{2r_i}{1+r_i}\le 2r_i$, so
\[
\|\Delta g_i\|~\le~ \frac{2G}{\tau}\,\frac{Z_{\mathrm{pos}}}{Z_{\mathrm{neg}}}.
\]
On the high-probability event of Cor.~\ref{cor:comp-hp}, for every anchor
\[
\frac{Z_{\mathrm{pos}}}{Z_{\mathrm{neg}}}
~\le~ \frac{\mathrm e^{2/\tau}\bigl(\tfrac{1}{C}+\epsilon_{B,\delta}\bigr)}{1-\tfrac{1}{C}-\epsilon_{B,\delta}}
~=~\tfrac12\,\Delta_{C,\delta}(B;\tau),
\]
so $\|\Delta g_i\|\le \frac{G}{\tau}\,\Delta_{C,\delta}(B;\tau)$ for all anchors $i$.

Finally, the batch gradients are averages over anchors:
\[
\nabla \bar\ell^{\CL}_{\mathcal B_t}-\nabla \bar\ell^{\NS}_{\mathcal B_t}
~=~\frac{1}{B}\sum_{i\in\mathcal B_t}\Delta g_i,
\]
hence
\[
\bigl\|\nabla \bar\ell^{\CL}_{\mathcal B_t}-\nabla \bar\ell^{\NS}_{\mathcal B_t}\bigr\|
~\le~ \frac{1}{B}\sum_{i\in\mathcal B_t}\|\Delta g_i\|
~\le~ \frac{1}{B}\sum_{i\in\mathcal B_t} \frac{G}{\tau}\,\Delta_{C,\delta}(B;\tau)
~=~ \frac{G}{\tau}\,\Delta_{C,\delta}(B;\tau).
\]
\end{proof}

\nonconvex*

\begin{proof}
Let $\Phi_t^{\CL}(w):=\bar\ell^{\CL}_{\mathcal B_t}(w)$ and $\Phi_t^{\NS}(w):=\bar\ell^{\NS}_{\mathcal B_t}(w)$.
Assume each $\Phi_t^{\CL}$ is $\beta$-smooth. Set $e_t:=\|w_t^{\CL}-w_t^{\NS}\|$.

Write
\begin{align*}
e_{t+1}
~&=~ \bigl\|w_{t+1}^{\CL}-w_{t+1}^{\NS}\bigr\|
~=~ \bigl\|T_t(w_t^{\CL})-\bigl(w_t^{\NS}-\eta_t\,\nabla\Phi_t^{\NS}(w_t^{\NS})\bigr)\bigr\|\\
~&\le~ \underbrace{\|T_t(w_t^{\CL})-T_t(w_t^{\NS})\|}_{\text{(I)}} ~+~
\eta_t\,\underbrace{\|\nabla\Phi_t^{\CL}(w_t^{\NS})-\nabla\Phi_t^{\NS}(w_t^{\NS})\|}_{\text{(II)}}.
\end{align*}

\emph{Bounding (I).}
Using the integral Hessian representation,
\[
\nabla\Phi_t^{\CL}(u)-\nabla\Phi_t^{\CL}(v)=H_t(v,u)\,(u-v),\qquad
H_t(v,u):=\int_0^1 \nabla^2\Phi_t^{\CL}(v+\tau(u-v))\,d\tau,
\]
and $\beta$-smoothness gives $\|H_t(v,u)\|_{2\to2}\le \beta$. Hence
\begin{equation*}
\begin{aligned}
\|T_t(u)-T_t(v)\|
~&=~\|(I-\eta_t H_t(v,u))(u-v)\| \\
~&\le~ \bigl\|I-\eta_t H_t(v,u)\bigr\|_{2\to2}\,\|u-v\| \\
~&\le~ (1+\eta_t\beta)\,\|u-v\|. 
\end{aligned}
\end{equation*}
Thus, $\text{(I)}\le (1+\eta_t\beta)\,e_t$.

\emph{Bounding (II).}
On the high-probability event of Cor.~\ref{cor:comp-hp}, Lem.~\ref{lem:gap-param} yields
\[
\text{(II)}~\le~ \frac{G}{\tau}\,\Delta_{C,\delta}(B;\tau).
\]

Combining the bounds,
\begin{equation}\label{eq:nc-rec-tau}
e_{t+1} ~\le~ (1+\eta_t\beta)\,e_t\ +\ \eta_t\,\frac{G}{\tau}\,\Delta_{C,\delta}(B;\tau).
\end{equation}

Iterating \eqref{eq:nc-rec-tau} from $e_0=0$ gives
\[
e_T ~\le~ \sum_{t=0}^{T-1}\eta_t\,\frac{G}{\tau}\,\Delta_{C,\delta}(B;\tau)\,\prod_{s=t+1}^{T-1}(1+\eta_s\beta)
~\le~ \frac{G}{\tau}\,\Delta_{C,\delta}(B;\tau)\sum_{t=0}^{T-1}\eta_t\,\exp\Bigl(\beta\sum_{s=t+1}^{T-1}\eta_s\Bigr),
\]
where we used $1+x\le e^x$. Let $S_k:=\sum_{s=k}^{T-1}\eta_s$ so that $S_t=\eta_t+S_{t+1}$.
Then for each $t$,
\[
\eta_t\,\exp(\beta S_{t+1})
\ \le\ \frac{1}{\beta}\Bigl(\exp(\beta S_t)-\exp(\beta S_{t+1})\Bigr),
\]
since $e^{\beta\eta_t}-1\ge \beta\eta_t$. Summing over $t=0,\dots,T-1$ telescopes to
\[
e_T \ \le\ \frac{G}{\beta\tau}\,\Delta_{C,\delta}(B;\tau)\,\Bigl(\exp\Bigl(\beta\sum_{t=0}^{T-1}\eta_t\Bigr)-1\Bigr).
\]

This holds with probability at least $1-\delta$ (by Cor.~\ref{cor:comp-hp}). 
\end{proof}

\subsection{Similarity-space analysis and coupling}\label{app:sim-space}

\begin{lemma}[Per-step gradient gap in similarity space]\label{lem:grad-gap-sim}
On the event of Cor.~\ref{cor:comp-hp}, for any step $t$,
\[
\bigl\|G^{\CL}_t(\Sigma^{\CL}_t)-G^{\NS}_t(\Sigma^{\NS}_t)\bigr\|_F
~\le~ \underbrace{\frac{1}{\tau}\cdot\frac{\Delta_{C,\delta}(B;\tau)}{\sqrt{B}}}_{\text{reweighting (block-orth.)}}
 ~+~ \underbrace{\frac{1}{2\tau^2 B}\,\bigl\|\Sigma^{\CL}_t-\Sigma^{\NS}_t\bigr\|_F}_{\text{Lipschitz in }\Sigma}.
\]
\end{lemma}

\begin{proof}
Add and subtract $G^{\NS}_t(\Sigma^{\CL}_t)$ and apply the triangle inequality:
\begin{equation}
\begin{aligned}
&\bigl\|G^{\CL}_t(\Sigma^{\CL}_t)-G^{\NS}_t(\Sigma^{\NS}_t)\bigr\|_F \\
&\le~ \underbrace{\bigl\|G^{\CL}_t(\Sigma^{\CL}_t)-G^{\NS}_t(\Sigma^{\CL}_t)\bigr\|_F}_{\text{(A)}}
~+~ \underbrace{\bigl\|G^{\NS}_t(\Sigma^{\CL}_t)-G^{\NS}_t(\Sigma^{\NS}_t)\bigr\|_F}_{\text{(B)}}.
\end{aligned}
\end{equation}

\emph{Term (B): Lipschitz in $\Sigma$.}
By the temperature-$\tau$ softmax–Hessian bound \eqref{eq:hess-lip},
\[
\text{(B)} ~\le~ \frac{1}{2\tau^2 B}\,\|\Sigma^{\CL}_t-\Sigma^{\NS}_t\|_F.
\]

\emph{Term (A): reweighting gap at fixed $\Sigma^{\CL}_t$.}
Decompose the batch gradient into anchor blocks:
\[
G^\circ_t(\Sigma)~=~\frac{1}{B}\sum_{i\in\mathcal B_t} g^\circ_{t,i}(\Sigma),
\qquad \circ\in\{\CL,\NS\},
\]
where each $g^\circ_{t,i}$ has support only on the coordinates of anchor $i$. For anchor $i$, with temperature $\tau$,
\(
g^{\CL}_{t,i}(\Sigma^{\CL}_t)=(1/\tau)(p_i-e_{i'}),\;
g^{\NS}_{t,i}(\Sigma^{\CL}_t)=(1/\tau)(q_i-e_{i'}),
\)
so $g^{\CL}_{t,i}(\Sigma^{\CL}_t)-g^{\NS}_{t,i}(\Sigma^{\CL}_t)=(1/\tau)(p_i-q_i)$ on that block. By block orthogonality (Lem.~\ref{lem:block-orth}),
\[
\text{(A)}
~=~\frac{1}{B}\Bigl\|\sum_{i\in\mathcal B_t} \frac{1}{\tau}(p_i-q_i)\Bigr\|_F
~=~\frac{1}{\tau B}\sqrt{\sum_{i\in\mathcal B_t}\|p_i-q_i\|_2^2}.
\]
On the event of Cor.~\ref{cor:comp-hp}, Lem.~\ref{lem:softmax-reweight} gives
$\|p_i-q_i\|_2\le \Delta_{C,\delta}(B;\tau)$ for every anchor, hence
\[
\text{(A)} ~\le~ \frac{1}{\tau B}\sqrt{B\,\Delta_{C,\delta}(B;\tau)^2}
~=~ \frac{1}{\tau}\cdot\frac{\Delta_{C,\delta}(B;\tau)}{\sqrt{B}}.
\]
Combining the bounds on (A) and (B) yields the claim.
\end{proof}

\SimCoupling*

\begin{proof}
Condition on the event of Cor.~\ref{cor:comp-hp} (which holds with probability at least $1-\delta$). Let
$D_t:=\|\Sigma^{\CL}_t-\Sigma^{\NS}_t\|_F$. From the coupled updates \eqref{eq:Sigma-descent},
\[
\Sigma^{\CL}_{t+1}-\Sigma^{\NS}_{t+1}
~=~\bigl(\Sigma^{\CL}_t-\Sigma^{\NS}_t\bigr)-\eta_t \left(G^{\CL}_t(\Sigma^{\CL}_t)-G^{\NS}_t(\Sigma^{\NS}_t)\right),
\]
hence
\[
D_{t+1}\ \le\ D_t+\eta_t\,\bigl\|G^{\CL}_t(\Sigma^{\CL}_t)-G^{\NS}_t(\Sigma^{\NS}_t)\bigr\|_F.
\]
Add and subtract $G^{\NS}_t(\Sigma^{\CL}_t)$ and apply Lem.~\ref{lem:grad-gap-sim} (reweighting gap $+$ Lipschitz with temperature $\tau$):
\[
\bigl\|G^{\CL}_t(\Sigma^{\CL}_t)-G^{\NS}_t(\Sigma^{\NS}_t)\bigr\|_F
~\le~ \frac{1}{\tau}\cdot\frac{\Delta_{C,\delta}(B;\tau)}{\sqrt{B}}
+\frac{1}{2\tau^2 B}\,D_t.
\]
Therefore,
\[
D_{t+1}~\le~ \Bigl(1+\frac{\eta_t}{2\tau^2 B}\Bigr)D_t ~+~\eta_t\,\frac{1}{\tau}\cdot\frac{\Delta_{C,\delta}(B;\tau)}{\sqrt{B}}.
\]
Let $\alpha_t:=\tfrac{\eta_t}{2\tau^2 B}$ and $\gamma_t:=\eta_t\,\tfrac{\Delta_{C,\delta}(B;\tau)}{\tau\sqrt{B}}$. With $D_0=0$ (shared initialization), the discrete Gr\"onwall/product form gives
\[
D_T ~\le~ \sum_{s=0}^{T-1}\gamma_s\prod_{u=s+1}^{T-1}(1+\alpha_u)
~\le~ \exp\Bigl(\sum_{u=0}^{T-1}\alpha_u\Bigr)\,\sum_{s=0}^{T-1}\gamma_s,
\]
using $\prod_u(1+\alpha_u)\le \exp(\sum_u\alpha_u)$. Substituting $\alpha_t,\gamma_t$ yields
\[
D_T ~\le~ \exp\Bigl(\frac{1}{2\tau^2 B}\sum_{t=0}^{T-1}\eta_t\Bigr)\;
\frac{1}{\tau\sqrt{B}}\Bigl(\sum_{t=0}^{T-1}\eta_t\Bigr)\,\Delta_{C,\delta}(B;\tau),
\]
as desired.
\end{proof}

{\bf Consequences for CKA and RSA.\enspace}

\CKAlowerbound*

\begin{proof}
Let $A_T:=\|K^{\CL}_T\|_F>0$ and $\Delta_{K,T}:=\|K^{\CL}_T-K^{\NS}_T\|_F$, where all norms are Frobenius.
Then
\begin{equation}
\begin{aligned}
\langle K^{\CL}_T,K^{\NS}_T\rangle
&= \big\langle K^{\CL}_T,\,K^{\CL}_T+(K^{\NS}_T-K^{\CL}_T)\big\rangle\\
&= \|K^{\CL}_T\|_F^2+\big\langle K^{\CL}_T,\,K^{\NS}_T-K^{\CL}_T\big\rangle 
~\ge~ A_T^2 - A_T\,\Delta_{K,T},   
\end{aligned}
\end{equation}
by Cauchy–Schwarz. By the triangle inequality, $\|K^{\NS}_T\|_F\le A_T+\Delta_{K,T}$. Hence
\[
\CKA_T
~=~\frac{\langle K^{\CL}_T,K^{\NS}_T\rangle}{\|K^{\CL}_T\|_F\,\|K^{\NS}_T\|_F}
~\ge~ \frac{A_T^2-A_T\Delta_{K,T}}{A_T(A_T+\Delta_{K,T})}
~=~ \frac{1-\Delta_{K,T}/A_T}{1+\Delta_{K,T}/A_T}.
\]
Next, $K^{\circ}_T=H\Sigma^{\circ}_T H$ with the centering projector $H=I-\tfrac1N\mathbf 1\mathbf 1^\top$, so $\Delta_{K,T}=\|H(\Sigma^{\CL}_T-\Sigma^{\NS}_T)H\|_F\le \|\Sigma^{\CL}_T-\Sigma^{\NS}_T\|_F$ because $\|H\|_{2\to2}=1$. By Thm.~\ref{thm:sim-coupling}, with probability at least $1-\delta$,
\[
\|\Sigma^{\CL}_T-\Sigma^{\NS}_T\|_F
\ \le\
\exp\Bigl(\frac{1}{2\tau^2 B}\sum_{t=0}^{T-1}\eta_t\Bigr)\;
\frac{1}{\tau\sqrt{B}}\Bigl(\sum_{t=0}^{T-1}\eta_t\Bigr)\,\Delta_{C,\delta}(B;\tau).
\]
Combining the last two equations yields the lower bound on $\CKA_T$ with probability at least $1-\delta$.
\end{proof}

\RSAlowerbound*

\begin{proof}
Let $M=\binom{N}{2}$ and let $C:=I-\tfrac{1}{M}\mathbf 1\mathbf 1^\top$ be the centering projector in $\mathbb R^M$.
Write $a_c:=Ca_T$ and $b_c:=Cb_T$. Then
\[
\RSA_T~=~\frac{\langle a_c,b_c\rangle}{\|a_c\|_2\,\|b_c\|_2}.
\]
For any nonzero $u$ and any $v$ in an inner-product space,
\[
\langle u,v\rangle~=~\langle u,u+(v-u)\rangle~=~\|u\|_2^2+\langle u,v-u\rangle
~\ge~ \|u\|_2^2-\|u\|_2\,\|v-u\|_2,
\]
and $\|v\|_2\le \|u\|_2+\|v-u\|_2$. Therefore,
\[
\frac{\langle u,v\rangle}{\|u\|_2\,\|v\|_2}
\ \ge\ \frac{1-\|v-u\|_2/\|u\|_2}{1+\|v-u\|_2/\|u\|_2}.
\]
Apply this with $u=a_c$ and $v=b_c$ to obtain
\[
\RSA_T\ ~\ge~ \frac{1-\|b_c-a_c\|_2/\|a_c\|_2}{1+\|b_c-a_c\|_2/\|a_c\|_2}.
\]
Since $C$ is an orthogonal projector, $\|b_c-a_c\|_2=\|C(b_T-a_T)\|_2\le \|b_T-a_T\|_2$.
By construction of the RDM vectors,
\[
b_T-a_T ~=~ -\,\mathrm{vec} \left(\mathrm{off}\bigl(\Sigma^{\NS}_T-\Sigma^{\CL}_T\bigr)\right),
\]
so $\|b_T-a_T\|_2=\|\mathrm{off}(\Sigma^{\NS}_T-\Sigma^{\CL}_T)\|_F\le \|\Sigma^{\NS}_T-\Sigma^{\CL}_T\|_F$.
Finally, by Thm.~\ref{thm:sim-coupling}, with probability at least $1-\delta$,
\[
\|\Sigma^{\NS}_T-\Sigma^{\CL}_T\|_F
~\le~
\exp\Bigl(\frac{1}{2\tau^2 B}\sum_{t=0}^{T-1}\eta_t\Bigr)\;
\frac{1}{\tau\sqrt{B}}\Bigl(\sum_{t=0}^{T-1}\eta_t\Bigr)\,\Delta_{C,\delta}(B;\tau).
\]
Combining the last three displays yields the stated $(1-r)/(1+r)$ lower bound on $\RSA_T$ after substituting $\|a_c\|_2=\sqrt{M}\,\sigma_{D,T}$.
\end{proof}

\end{document}